\documentclass{article}

\usepackage{times}
\usepackage{graphicx} 
\usepackage{subfigure} 
\usepackage{xspace} 

\usepackage{algorithm}
\usepackage{algorithmic}
\usepackage{amsthm}
\usepackage{amsmath}
\usepackage{amssymb}

\usepackage[accepted]{icml2017}

\newtheorem{theorem}{Theorem}
\newtheorem{lemma}[theorem]{Lemma}
\newtheorem{proposition}[theorem]{Proposition}
\newtheorem{corollary}[theorem]{Corollary}

\newtheorem{assumption}{Assumption}

\newtheorem*{theorem*}{Theorem}
\newtheorem*{corollary*}{Corollary}
\newtheorem*{lemma*}{Lemma}
\newtheorem*{proposition*}{Proposition}

\newcommand*{\A}{\mathbf{A}}
\newcommand*{\X}{\mathbf{X}}
\newcommand*{\Y}{\mathbf{Y}}
\newcommand*{\W}{\mathbf{W}}
\newcommand*{\F}{\mathbf{F}}
\newcommand*{\E}{\mathbb{E}}

\newcommand*{\p}{\mathbf{P}}
\newcommand*{\R}{\mathbb{R}}
\newcommand*{\Tr}{\mathrm{Tr}}

\newcommand*{\St}{\mathbf{S}}
\newcommand{\transp}{\mathsf{T}}
\newcommand{\bepsilon}{\boldsymbol{\epsilon}}
\newcommand{\wt}[1]{\widetilde{#1}}
\newcommand{\wh}[1]{\widehat{#1}}

\newcommand{\wb}[1]{\overline{#1}}
\newcommand{\sparsetraceucb}{\textsc{Sparse-Trace-UCB}\xspace}
\newcommand{\traceucb}{\textsc{Trace-UCB}\xspace}
\newcommand{\varucb}{\textsc{Var-UCB}\xspace}
\newcommand{\chas}{\textsc{CH-AS}\xspace}
\newcommand{\gafs}{\textsc{\small{GAFS-MAX}}\xspace}

\icmltitlerunning{Active Learning for Accurate Estimation of Linear Models}

\begin{document} 

\twocolumn[
\icmltitle{Active Learning for Accurate Estimation of Linear Models}

\icmlsetsymbol{equal}{*}

\begin{icmlauthorlist}
\icmlauthor{Carlos Riquelme}{st}
\icmlauthor{Mohammad Ghavamzadeh}{goog}
\icmlauthor{Alessandro Lazaric}{in}
\end{icmlauthorlist}

\icmlaffiliation{st}{Stanford University, Stanford, CA, USA.}
\icmlaffiliation{goog}{DeepMind, Mountain View, CA, USA (The work was done when the author was with Adobe Research).}
\icmlaffiliation{in}{Inria Lille, France}

\icmlcorrespondingauthor{Carlos Riquelme}{rikel@stanford.edu}

\icmlkeywords{boring formatting information, machine learning, ICML}

\vskip 0.3in
]

\printAffiliationsAndNotice{}

\begin{abstract} 
We explore the sequential decision-making problem where the goal is to estimate a number of linear models uniformly well, given a shared budget of random contexts independently sampled from a known distribution. For each incoming context, the decision-maker selects one of the linear models and receives an observation that is corrupted by the unknown noise level of that model. We present Trace-UCB, an adaptive allocation algorithm that learns the models' noise levels while balancing contexts accordingly across them, and prove bounds for its simple regret in both expectation and high-probability. We extend the algorithm and its bounds to the high dimensional setting, where the number of linear models times the dimension of the contexts is more than the total budget of samples. Simulations with real data suggest that Trace-UCB is remarkably robust, outperforming a number of baselines even when its assumptions are violated.
\end{abstract}


\section{Introduction}
\label{intro}

We study the problem faced by a decision-maker whose goal is to estimate a number of regression problems equally well (i.e., with a small prediction error for each of them), and has to adaptively allocate a limited budget of samples to the problems in order to gather information and improve its estimates. 
Two aspects of the problem formulation are key and drive the algorithm design: {\bf 1)} The observations $Y$ collected from each regression problem depend on side information (i.e.,~contexts $X\in\R^d$) and we model the relationship between $X$ and $Y$ in each problem $i$ as a linear function with unknown parameters $\beta_i\in\R^d$, and {\bf 2)} The ``hardness'' of learning each parameter $\beta_i$ is unknown in advance and may vary across the problems. In particular, we assume that the observations are corrupted by noise levels that are problem-dependent and must be learned as well.

This scenario may arise in a number of different domains where a fixed experimentation budget (number of samples) should be allocated to different problems.
Imagine a drug company that has developed several treatments for a particular form of disease. Now it is interested in having an accurate estimate of the performance of each of these treatments for a specific population of patients (e.g.,~at a particular geographical location). Given the budget allocated to this experiment, a number of patients $n$ can participate in the clinical trial. 
Volunteered patients arrive sequentially over time
and they are represented by a context $X \in \R^d$ summarizing their profile. We model the health status of patient $X$ after being assigned to treatment $i$ by scalar $Y_i \in \R$, which depends on the specific drug through a linear function with parameter $\beta_i$ (i.e., $Y_i \approx X^\transp \beta_i$). The goal is to assign each incoming patient to a treatment in such a way that at the end of the trial, we have an accurate estimate for all $\beta_i$'s. This will allow us to reliably predict the expected health status of each new patient $X$ for any treatment $i$. Since the parameters $\beta_i$ and the noise levels are initially unknown, achieving this goal requires an adaptive allocation strategy for the $n$ patients. Note that while $n$ may be relatively small, as the ethical and financial costs of treating a patient are high, the distribution of the contexts $X$ (e.g., the biomarkers of cancer patients) can be precisely estimated in advance.



This setting is clearly related to the problem of pure exploration and active learning in multi-armed bandits~\citep{antos2008active}, where the learner wants to estimate the mean of a finite set of arms by allocating a finite budget of $n$ pulls.~\citet{antos2008active} first introduced this setting where the objective is to minimize the largest mean square error (MSE) in estimating the value of each arm. While the optimal solution is trivially to allocate the pulls proportionally to the variance of the arms, when the variances are unknown an exploration-exploitation dilemma arises, where variance and value of the arms must be estimated at the same time in order to allocate pulls where they are more needed (i.e.,~arms with high variance).~\citet{antos2008active} proposed a forcing algorithm where all arms are pulled at least $\sqrt{n}$ times before allocating pulls proportionally to the estimated variances. They derived bounds on the regret, measuring the difference between the MSEs of the learning algorithm and an optimal allocation showing that the regret decreases as $O(n^{-3/2})$. A similar result is obtained by~\citet{carpentier2011upper} that proposed two algorithms that use upper confidence bounds on the variance to estimate the MSE of each arm and select the arm with the larger MSE at each step. When the arms are embedded in $\R^d$ and their mean is a linear combination with an unknown parameter, then the problem becomes an optimal experimental design problem~\citep{pukelsheim2006optimal}, where the objective is to estimate the linear parameter and minimize the prediction error over all arms (see e.g.,~\citealt{WieLi14,Sabato14AR}). In this paper, we consider an orthogonal extension to the original problem where a finite number of linear regression problems is available (i.e., the arms) and random contexts are observed at each time step. Similarly to the setting of~\citet{antos2008active}, we assume each problem is characterized by a noise with different variance and the objective is to return regularized least-squares (RLS) estimates with small prediction error (i.e., MSE). While we leverage on the solution proposed by~\citet{carpentier2011upper} to deal with the unknown variances, in our setting the presence of random contexts make the estimation problem considerably more difficult. In fact, the MSE in one specific regression problem is not only determined by the variance of the noise and the number of samples used to compute the RLS estimate, but also by the contexts observed over time. 

\textbf{Contributions.} We propose \traceucb, an algorithm that simultaneously learns the ``hardness'' of each problem, allocates observations proportionally to these estimates, and balances contexts across problems. We derive performance bounds for \traceucb in expectation and high-probability, and compare the algorithm with several baselines. \traceucb performs remarkably well in scenarios where the dimension of the contexts or the number of instances is large compared to the total budget, motivating the study of the high-dimensional setting, whose analysis and performance bounds are reported in App.~\ref{app:high.dimensions} of~\citet{riquelme2017active}.
Finally, we provide simulations with synthetic data that support our theoretical results, and with real data that demonstrate the robustness of our approach even when some of the assumptions do not hold.


\section{Preliminaries}
\label{problem}

\textbf{The problem.}
We consider $m$ linear regression problems, where each instance $i\in[m] = \{1,\ldots,m\}$ is characterized by a parameter $\beta_i\in\R^d$ such that for any context $X\in\R^d$, a random observation $Y\in\R$ is obtained as
\begin{align}\label{eq:observation}
Y = X^\transp \beta_i + \epsilon_i,
\end{align}
where the noise $\epsilon_i$ is an i.i.d.\ realization of a Gaussian distribution $\mathcal{N}(0, \sigma_i^2)$. We denote by $\sigma_{\max}^2 = \max_i \sigma_i^2$ and by $\wb{\sigma}^2 = 1/m \sum_i \sigma_i^2$, the largest and the average variance, respectively. We define a sequential decision-making problem over $n$ rounds, where at each round $t\in[n]$, the learning algorithm $\mathcal{A}$ receives a context $X_t$ drawn i.i.d.\ from $\mathcal{N}(0, \Sigma)$, selects an instance $I_t$, and observes a random sample $Y_{I_t,t}$ according to~\eqref{eq:observation}. By the end of the experiment, a training set $\mathcal{D}_n = \{X_t, I_t, Y_{I_t,t}\}_{t\in[n]}$ has been collected and all the $m$ linear regression problems are solved, each problem $i\in[m]$ with its own training set $\mathcal{D}_{i,n}$ (i.e.,~a subset of $\mathcal{D}_n$ containing samples with $I_t=i$), and estimates of the parameters $\{\hat\beta_{i,n}\}_{i\in[m]}$ are returned. For each $\hat\beta_{i,n}$, we measure its accuracy by the mean-squared error (MSE)
\begin{align}\label{eq:mse}
\hspace{-0.1in}L_{i,n}(\hat\beta_{i,n}) \!=\! \E_{X}\big[(X^\transp \beta_i \!-\!X^\transp \hat\beta_{i,n})^2\big] \!=\! \| \beta_i \!-\!\hat\beta_{i,n}\|^2_{\Sigma}.
\end{align}
We evaluate the overall accuracy of the estimates returned by the algorithm $\mathcal{A}$ as
\begin{equation}\label{eq_global_loss}
L_n(\mathcal{A}) = \max_{i \in [m]} \E_{\mathcal{D}_n} \big[L_{i,n}(\hat\beta_{i,n})\big],
\end{equation}
where the expectation is w.r.t.~the randomness of the contexts $X_t$ and observations $Y_{i,t}$ used to compute $\hat\beta_{i,n}$. The objective is to design an algorithm $\mathcal{A}$ that minimizes the loss~\eqref{eq_global_loss}. This requires defining an allocation rule to select the instance $I_t$ at each step $t$ and the algorithm to compute the estimates $\hat\beta_{i,n}$, e.g.,~ordinary least-squares (OLS), regularized least-squares (RLS), or Lasso. In designing a learning algorithm, we rely on the following assumption.
\begin{assumption}\label{asm:covariance}
The covariance matrix $\Sigma$ of the Gaussian distribution generating the contexts $\{X_t\}_{t=1}^n$ is known.
\end{assumption}
This is a standard assumption in active learning, since in this setting the learner has access to the input distribution and the main question is for which context she should ask for a label~\cite{Sabato14AR, riquelme2016online}. Often times, companies, like the drug company considered in the introduction, own enough data to have an accurate estimate of the distribution of their customers (patients).

While in the rest of the paper we focus on $L_n(\mathcal{A})$, our algorithm and analysis can be easily extended to similar objectives such as replacing the maximum in~\eqref{eq_global_loss} with average across all instances, i.e., $1/m \sum_{i=1}^m \E_{\mathcal{D}_n}\big[L_{i,n}(\hat\beta_{i,n})\big]$, and using weighted errors, i.e.,~$\max_i w_i \ \E_{\mathcal{D}_n}\big[L_{i,n}(\hat\beta_{i,n})\big]$, by updating the score to focus on the estimated standard deviation and by including the weights in the score, respectively. Later in the paper, we also consider the case where the expectation in~\eqref{eq_global_loss} is replaced by the high-probability error (see Eq.~\ref{eq:random.loss}).


\textbf{Optimal static allocation with OLS estimates.} While the distribution of the contexts is fixed and does not depend on the instance $i$, the errors $L_{i,n}(\hat\beta_{i,n})$ directly depend on the variances $\sigma_i^2$ of the noise $\epsilon_i$. We define an optimal baseline obtained when the noise variances $\{\sigma_i^2\}_{i=1}^m$ are known. In particular, we focus on a static allocation algorithm $\mathcal{A}_{\text{stat}}$ that selects each instance $i$ exactly $k_{i,n}$ times, independently of the context,\footnote{This strategy can be obtained by simply selecting the first instance $k_{1,n}$ times, the second one $k_{2,n}$ times, and so on.} and returns an estimate $\hat\beta_{i,n}$ computed by OLS as
\begin{equation}\label{eq:ols}
\wh\beta_{i,n} = \big(\X_{i,n}^\transp \X_{i,n}\big)^{-1} \X_{i,n}^\transp \Y_{i,n},
\end{equation}
where $\X_{i,n} \in \R^{k_{i,n}\times d}$ is the matrix of (random) samples obtained at the end of the experiment, and $\Y_{i,n} \in\R^{k_{i,n}}$ is its corresponding vector of observations. It is simple to show that the global error corresponding to $\mathcal{A}_{\text{stat}}$ is
\begin{equation}\label{eq_global_loss.static}
L_n(\mathcal{A}_{\text{stat}}) = \max_{i \in [m]} \frac{\sigma_i^2}{k_{i,n}} \Tr\Big( \Sigma \E_{\mathcal{D}_n} \big[\wh\Sigma_{i,n}^{-1}\big]\Big),
\end{equation}
where $\wh\Sigma_{i,n} = \X_{i,n}^\transp \X_{i,n}/k_{i,n} \in \R^{d\times d}$ is the empirical covariance matrix of the contexts assigned to instance $i$. Since the algorithm does not change the allocation depending on the contexts and $X_t \sim \mathcal{N}(0, \Sigma)$, $\wh\Sigma_{i,n}^{-1}$ is distributed as an inverse-Wishart and we may write~\eqref{eq_global_loss.static} as
\begin{equation}\label{eq_global_loss.static2}
L_n(\mathcal{A}_{\text{stat}}) = \max_{i \in [m]} \frac{d\sigma_i^2}{k_{i,n}-d-1}.
\end{equation}
Thus, we derive the following proposition for the optimal static allocation algorithm $\mathcal{A}^*_{\text{stat}}$.
\begin{proposition}\label{p:optimal.static}
Given $m$ linear regression problems, each characterized by a parameter $\beta_i$, Gaussian noise with variance $\sigma_i^2$, and Gaussian contexts with covariance $\Sigma$, let $n > m(d+1)$, then the optimal OLS static allocation algorithm $\mathcal{A}^*_{\text{stat}}$ selects each instance
\begin{equation}\label{eq_opt_stat_alloc}
k_{i,n}^* = \frac{\sigma_i^2}{\sum_j \sigma_j^2} \ n + (d+1) \ \left( 1 - \frac{\sigma_i^2}{\wb{\sigma}^2} \right),
\end{equation}
times (up to rounding effects), and incurs the global error
\begin{equation}\label{eq_opt_stat_loss}
L^*_n = L_n(\mathcal{A}^*_{\text{stat}}) = \wb{\sigma}^2 \frac{md}{n} + O \left( \wb{\sigma}^2 \left( \frac{md}{n} \right)^2 \right).
\end{equation}
\end{proposition}
\begin{proof}
See Appendix~\ref{app:optimal.static1}.\footnote{All the proofs can be found in the appendices of the extended version of the paper~\cite{riquelme2017active}.}
\end{proof}
Proposition~\ref{p:optimal.static} divides the problems into two types: those for which $\sigma_i^2 \ge \bar \sigma^2$ (\emph{wild} instances) and those for which $\sigma_i^2 < \bar \sigma^2$ (\emph{mild} instances). We see that for the first type, the second term in~\eqref{eq_opt_stat_alloc} is negative and the instance should be selected less frequently than in the context-free case (where the optimal allocation is given just by the first term). On the other hand, instances whose variance is below the mean variance should be pulled more often. In any case, we see that the correction to the context-free allocation (i.e.,~the second term) is \emph{constant}, as it does not depend on $n$. Nonetheless, it does depend on $d$ and this suggests that in high-dimensional problems, it may significantly skew the optimal allocation.

While $\mathcal{A}^*_{\text{stat}}$ effectively minimizes the prediction loss $L_n$, it cannot be implemented in practice since the optimal allocation $k_i^*$ requires the variances $\sigma_i^2$ to be known at the beginning of the experiment. As a result, we need to devise a learning algorithm $\mathcal{A}$ whose performance approaches $L_n^*$ as $n$ increases. More formally, we define the regret of $\mathcal{A}$ as
\begin{align}\label{eq:regret}
R_n(\mathcal{A}) = L_n(\mathcal{A}) - L_n(\mathcal{A}^*_{\text{stat}}) = L_n(\mathcal{A}) - L_n^*,
\end{align}
and we expect $R_n(\mathcal{A}) = o(1/n)$. In fact, any allocation strategy that selects each instance a linear number of times (e.g.,~uniform sampling) achieves a loss $L_n = O(1/n)$, and thus, a regret of order $O(1/n)$. However, we expect that the loss of an effective learning algorithm decreases not just at the same rate as $L_n^*$ but also with the very same constant, thus implying a regret that decreases faster than $O(1/n)$.

\section{The \traceucb Algorithm}

In this section, we present and analyze an algorithm of the form discussed at the end of Section~\ref{problem}, which we call~\traceucb, whose pseudocode is in Algorithm~\ref{alg:trace_ucb}.

\begin{algorithm}[ht]
\begin{algorithmic}[1]
\FOR{$\;i=1,\ldots,m\;$}
\STATE Select problem instance $i$ exactly $d+1$ times
\STATE Compute its OLS estimates $\hat\beta_{i,m(d+1)}$ and $\hat\sigma^2_{i,m(d+1)}$
\ENDFOR
\FOR{steps $\;t=m(d+1)+1,\ldots,n\;$}
\FOR{problem instance $\;1 \le i \le m\;$}
\STATE Compute score \hfill \begin{small}{\em ($\Delta_{i,t-1}$ is defined in \eqref{eq:good_event_sigma})}\end{small}
\begin{equation*}
s_{i,t-1} = \frac{\wh\sigma_{i,t-1}^2 + \Delta_{i,t-1}}{k_{i,t-1}}\mathrm{Tr}\big( \Sigma \hat\Sigma^{-1}_{i,t-1} \big) 
\end{equation*}
\ENDFOR
\STATE Select problem instance $I_t = \arg\max_{i\in[m]} s_{i,t-1}$
\STATE Observe $X_t$ and $Y_{I_t, t}$
\STATE Update its OLS estimators $\hat\beta_{I_t,t}$ and $\hat\sigma^2_{I_t,t}$
\ENDFOR
\STATE Return RLS estimates $\{\hat\beta_{i,n}^{\lambda}\}_{i=1}^m$ with regularization $\lambda$
\end{algorithmic}
\caption{\traceucb Algorithm}
\label{alg:trace_ucb}
\end{algorithm}

The regularization parameter $\lambda = O(1/n)$ is provided to the algorithm as input, while in practice one could set $\lambda$ independently for each arm using cross-validation.

\textbf{Intuition.} Equation~\eqref{eq_global_loss.static2} suggests that while the parameters of the context distribution, particularly its covariance $\Sigma$, do not impact the prediction error, the noise variances play the most important role in the loss of each problem instance. This is in fact confirmed by the optimal allocation $k_{i,n}^*$ in~\eqref{eq_opt_stat_alloc}, where only the variances $\sigma_i^2$ appear. This evidence suggests that an algorithm similar to \gafs~\citep{antos2008active} or \chas~\citep{carpentier2011upper}, which were designed for the context-free case (i.e., each instance $i$ is associated to an expected value and not a linear function) would be effective in this setting as well. Nonetheless, \eqref{eq_global_loss.static2} holds only for static allocation algorithms that completely ignore the context and the history to decide which instance $I_t$ to choose at time $t$. On the other hand, adaptive learning algorithms create a strong correlation between the dataset $\mathcal{D}_{t-1}$ collected so far, the current context $X_t$, and the decision $I_t$. As a result, the sample matrix $\X_{i,t}$ is no longer a random variable independent of $\mathcal{A}$, and using~\eqref{eq_global_loss.static2} to design a learning algorithm is not convenient, since the impact of the contexts on the error is completely overlooked. Unfortunately, in general, it is very difficult to study the potential correlation between the contexts $\X_{i,t}$, the intermediate estimates $\hat\beta_{i,t}$, and the most suitable choice $I_t$. However, in the next lemma, we show that if at each step $t$, we select $I_t$ as a function of $\mathcal{D}_{t-1}$, and \textit{not} $X_t$, we may still recover an expression for the final loss that we can use as a basis for the construction of an effective learning algorithm.

\begin{lemma}\label{lem:loss.learning}
Let $\mathcal{A}$ be a learning algorithm that selects the instances $I_t$ as a function of the previous history, i.e.,~$\mathcal{D}_{t-1} = \{X_1, I_1, Y_{I_1, 1}, \ldots, X_{t-1}, I_{t-1}, Y_{I_{t-1},t-1}\}$ and computes estimates $\wh\beta_{i,n}$ using OLS. Then, its loss after $n$ steps can be expressed as
\begin{align}\label{eq:loss.learning}
L_n(\mathcal{A}) = \max_{i\in[m]}\ \E_{\mathcal{D}_n}\bigg[ \frac{\sigma_i^2}{k_{i,n}} \Tr\Big(\Sigma \wh\Sigma_{i,n}^{-1}\Big)\bigg],
\end{align}
where $k_{i,n} = \sum_{t=1}^n \mathbb{I}\{I_t=i\}$ and $\wh\Sigma_{i,n} = \X_{i,n}^\transp \X_{i,n} / k_{i,n}$.
\end{lemma}
\begin{proof}
See Appendix~\ref{app:ols.loss.learning}.
\end{proof}
\paragraph{Remark~1 (assumptions).}We assume noise and contexts are Gaussian. The noise Gaussianity is crucial for the estimates of the parameter $\wh\beta_{i,t}$ and variance $\wh\sigma^2_{i,t}$ to be independent of each other, for each instance $i$ and time $t$ (we actually need and derive a stronger result in Lemma~\ref{lm:indep_hatbeta_hatsigma}, see Appendix~\ref{app:ols.loss.learning}).
This is key in proving Lemma~\ref{lem:loss.learning}, as it allows us to derive a closed form expression for the loss function which holds under our algorithm, and is written in terms of the number of pulls and the trace of the inverse empirical covariance matrix.
Note that $\wh\beta_{i,t}$ drives our loss, while $\wh\sigma^2_{i,t}$ drives our decisions.
One way to remove this assumption is by defining and directly optimizing a surrogate loss equal to~\eqref{eq:loss.learning} instead of~\eqref{eq_global_loss}.
On the other hand, the Gaussianity of contexts leads to the whitened inverse covariance estimate $\Sigma\wh\Sigma^{-1}_{i,n}$ being distributed as an inverse Wishart.
As there is a convenient closed formula for its mean, we can find the exact optimal static allocation $k_{i,n}^*$ in Proposition~\ref{p:optimal.static}, see \eqref{eq_opt_stat_alloc}.
In general, for sub-Gaussian contexts, no such closed formula for the trace is available.
However, as long as the optimal allocation $k_{i,n}^*$ has no second order $n^\alpha$ terms for $1/2 \le \alpha < 1$, it is possible to derive the same regret rate results that we prove later on for \traceucb.  

Equation \eqref{eq:loss.learning} makes it explicit that the prediction error comes from two different sources. The first one is the noise in the measurements $\Y$, whose impact is controlled by the unknown variances $\sigma_i^2$'s. Clearly, the larger the $\sigma_i^2$ is, the more observations are required to achieve the desired accuracy. At the same time, the \emph{diversity} of contexts across instances also impacts the overall prediction error.
This is very intuitive, since it would be a terrible idea for the research center discussed in the introduction to estimate the parameters of a drug by providing the treatment only to a hundred almost identical patients.
We say contexts are balanced when $\wh\Sigma_{i,n}$ is well conditioned. Therefore, a good algorithm should take care of both aspects.

There are two extreme scenarios regarding the contributions of the two sources of error. {\bf 1)} If the number of contexts $n$ is relatively large, since the context distribution is fixed, one can expect that contexts allocated to each instance eventually become balanced (i.e., \traceucb does not bias the distribution of the contexts). In this case, it is the difference in $\sigma_i^2$'s that drives the number of times each instance is selected. {\bf 2)} When the dimension $d$ or the number of arms $m$ is large w.r.t.~$n$, balancing contexts becomes critical, and can play an important role in the final prediction error, whereas the $\sigma_i^2$'s are less relevant in this scenario. While a learning algorithm cannot deliberately choose a specific context (i.e.,~$X_t$ is a random variable), we may need to favor instances in which the contexts are poorly balanced and their prediction error is large, despite the fact that they might have small noise variances.

\textbf{Algorithm.} \traceucb is designed as a combination of the upper-confidence-bound strategy used in \chas~\citep{carpentier2011upper} and the loss in~\eqref{eq:loss.learning}, so as to obtain a learning algorithm capable of allocating according to the estimated variances and at the same time balancing the error generated by context mismatch. We recall that all the quantities that are computed at every step of the algorithm are indexed at the beginning and end of a step $t$ by $i,t-1$ (e.g., $\wh\sigma^2_{i,t-1}$) and $i,t$ (e.g., $\wh\beta_{i,t}$), respectively. At the end of each step $t$, \traceucb first computes an OLS estimate $\wh\beta_{i,t}$, and then use it to estimate the variance $\wh\sigma_{i,t}^2$ as
\begin{align*}
\wh\sigma_{i,t}^2 = \frac{1}{k_{i,t}-d} \big\| \Y_{i,t} - \X_{i,t}^\transp \wh\beta_{i,t} \big\|^2,
\end{align*}
which is the average squared deviation of the predictions based on $\wh\beta_{i,t}$. We rely on the following concentration inequality for the variance estimate of linear regression with Gaussian noise, whose proof is reported in Appendix~\ref{app:concentration1}.

\begin{proposition}\label{prop:sigma.concentration}
Let the number of pulls $k_{i,t} \geq d+1$ and $R \ge \max_i \sigma_i^2$. If $\delta\in(0,3/4)$, then for any instance $i$ and step $t > m(d+1)$, with probability at least $1 - \frac{\delta}{2}$, we have 
\begin{equation}\label{eq:good_event_sigma}
|\hat{\sigma}_{i,t}^2 - \sigma_i^2 | \le \Delta_{i,t} \stackrel{\Delta}{=} R \sqrt{\frac{64}{k_{i,t} - d} \left( \log \frac{2mn}{\delta} \right)^2 }.
\end{equation}
\end{proposition}

Given~\eqref{eq:good_event_sigma}, we can construct an upper-bound on the prediction error of any instance $i$ and time step $t$ as
\begin{equation}\label{eq_score_it}
s_{i,t-1} = \frac{{\hat{\sigma}_{i,t-1}^2 + \Delta_{i,t-1}}}{k_{i,t-1}} \ \mathrm{Tr}\left(\Sigma \hat\Sigma_{i,t-1}^{-1}\right),
\end{equation}
and then simply select the instance which maximizes this score, i.e.,~$I_t = \arg\max_{i} s_{i,t-1}$.
Intuitively, \traceucb favors problems where the prediction error is potentially large, either because of a large noise variance or because of significant unbalance in the observed contexts w.r.t.~the target distribution with covariance $\Sigma$. A subtle but critical aspect of \traceucb is that by ignoring the current context $X_t$ (but using all the past samples $\X_{t-1}$) when choosing $I_t$, the distribution of the contexts allocated to each instance stays untouched and the second term in the score $s_{i,t-1}$, i.e.,~$\Tr(\Sigma \wh\Sigma_{i,t-1}^{-1})$, naturally tends to $d$ as more and more (random) contexts are allocated to instance $i$. This is shown by Proposition~\ref{prop:trace.concentration} whose proof is in Appendix~\ref{app:concentration2}.

\begin{proposition}\label{prop:trace.concentration}
Force the number of samples $k_{i,t} \geq d+1$. If $\delta\in(0,1)$, for any $i \in [m]$ and step $t > m(d+1)$ with probability at least $1 - \delta/2$, we have
\begin{align*}
\bigg(1-C_{\Tr}\sqrt{\frac{d}{k_{i,t}}}\bigg)^2 \leq \frac{\Tr\Big( \Sigma \hat\Sigma^{-1}_{i,t} \Big)}{d} \le \bigg(1+2C_{\Tr}\sqrt{\frac{d}{k_{i,t}}}\bigg)^2,
\end{align*}
with $C_{\Tr} = 1+\sqrt{2\log(4nm/\delta)/d}$.
\end{proposition}

While Proposition~\ref{prop:trace.concentration} shows that the error term due to context mismatch tends to the constant $d$ for all instances $i$ as the number of samples tends to infinity, when $t$ is small w.r.t.~$d$ and $m$, correcting for the context mismatch may significantly improve the accuracy of the estimates $\wh\beta_{i,n}$ returned by the algorithm. Finally, note that while \traceucb uses OLS to compute estimates $\wh\beta_{i,t}$, it computes its returned parameters $\wh\beta_{i,n}$ by ridge regression (RLS) with regularization parameter $\lambda$ as
\begin{equation}\label{eq:rls}
\hat\beta^\lambda_i = (\X_{i,n}^\transp\X_{i,n} + \lambda \mathbf{I})^{-1} \X_{i,n}^\transp\Y_{i,n}.
\end{equation}
As we will discuss later, using RLS makes the algorithm more robust and is crucial in obtaining regret bounds both in expectation and high probability.

\textbf{Performance Analysis.}
Before proving a regret bound for \traceucb, we report an intermediate result (proof in App.~\ref{app:traceucb1}) that shows that \traceucb \textit{behaves} similarly to the optimal static allocation.


\begin{theorem}\label{th:lower_bound_numpulls}
Let $\delta > 0$.
With probability at least $1 - \delta$, the total number of contexts that \traceucb allocates to each problem instance $i$ after $n$ rounds satisfies
\begin{equation}
\label{eq_lower_bound_numpulls}
k_{i,n} \geq k_{i,n}^* - \frac{C_\Delta + 8C_{\Tr}}{\sigma_{\min}^2} \sqrt{\frac{nd}{\lambda_{\min}}} -\Omega(n^{1/4})
\end{equation}
where $R \ge \sigma_{\max}^2$ is known by the algorithm, and we defined $C_\Delta = 16 R \log(2mn/\delta)$ and $\lambda_{\min} = \sigma_{\min}^2 / \sum_j \sigma_j^2$.
\end{theorem}
We now report our regret bound for the \traceucb algorithm. The proof of Theorem~\ref{thm_exp_regret} is in Appendix~\ref{app:traceucb2}.
\begin{theorem}\label{thm_exp_regret}
The regret of the Trace-UCB algorithm, i.e.,~the difference between its loss and the loss of optimal static allocation (see Eq.~\eqref{eq_opt_stat_loss}), is upper-bounded by
\begin{equation}
\label{eq:regret-bound1}
L_n(\mathcal{A}) - L^*_n \leq O\bigg(\frac{1}{\sigma_{\min}^2}\Big(\frac{d}{\lambda_{\min}n}\Big)^{3/2}\bigg).
\end{equation}
\end{theorem}

\vspace{-0.1in}

Eq.~\eqref{eq:regret-bound1} shows that the regret decreases as $O(n^{-3/2})$ as expected. 
This is consistent with the context-free results~\cite{antos2008active,carpentier2011upper}, where the regret decreases as $n^{-3/2}$, which is conjectured to be optimal. However, it is important to note that in the contextual case, the numerator also includes  the dimensionality $d$. Thus, when $n \gg d$, the regret will be small, and it will be larger when $n \approx d$. This motivates studying the high-dimensional setting (App.~\ref{app:high.dimensions}). Eq.~\eqref{eq:regret-bound1} also indicates that the regret depends on a problem-dependent constant $1 / \lambda_{\min}$, which measures the complexity of the problem. Note that when $\sigma_{\max}^2 \approx \sigma_{\min}^2$, we have $1 / \lambda_{\min} \approx m$, but $1 / \lambda_{\min}$ could be much larger when $\sigma_{\max}^2 \gg \sigma_{\min}^2$. 

\paragraph{Remark 2.} 
We introduce a baseline motivated by the context-free problem.
At round $t$, let \varucb selects the instance that maximizes the score\footnote{Note that \varucb is similar to both the \chas and B-AS algorithms in~\citet{carpentier2011upper}.} 
\begin{equation}\label{eq_score_it_cfree}
s^{\prime}_{i,t-1} = \frac{\hat{\sigma}_{i,t-1}^2 + \Delta_{i,t-1}}{k_{i,t-1}}.
\end{equation}
The only difference with the score used by \traceucb is the lack of the trace term in \eqref{eq_score_it}.
Moreover, the regret of this algorithm has similar \emph{rate} in terms of $n$ and $d$ as that of \traceucb reported in Theorem~\ref{thm_exp_regret}. However, the simulations of Sect.~\ref{sims} show that the regret of \varucb is actually much higher than that of \traceucb, specially when $dm$ is close to $n$.
Intuitively, when $n$ is close to $dm$, balancing contexts becomes critical, and \varucb suffers because its score does not explicitly take them into account.

%


\textbf{Sketch of the proof of Theorem \ref{thm_exp_regret}.}
The proof is divided into three parts. {\bf 1)} We show that the behavior of the ridge loss of \traceucb is similar to that reported in Lemma~\ref{lem:loss.learning} for algorithms that rely on OLS; see Lemma~\ref{lm:algo_ridge_loss} in Appendix~\ref{app:rls.loss.learning}. The independence of the $\hat\beta_{i,t}$ and $\hat\sigma_{i,t}^2$ estimates is again essential (see Remark~1).
Although the loss of \traceucb depends on the ridge estimate of the parameters $\hat{\beta}^\lambda_{i,n}$, the decisions made by the algorithm at each round only depend on the variance estimates $\hat\sigma_{i,t}^2$ and observed contexts. {\bf 2)} We follow the ideas in~\citet{carpentier2011upper} to lower-bound the total number of pulls $k_{i,n}$ for each $i \in [m]$ under a good event (see Theorem~\ref{th:lower_bound_numpulls} and its proof in Appendix~\ref{app:traceucb1}). {\bf 3)} We finally use the ridge regularization to bound the impact of those cases \emph{outside} the good event, and combine everything in Appendix~\ref{app:traceucb2}.

The regret bound of Theorem~\ref{thm_exp_regret} shows that the largest \textit{expected} loss across the problem instances incurred by \traceucb quickly approaches the loss of the optimal static allocation algorithm (which knows the true noise variances). 
While $L_n(\mathcal{A})$ measures the worst \textit{expected} loss, at any specific \textit{realization} of the algorithm, there may be one of the instances which is very poorly estimated. As a result, it would also be desirable to obtain guarantees for the (random) maximum loss
\begin{equation}\label{eq:random.loss}
\wt L_n(\mathcal{A}) = \max_{i \in [m]}  \| \beta_i - \hat\beta_{i,n} \|^2_{\Sigma}.
\end{equation}
In particular, we are able to prove the following high-probability bound on $\wt L_n(\mathcal{A})$ for \traceucb.

\begin{theorem}\label{thm_high_prob}
Let $\delta > 0$, and assume $\| \beta_i \|_2 \le Z$ for all $i$, for some $Z > 0$.
With probability at least $1 - \delta$,
\small{
\begin{equation}\label{eq_thm_highprob}
\wt L_n \le \frac{\sum\limits_{j=1}^m \sigma_j^2}{n}\Big( d+2\log\frac{3m}{\delta} \Big) \!+\! O\bigg( \frac{1}{\sigma_{\min}^2}\Big( \frac{d}{n\lambda_{\min}}  \Big)^{\frac{3}{2}} \bigg).
\end{equation}
}
%
\end{theorem}

Note that the first term in \eqref{eq_thm_highprob} corresponds to the first term of the loss for the optimal static allocation, and the second term is, again, a $n^{-3/2}$ deviation.
However, in this case, the guarantees hold \emph{simultaneously} for all the instances.

\textbf{Sketch of the proof of Theorem~\ref{thm_high_prob}.} In the proof we slightly modify the confidence ellipsoids for the $\hat\beta_{i,t}$'s, based on self-normalized martingales, and derived in \cite{abbasi2011improved}; see Thm.~\ref{th:martingale_conf_ellip} in App.~\ref{app:concentration}. By means of the confidence ellipsoids we control the loss in \eqref{eq:random.loss}.
Their radiuses depend on the number of samples per instance, and we rely on a high-probability events to compute a lower bound on the number of samples.
In addition, we need to make sure the mean norm of the contexts will not be too large (see Corollary~\ref{cor:bounded_norm_obs_gauss} in App.~\ref{app:concentration}).
Finally, we combine the lower bound on $k_{i, n}$ with the confidence ellipsoids to conclude the desired high-probability guarantees in Thm.~\ref{thm_high_prob}.

\textbf{High-Dimensional Setting.} 
High-dimensional linear models are quite common in practice, motivating the study of the $n < dm$ case, where the algorithms discussed so far break down.
We propose \sparsetraceucb in Appendix~\ref{app:high.dimensions}, an extension of \traceucb that assumes and takes advantage of \emph{joint} sparsity across the linear functions.
The algorithm has two-stages: first, an approximate support is recovered, and then, \traceucb is applied to the induced lower dimensional space.
We discuss and extend our high-probability guarantees to \sparsetraceucb under suitable standard assumptions in Appendix~\ref{app:high.dimensions}.


\section{Simulations}
\label{sims}

In this section, we provide empirical evidence to support our theoretical results. We consider both synthetic and real-world problems, and compare the performance (in terms of normalized MSE) of \traceucb to uniform sampling, optimal static allocation (which requires the knowledge of noise variances), and the context-free algorithm \varucb (see Remark~2).
We do not compare to GFSP-MAX and GAFS-MAX~\citep{antos2008active} since they are outperformed by CH-AS~\citet{carpentier2011upper} and \varucb is the same as CH-AS, except for the fact that we use the concentration inequality in Prop.~\ref{prop:sigma.concentration}, since we are estimating the variance from a regression problem using OLS.  

First, we use synthetic data to ensure that all the assumptions of our model are satisfied, namely we deal with linear regression models with Gaussian context and noise. We set the number of problem instances to $m=7$ and consider two scenarios: one in which all the noise variances are equal to $1$ and one where they are {\em not} equal, and $\sigma^2=(0.01,0.02,0.75,1,2,2,3)$.
In the latter case, $\sigma_{\max}^2 / \sigma_{\min}^2 = 300$.
We study the impact of (independently) increasing dimension $d$ and horizon $n$ on the performance, while keeping all other parameters fixed. Second, we consider real-world datasets in which the underlying model is non-linear and the contexts are not Gaussian, to observe how \traceucb behaves (relative to the baselines) in settings where its main underlying assumptions are violated. 

 \begin{figure*}[t!]
  \centering
  \subfigure[$\sigma^2 = (1, 1, 1, 1, 1, 1, 1)$.]{\includegraphics[width=0.66 \columnwidth]{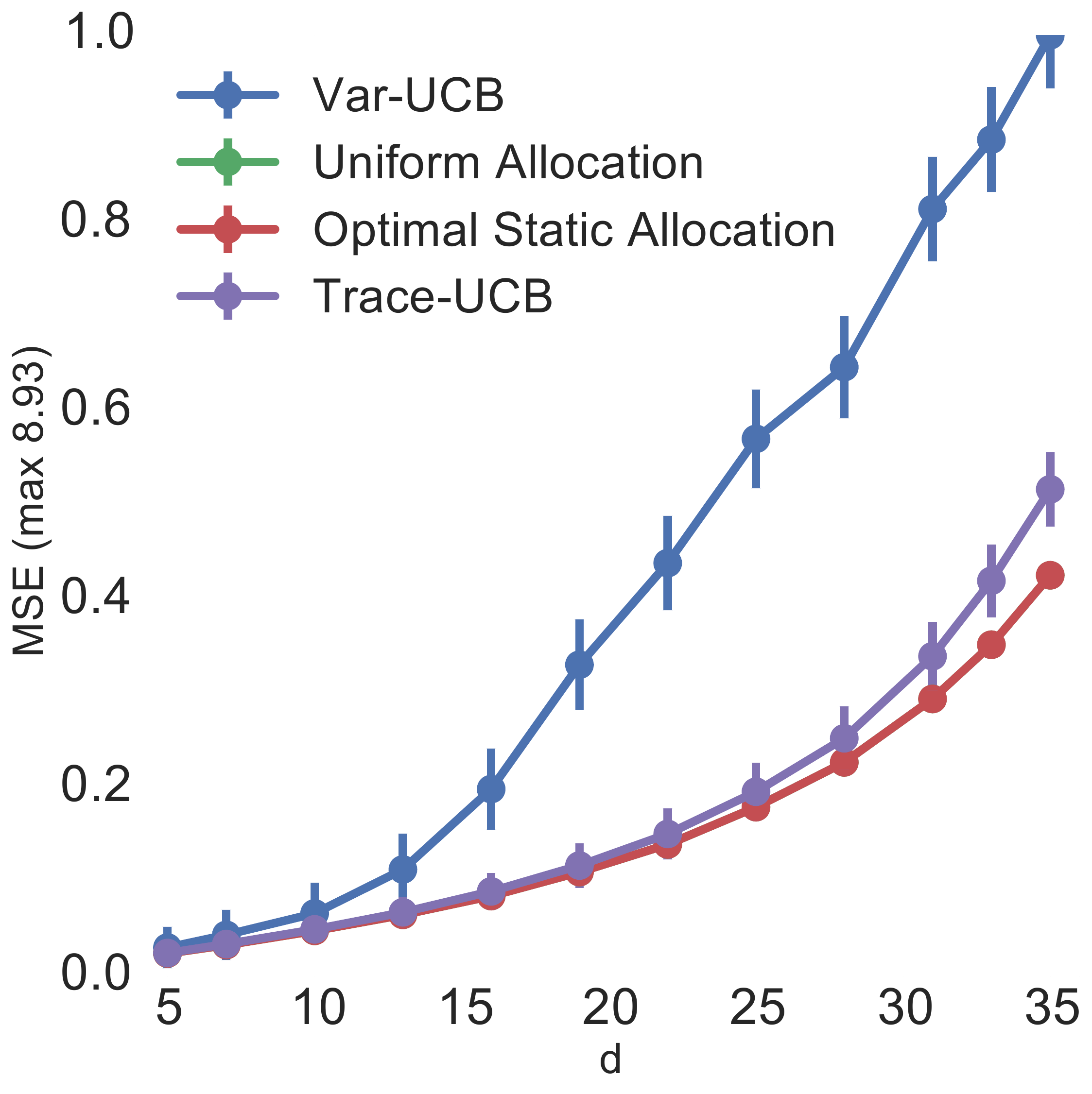}}\quad
 \subfigure[$\sigma^2 = (0.01, 0.02, 0.75, 1, 2, 2, 3)$.]{\includegraphics[width=0.66 \columnwidth]{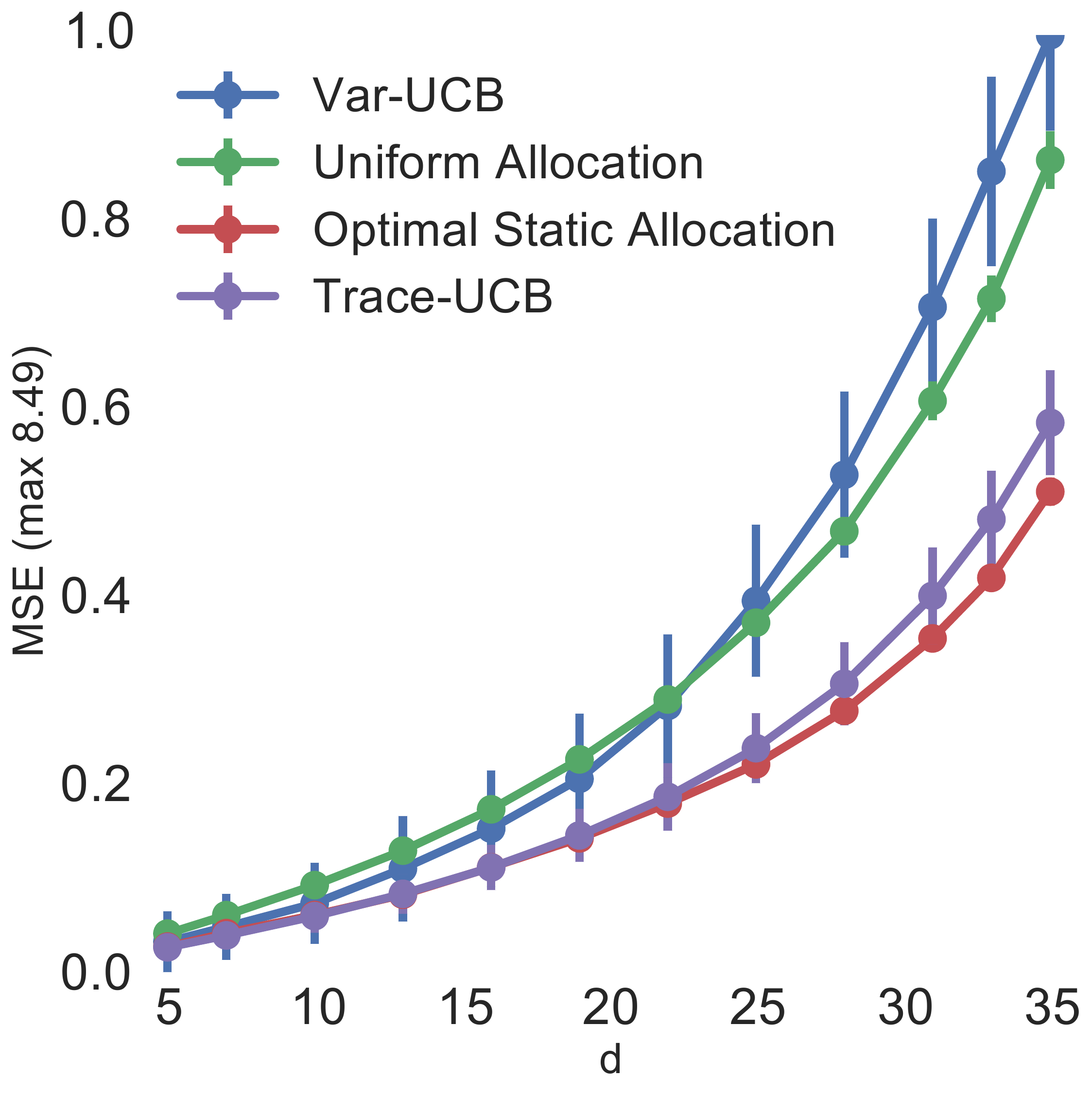}}\quad
   \subfigure[$\sigma^2 = (1, 1, 1, 1, 1, 1, 1)$.]{\includegraphics[width=0.66 \columnwidth]{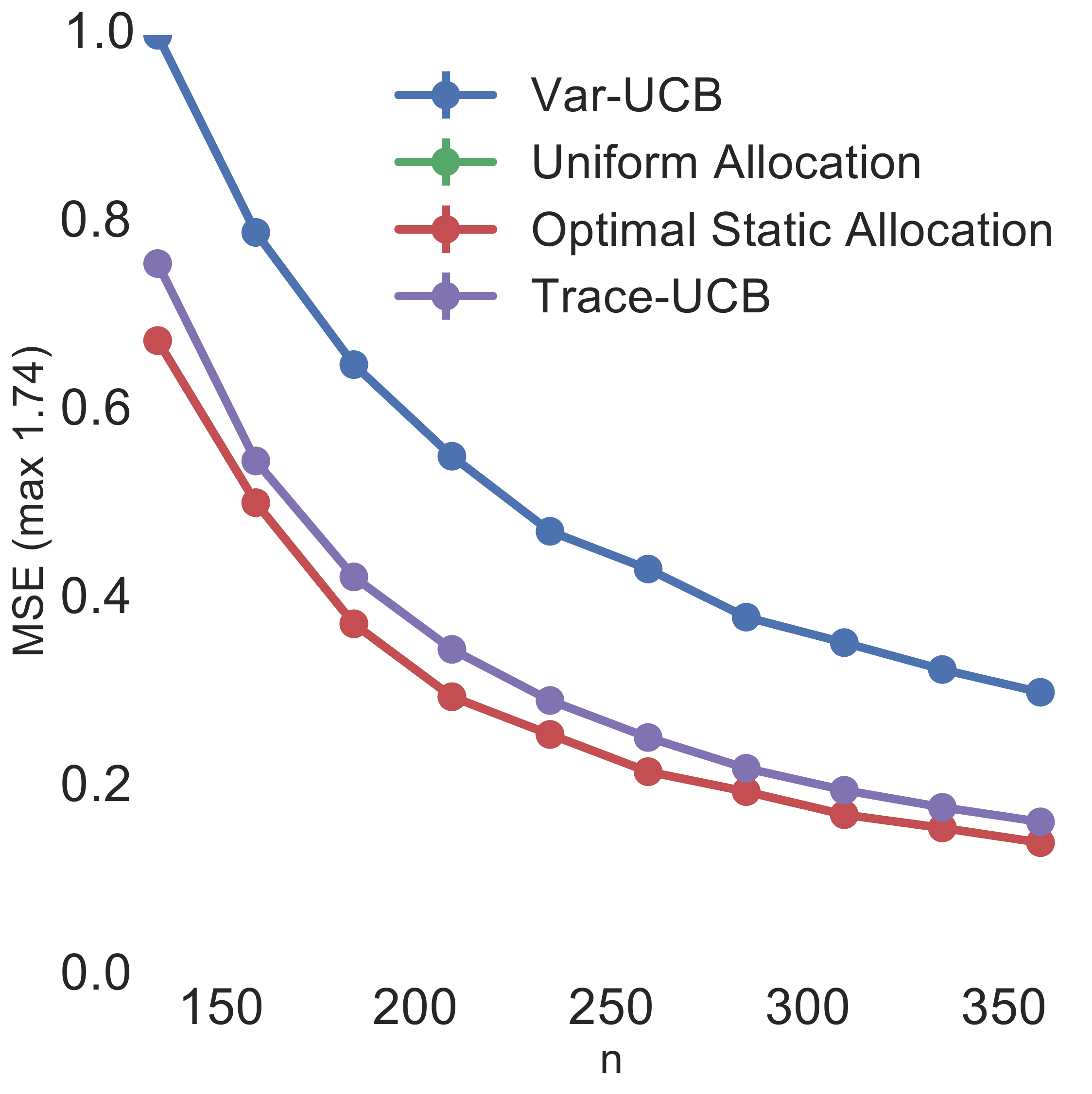}}
  \centering
 \subfigure[$\sigma^2 = (0.01, 0.02, 0.75, 1, 2, 2, 3)$.]{\includegraphics[width=0.66 \columnwidth]{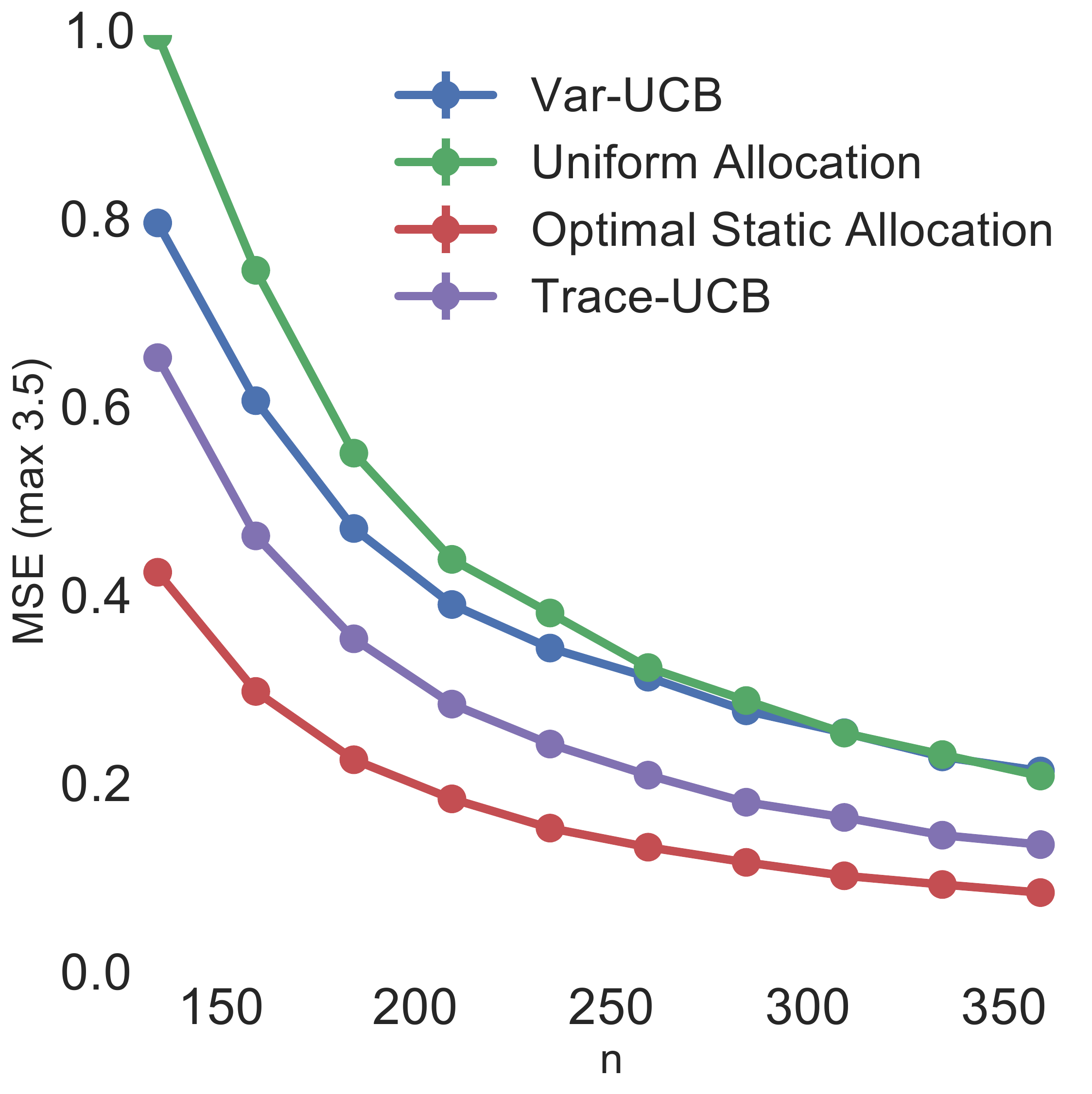}}
 \quad
   \subfigure[$\sigma^2 = (1, 1, 1, 1, 1, 1, 1)$.]{\includegraphics[width=0.66 \columnwidth]{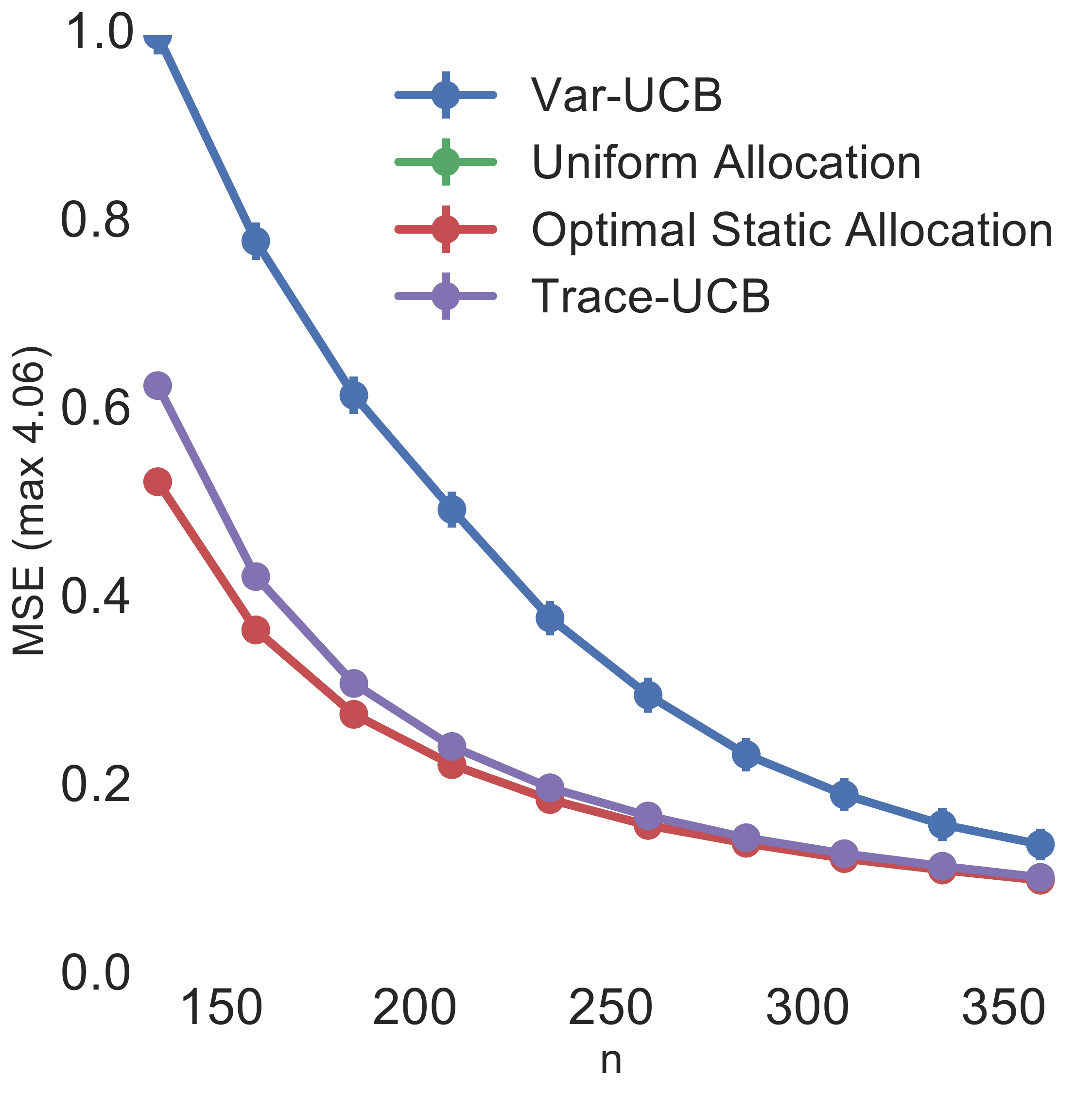}}\quad
 \subfigure[$\sigma^2 = (0.01, 0.02, 0.75, 1, 2, 2, 3)$.]{\includegraphics[width=0.66 \columnwidth]{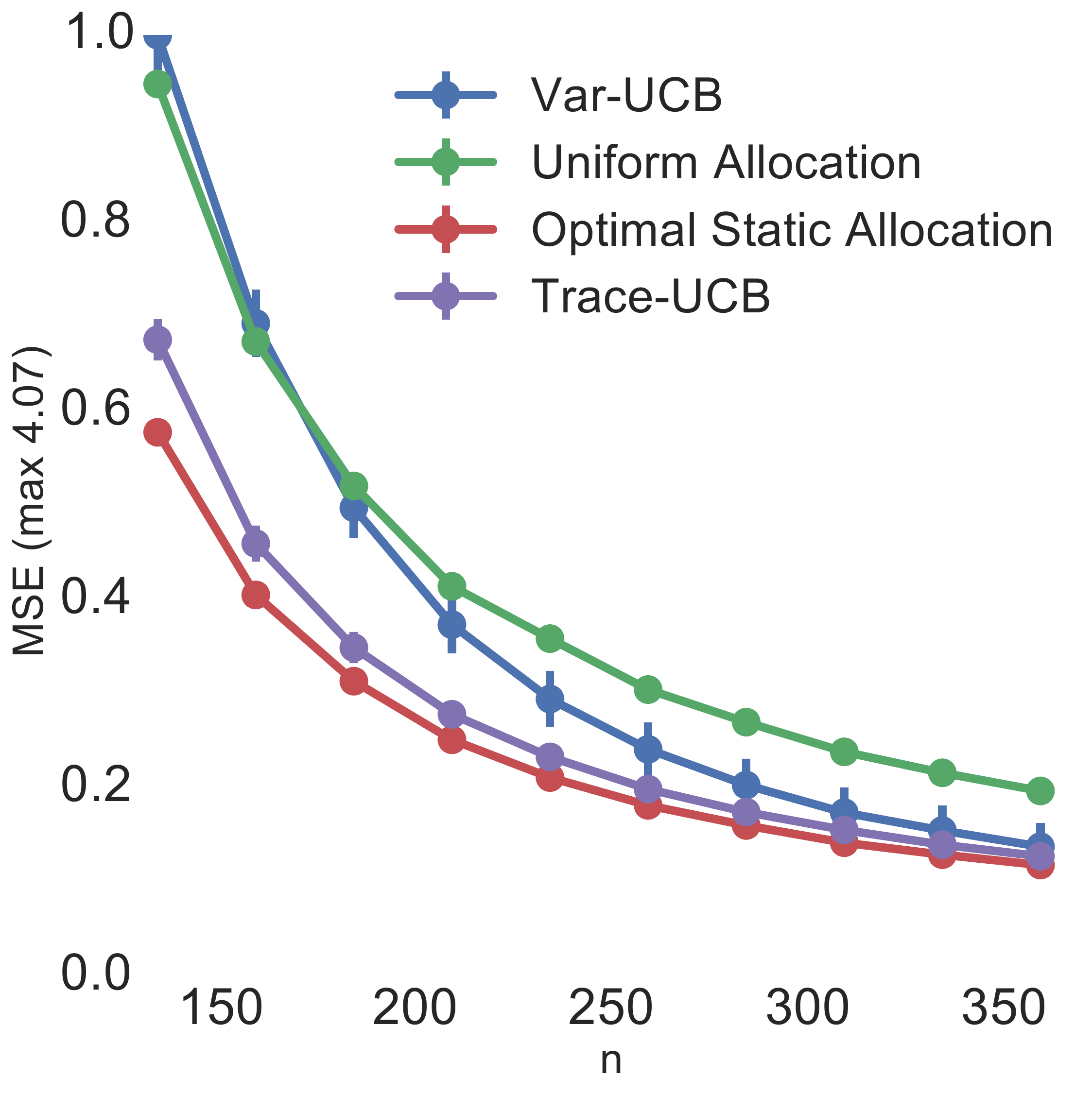}}
 \vspace{-0.15in}
\caption{White Gaussian synthetic data with $m = 7$. In Figures (a,b), we set $n = 350$. In Figures (c,d,e,f), we set $d = 10$.}
 \label{syn}
  \vspace{-0.1in}
\end{figure*}

{\bf{Synthetic Data.}} In Figures~\ref{syn}(a,b), we display the results for fixed horizon $n = 350$ and increasing dimension $d$.
For each value of $d$, we run $10,000$ simulations and report the median of the maximum error across the instances for each simulation. In Fig.~\ref{syn}(a), where $\sigma_i^2$'s are equal, uniform sampling and optimal static allocation execute the same allocation since there is no difference in the expected losses of different instances. Nonetheless we notice that \varucb suffers from poor estimation as soon as $d$ increases, while \traceucb is competitive with the optimal performance. This difference in performance can be explained by the fact that \varucb does not control for contextual balance, which becomes a dominant factor in the loss of a learning strategy for problems of high dimensionality. In Fig.~\ref{syn}(b), in which $\sigma_i^2$'s are different, uniform sampling is no longer optimal but even in this case \varucb performs better than uniform sampling only for small $d < 23$, where it is more important to control for the $\sigma_i^2$'s. For larger dimensions, balancing uniformly the contexts eventually becomes a better strategy, and uniform sampling outperforms \varucb. In this case too, \traceucb is competitive with the optimal static allocation even for large $d$, successfully balancing both noise variance and contextual error. 

\begin{figure*}[t!]
  \centering
  \subfigure{\includegraphics[width=0.66 \columnwidth]{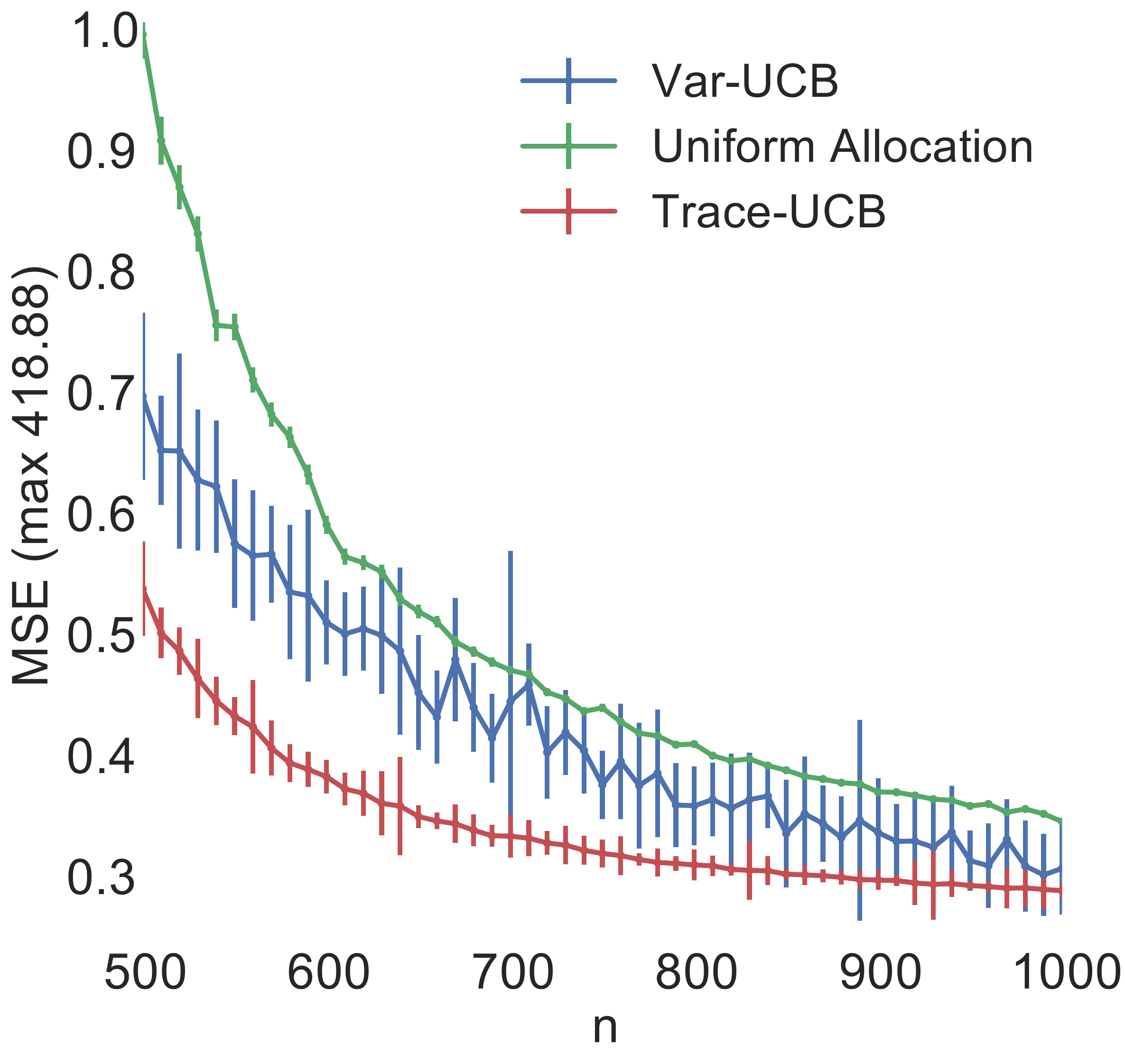}}\qquad\qquad\qquad\qquad
 \subfigure{\includegraphics[width=0.66 \columnwidth]{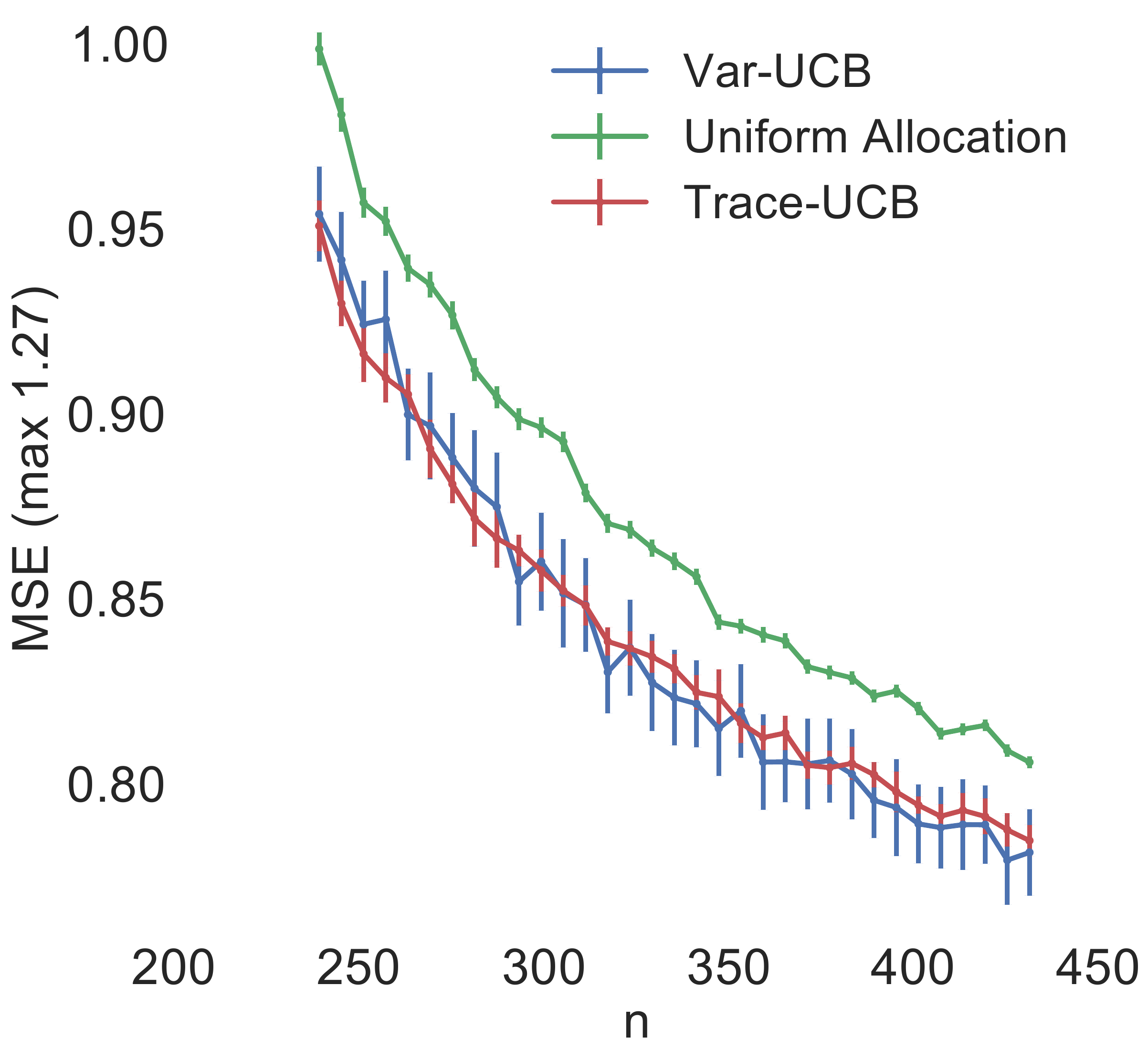}}
 \vspace{-0.15in}
\caption{Results on Jester \textit{(left)} with $d \!=\! 40, m \!=\! 10$ and MovieLens \textit{(right)} with $d \!=\! 25, m \!=\! 5$. Median over 1000 simulations.}
 \label{realdata_sim}
 \vspace{-0.075in}
\end{figure*}

Next, we study the performance of the algorithms w.r.t.\ $n$. We report two different losses, one in expectation~\eqref{eq_global_loss} and one in high probability~\eqref{eq:random.loss}, corresponding to the results we proved in Theorems~\ref{thm_exp_regret} and~\ref{thm_high_prob}, respectively. In order to approximate the loss in \eqref{eq_global_loss} (Figures~\ref{syn}(c,d)) we run $30,000$ simulations, compute the average prediction error for each instance $i \in [m]$, and finally report the maximum mean error across the instances. On the other hand, we estimate the loss in \eqref{eq:random.loss} (Figures~\ref{syn}(e,f)) by running $30,000$ simulations, taking the maximum prediction error across the instances for each simulation, and finally reporting their median.

In Figures~\ref{syn}(c, d), we display the loss for fixed dimension $d = 10$ and horizon from $n = 115$ to $360$. In Figure~\ref{syn}(c), \traceucb performs similarly to the optimal static allocation, whereas \varucb performs significantly worse, ranging from 25\% to 50\% higher errors than \traceucb, due to some catastrophic errors arising from unlucky contextual realizations for an instance.
In Fig.~\ref{syn}(d), as the number of contexts grows, uniform sampling's simple context balancing approach is enough to perform as well as \varucb that again heavily suffers from large mistakes. In both figures, \traceucb smoothly learns the $\sigma_i^2$'s and outperforms uniform sampling and \varucb. Its performance is comparable to that of the optimal static allocation, especially in the case of equal variances in Fig.~\ref{syn}(c). 

In Figure~\ref{syn}(e), \traceucb learns and properly balances observations extremely fast and obtains an almost optimal performance. Similarly to figures~\ref{syn}(a,c), \varucb struggles when variances $\hat\sigma_i^2$ are almost equal, mainly because it gets confused by random deviations in variance estimates $\hat\sigma_i^2$, while overlooking potential and harmful context imbalances. Note that even when $n = 360$ (rightmost point), its median error is still $25\%$ higher than \traceucb's. In Fig.~\ref{syn}(f), as expected, uniform sampling performs poorly, due to mismatch in variances, and only outperforms \varucb for small horizons in which uniform allocation pays off. On the other hand, \traceucb is able to successfully handle the tradeoff between learning and allocating according to variance estimates $\hat{\sigma}_i^2$, while accounting for the contextual trace $\widehat{\Sigma}_i$, even for very low $n$. We observe that for large $n$, \varucb eventually reaches the performance of the optimal static allocation and \traceucb.

In practice the loss in \eqref{eq:random.loss} (figures~\ref{syn}(e,f)) is often more relevant than \eqref{eq_global_loss}, since it is in high probability and not in expectation, and \traceucb shows excellent performance and robustness, regardless of the underlying variances $\sigma_i^2$.

{\bf{Real Data.}} \traceucb is based on assumptions such as linearity, and Gaussianity of noise and context that may not hold in practice, where data may show complex dependencies. Therefore, it is important to evaluate the algorithm with real-world data to see its robustness to the violation of its assumptions. We consider two collaborative filtering datasets in which users provide ratings for items. We choose a dense subset of $k$ users and $p$ items, where every user has rated every item. Thus, each user is represented by a $p$-dimensional vector of ratings. We define the user context by $d$ out of her $p$ ratings, and learn to predict her remaining $m=p - d$ ratings (each one is a problem instance).
All item ratings are first centered, so each item's mean is zero.
In each simulation, $n$ out of the $k$ users are selected at random to be fed to the algorithm, also in random order. Algorithms can select any instance as the dataset contains the ratings of every instance for all the users. At the end of each simulation, we compute the prediction error for each instance by using the $k - n$ users that did not participate in training for that simulation. Finally, we report the median error across all simulations. 


Fig.~\ref{realdata_sim}(a) reports the results using the Jester Dataset by~\cite{goldberg2001eigentaste} that consists of joke ratings in a continuous scale from $-10$ to $10$. We take $d=40$ joke ratings as context and learn the ratings for another $9$ jokes. In addition, we add another function that counts the total number of movies originally rated by the user. The latter is also centered, bounded to the same scale, and has higher variance (without conditioning on $X$). The number of total users is $k=3811$, and $m=10$. When the number of observations is limited, the advantage of \traceucb is quite significant (the improvement w.r.t.\ uniform allocation goes from 45\% to almost 20\% for large $n$, while w.r.t.\ \varucb it goes from almost 30\% to roughly 5\%), even though the model and context distribution are far from linear and Gaussian, respectively.


Fig.~\ref{realdata_sim}(b) shows the results for the MovieLens dataset~\cite{harper2016movielens} that consists of movie ratings between $0$ and $5$ with $0.5$ increments. We select $30$ popular movies rated by $k=1363$ users, and randomly choose $m=5$ of them to learn (so $d = 25$). In this case, all problems have similar variance ($\hat{\sigma}_{\max}^2 / \hat{\sigma}_{\min}^2 \approx 1.3$) so uniform allocation seems appropriate.
Both \traceucb and \varucb modestly improve uniform allocation, while their performance is similar.

\vspace{-0.05in}
\section{Conclusions}
\vspace{-0.05in}
We studied the problem of adaptive allocation of $n$ contextual samples of dimension $d$ to estimate $m$ linear functions equally well, under heterogenous noise levels $\sigma_i^2$ that depend on the linear instance and are unknown to the decision-maker.
We proposed \traceucb, an optimistic algorithm that successfully solves the exploration-exploitation dilemma by simultaneously learning the $\sigma_i^2$'s, allocating samples accordingly to their estimates, and balancing the contextual information across the instances.
We also provide strong theoretical guarantees for two losses of interest: in expectation and high-probability.
Simulations were conducted in several settings, with both synthetic and real data.
The favorable results suggest that \traceucb is reliable, and remarkably robust even in settings that fall outside its assumptions, thus, a useful and simple tool to implement in practice.

{\small
\textbf{\small Acknowledgements.}
\label{sec:Acknowledgements}
A.~Lazaric is supported by French Ministry of
Higher Education and Research, Nord-Pas-de-Calais Regional Council and French National Research Agency projects ExTra-Learn (n.ANR-14-CE24-0010-01).}

%
%
%

\bibliography{pure_exploration_gaussian_bandits}
\bibliographystyle{icml2017}

\onecolumn
\appendix

\newpage

\section{Optimal Static Allocation }
\label{app:optimal.static}

\subsection{Proof of Proposition~\ref{p:optimal.static}}
\label{app:optimal.static1}

\begin{proposition*}
Given $m$ linear regression problems, each characterized by a parameter $\beta_i$, Gaussian noise with variance $\sigma_i^2$, and Gaussian contexts with covariance $\Sigma$, let $n > m(d+1)$, then the optimal OLS static allocation algorithm $\mathcal{A}^*_{\text{stat}}$ selects each instance
\begin{equation}\label{eq_opt_stat_alloc_app}
k_{i,n}^* = \frac{\sigma_i^2}{\sum_j \sigma_j^2} \ n + (d+1) \ \left( 1 - \frac{\sigma_i^2}{\wb{\sigma}^2} \right),
\end{equation}
times (up to rounding effects), and incurs the global error
\begin{equation}\label{eq_opt_stat_loss_app}
L^*_n = L_n(\mathcal{A}^*_{\text{stat}}) = \wb{\sigma}^2 \frac{md}{n} + O \left( \wb{\sigma}^2 \left( \frac{md}{n} \right)^2 \right).
\end{equation}
\end{proposition*}

\begin{proof}
For the sake of readability in the following we drop the dependency on $n$.

We first derive the equality in Eq.~\ref{eq:mse}
\begin{align*}
L_i(\hat\beta_i) &= \E_{X}\big[(X^\transp \beta_i - X^\transp \hat\beta_i)^2\big] \\
&= \E_X[(\hat\beta_i - \beta_i)^\transp XX^\transp (\hat\beta_i - \beta_i)] \\
&= (\hat\beta_i - \beta_i)^\transp \E[XX^\transp] (\hat\beta_i - \beta_i) \\
&= (\hat\beta_i - \beta_i)^\transp \Sigma (\hat\beta_i - \beta_i) \\
&= \| \beta_i -\hat\beta_i\|^2_{\Sigma}.
\end{align*}
As a result, we can write the global error as
\begin{align*}
L_n(\mathcal{A}_{\text{stat}}) &= \max_{i \in [m]} \ \E_{\mathcal{D}_{i,n}} \Big[ \| \beta_i -\hat\beta_i\|^2_{\Sigma} \Big] \\
&= \max_{i \in [m]} \ \E_{\mathcal{D}_{i,n}} \bigg[ \Tr \Big( (\beta_i -\hat\beta_i)^\transp \Sigma (\beta_i -\hat\beta_i)\Big) \bigg] \\
&= \max_{i \in [m]} \ \E_{\mathcal{D}_{i,n}} \bigg[ \Tr \Big( \Sigma (\beta_i -\hat\beta_i)(\beta_i -\hat\beta_i)^\transp\Big) \bigg] \\
&= \max_{i \in [m]} \ \Tr \bigg( \E_{\mathcal{D}_{i,n}} \Big[  \Sigma(\beta_i -\hat\beta_i)(\beta_i -\hat\beta_i)^\transp\Big] \bigg),
\end{align*}
where $\mathcal{D}_{i,n}$ is the training set extracted from $\mathcal{D}_{n}$ containing the samples for instance $i$. Since contexts and noise are independent random variables, we can decompose $\mathcal{D}_{i,n}$ into the randomness related to the context matrix $\X_i \in \R^{k_i\times d}$ and the noise vector $\boldsymbol{\epsilon}_i \in \R^{k_i}$. We recall that for any fixed realization of $\X_i \in \R^{k_i\times d}$, the OLS estimates $\wh\beta_i$ is distributed as
\begin{equation}\label{eq:ols.distribution}
\hat\beta_i \mid \X_{i} \sim \mathcal{N}(\beta_i, \sigma_i^2 (\X_{i}^\transp \X_{i})^{-1}),
\end{equation}
which means that $\hat\beta_i$ conditioned on $\X_{i}$ is unbiased with covariance matrix given by $\sigma_i^2 (\X_i^\transp \X_i)^{-1}$. Thus, we can further develop $L_n(\mathcal{A}_{\text{stat}})$ as
\begin{align}\label{eq:loss.stat.allocation.step}
L_n(\mathcal{A}_{\text{stat}}) &= \max_{i \in [m]} \ \Tr \bigg( \E_{\X_i} \bigg[ \E_{\boldsymbol{\epsilon}_{i}} \Big[ \Sigma(\beta_i -\hat\beta_i)(\beta_i -\hat\beta_i)^\transp \big| \X_i\Big] \bigg]\bigg),\\
&=\max_{i \in [m]} \ \sigma_i^2 \Tr \bigg( \Sigma \E_{\X_i} \Big[ (\X_{i}^\transp \X_{i})^{-1} \Big] \bigg) \nonumber\\
&= \max_{i \in [m]} \ \sigma_i^2 \Tr \bigg( \E_{\X_i} \Big[ (\wb\X_{i}^\transp \wb\X_{i})^{-1} \Big] \bigg),\nonumber
\end{align}
where $\wb X = \Sigma^{-1/2} X$ is a whitened context and $\wb\X_i$ is its corresponding whitened matrix. Since whitened contexts $\wb X$ are distributed as $\mathcal{N}(0, I)$, we know that $(\wb\X_{i}^\transp \wb\X_{i})^{-1}$ is distributed as an inverse Wishart $\mathcal{W}^{-1}(I_d, k_i)$, whose expectation is $I_d / (k_i-d-1)$, and thus,
\begin{align}\label{eq:loss.static}
L_n(\mathcal{A}_{\text{stat}}) &= \max_{i \in [m]} \ \sigma_i^2 \Tr \left[ \frac{1}{k_i - d - 1} I_d \right] = \max_{i \in [m]} \ \frac{\sigma_i^2 \ d}{k_i - d - 1}.
\end{align}
Note that this final expression requires that $k_i > d+1$, since it is not possible to compute an OLS estimate with less than $d+1$ samples. Therefore, we proceed by minimizing Eq.~\ref{eq:loss.static}, subject to $k_i > d+1$. We write $k_i = k_i^{\prime} + d + 1$ for some $k_i^{\prime} > 0$. Thus, equivalently, we minimize
\begin{equation}\label{eq_static_loss_objective_transf}
L_n(\mathcal{A}_{\text{stat}}) = \max_i \ \frac{\sigma_i^2 \ d}{k_i^{\prime}}.
\end{equation}
Since $\sum_i k_i^\prime = n - m(d+1)$, we may conclude that the optimal $k_i^\prime$ is given by
\begin{equation*}
k_i^\prime = \frac{\sigma_i^2}{\sum_j \sigma_j^2} \ \big(n - m(d+1)\big),
\end{equation*}
so that all the terms in the RHS of Eq.~\ref{eq_static_loss_objective_transf} are equal. This gives us the optimal static allocation 
\begin{align}
k_i^* &= \frac{\sigma_i^2}{\sum_j \sigma_j^2} \left(n - m(d+1)\right) + d + 1 \nonumber \\
&= \frac{\sigma_i^2}{\sum_j \sigma_j^2} n + (d+1) \left(1 - \frac{\sigma_i^2}{\wb\sigma^2} \right), \label{eq:opt_alloc_gaussian}
\end{align}
where $\wb\sigma^2 = (1/m) \sum_i \sigma_i^2$ is the mean variance across the $m$ problem instances.

Thus, for the optimal static allocation, the expected loss is given by
\begin{align*}
L^*_n = L_n(\mathcal{A}^*_{\text{stat}}) &= d \max_i \ \frac{\sigma_i^2}{\frac{\sigma_i^2}{\sum_j \sigma_j^2} \ n - (d+1) \frac{\sigma_i^2}{\bar \sigma^2}} \\
&= \frac{\left( \sum_j \sigma_j^2 \right) d}{n - m (d+1)} \\
&= \frac{\left( \sum_j \sigma_j^2 \right) d}{n} + \frac{\left( \sum_j \sigma_j^2 \right) md(d+1)}{n\big(n - m(d+1)\big)} \\
&= \frac{\left( \sum_j \sigma_j^2 \right) d}{n} + O \left( \frac{\left( \sum_j \sigma_j^2 \right) m d^2}{n^2} \right),
\end{align*}
which concludes the proof. Furthermore the following bounds trivially holds for any $n \geq 2m(d+1)$
\begin{align*}
\frac{md \wb\sigma^2}{n} \leq L^*_n \leq 2\frac{md \wb\sigma^2}{n}.
\end{align*}

\end{proof}


%
%
%
%
%

\newpage

\section{Loss of an OLS-based Learning Algorithm (Proof of Lemma~\ref{lem:loss.learning})}
\label{app:ols.loss.learning}

Unlike in the proof of Proposition~\ref{p:optimal.static}, when the number of pulls is random \emph{and} depends on the value of the previous observations (through $\mathcal{D}_n$), then in general, the OLS estimates $\wh\beta_{i,n}$ are no longer distributed as Eq.~\ref{eq:ols.distribution} and the derivation for $\mathcal{A}_{\text{stat}}$ no longer holds. In fact, for a learning algorithm, the value $k_{i,t}$ itself provides some information about the observations that have been obtained up until time $t$ and were used by the algorithm to determine $k_{i,t}$. In the following, we show that by ignoring the current context $X_t$ when choosing instance $I_t$, we are still able to analyze the loss of \traceucb and obtain a result very similar to the static case.

We first need two auxiliary lemmas (Lemmas~\ref{lm:recurrent_form_hatsigma} and~\ref{lm:indep_hatbeta_hatsigma}), one on the computation of an empirical estimate of the variance of the noise, and an independence result between the variance estimate and the linear regression estimate.


\begin{lemma}\label{lm:recurrent_form_hatsigma}
In any linear regression problem with noise $\epsilon\sim\mathcal{N}(0,\sigma^2)$, after $t \ge d+1$ samples, given an OLS estimator $\wh\beta_{t}$, the noise variance estimator can be computed in a recurrent form as
\begin{equation}\label{eq:recurrent.variance}
\hat{\sigma}_{t+1}^2 = \frac{t-d}{t-d+1} \ \hat{\sigma}_{t}^2 + \frac{1}{t-d+1} \ \frac{(X_{t+1}^\transp \hat{\beta}_t - Y_{t+1})^2}{1 + X_{t+1}^\transp (\X_{t}^\transp \X_{t})^{-1} X_{t+1}},
\end{equation}
where $\X_t\in\R^{t\times d}$ is the sample matrix.
\end{lemma}

\begin{proof}
We first recall the ``batch'' definition of the variance estimator
\begin{align*}
\wh\sigma_{t}^2 = \frac{1}{t-d} \sum_{s=1}^t (Y_s - X_s^\transp \wh\beta_{t})^2 = \frac{1}{t-d} \| \Y_t - \X_t^\transp \wh\beta_t\|^2
\end{align*}
Since $\Y_t = \X_t \beta + \bepsilon_t$ and $\wh\beta_t = \beta + (\X_t^\transp \X_t)^{-1} \X_t^\transp \bepsilon_t$, we have
\begin{align*}
\wh\sigma_{t}^2 = \frac{1}{t-d} \| (\X_t^\transp \X_t)^{-1} \X_t^\transp \bepsilon_t - \bepsilon_t \|^2 = \frac{1}{t-d} \Big(\bepsilon_t^\transp \bepsilon_t - \bepsilon_t^\transp \X_t (\X_t^\transp \X_t)^{-1} \X_t^\transp \bepsilon_t\Big) = \frac{1}{t-d} (E_{t+1} - V_{t+1}).
\end{align*}
We now devise a recursive formulation for the two terms in the previous expression. We have
\begin{align*}
E_{t+1} = \bepsilon_{t+1}^\transp \bepsilon_{t+1} =  \bepsilon_{t}^\transp \bepsilon_{t} + \epsilon_{t+1}^2 = E_t + \epsilon_{t+1}^2.
\end{align*}
In order to analyze the second term we first introduce the design matrix $\St_t = \X_t^\transp \X_t$, which has the simple update rule $\St_{t+1} = \St_t + X_{t+1}X_{t+1}^\transp$. Then we have
\begin{align*}
V_{t+1} &= \bepsilon_{t+1}^\transp \X_{t+1} (\X_{t+1}^\transp \X_{t+1})^{-1} \X_{t+1}^\transp \bepsilon_{t+1} \\
&= \big(\bepsilon_{t}^\transp \X_{t} + \epsilon_{t+1}X_{t+1}^\transp\big) \big(\St_t + X_{t+1}X_{t+1}^\transp\big)^{-1} \big(\bepsilon_{t}^\transp \X_{t} + \epsilon_{t+1}X_{t+1}^\transp\big)^\transp\\
&= \big(\bepsilon_{t}^\transp \X_{t} + \epsilon_{t+1}X_{t+1}^\transp\big) \bigg(\St_t^{-1} - \frac{\St_t^{-1} X_{t+1} X_{t+1}^\transp \St_t^{-1} }{1 + X_{t+1}^\transp \St_t^{-1} X_{t+1} }\bigg) \big(\bepsilon_{t}^\transp \X_{t} + \epsilon_{t+1}X_{t+1}^\transp\big)^\transp,
\end{align*}
where we used the Sherman-Morrison formula in the last equality. We further develop the previous expression as
\begin{align*}
V_{t+1} &= V_t + \epsilon_{t+1}X_{t+1}^\transp \St_t^{-1} X_{t+1} \epsilon_{t+1} + 2\epsilon_{t+1}X_{t+1}^\transp \St_t^{-1} \X_{t}^\transp \bepsilon_{t} \\
& - \bepsilon_{t}^\transp \X_{t} \frac{\St_t^{-1} X_{t+1} X_{t+1}^\transp \St_t^{-1} }{1 + X_{t+1}^\transp \St_t^{-1} X_{t+1} }\X_{t}^\transp \bepsilon_{t} - \epsilon_{t+1}X_{t+1}^\transp \frac{\St_t^{-1} X_{t+1} X_{t+1}^\transp \St_t^{-1} }{1 + X_{t+1}^\transp \St_t^{-1} X_{t+1} } X_{t+1} \epsilon_{t+1} - 2 \bepsilon_{t}^\transp \X_{t} \frac{\St_t^{-1} X_{t+1} X_{t+1}^\transp \St_t^{-1} }{1 + X_{t+1}^\transp \St_t^{-1} X_{t+1} } X_{t+1} \epsilon_{t+1}.
\end{align*}
We define $\alpha_{t+1} = X_{t+1}^\transp \St_t^{-1} \X_{t}^\transp \bepsilon_{t} $ and $\psi_{t+1} = X_{t+1}^\transp \St_t^{-1} X_{t+1}$, and then obtain
\begin{align*}
V_{t+1} &= V_t + \epsilon_{t+1}^2 \psi_{t+1} + 2\alpha_{t+1}\epsilon_{t+1} - \frac{\alpha_{t+1}^2}{1+\psi_{t+1}} - \epsilon_{t+1}^2 \frac{\psi_{t+1}^2}{1-\psi_{t+1}} - 2\epsilon_{t+1}\frac{\alpha_{t+1}\psi_{t+1}}{1+\psi_{t+1}} \\
&= V_t + \epsilon_{t+1}^2 \Big(\psi_{t+1} + \frac{\psi_{t+1}^2}{1+\psi_{t+1}}\Big) + 2\epsilon_{t+1}\frac{\alpha_{t+1}}{1+\psi_{t+1}} - \frac{\alpha_{t+1}^2}{1+\psi_{t+1}}.
\end{align*}
Bringing everything together we obtain
\begin{align*}
E_{t+1} - V_{t+1} &= E_t - V_t + \epsilon_{t+1}^2\Big( 1 -\psi_{t+1} + \frac{\psi_{t+1}^2}{1+\psi_{t+1}}\Big) - 2\epsilon_{t+1}\frac{\alpha_{t+1}}{1+\psi_{t+1}} + \frac{\alpha_{t+1}^2}{1+\psi_{t+1}} \\
&= E_t - V_t + \frac{1}{1+\psi_{t+1}} \Big( \epsilon_{t+1}^2 -2\epsilon_{t+1}\alpha_{t+1} + \alpha_{t+1}\Big) = E_t - V_t + \frac{\big( \epsilon_{t+1} -\alpha_{t+1}\big)^2}{1+\psi_{t+1}} .
\end{align*}
Since $\epsilon_{t+1} = Y_{t+1} - X_{t+1}^\transp \beta$, we may write
\begin{align*}
E_{t+1} - V_{t+1} &= E_t - V_t + \frac{\big( Y_{t+1} - X_{t+1}^\transp (\beta +\St_t^{-1} \X_{t}^\transp \bepsilon_{t}) \big)^2}{1+\psi_{t+1}} = E_t - V_t + \frac{\big( Y_{t+1} - X_{t+1}^\transp \wh\beta_t \big)^2}{1+\psi_{t+1}}.
\end{align*}
Recalling the definition of the variance estimate, we finally obtain
\begin{align*}
\hat\sigma^2_{t+1} &= \frac{1}{t - d + 1} (E_{t+1}-V_{t+1}) = \frac{1}{t - d + 1} (E_{t}-V_{t}) + \frac{1}{t - d + 1} \frac{\big( Y_{t+1} - X_{t+1}^\transp \wh\beta_t \big)^2}{1+X_{t+1}^\transp \St_t^{-1} X_{t+1}} \\
&= \frac{t-d}{t - d + 1} \hat\sigma^2_t + \frac{1}{t - d + 1} \frac{\big( Y_{t+1} - X_{t+1}^\transp \wh\beta_t \big)^2}{1+X_{t+1}^\transp \St_t^{-1} X_{t+1}},
\end{align*}
which concludes the proof.
\end{proof}


\begin{lemma}\label{lm:indep_hatbeta_hatsigma}
Let $\mathcal{F}_j$ be the $\sigma$-algebra generated by $X_1, \dots, X_n$ and $\hat{\sigma}^2_1, \dots, \hat{\sigma}^2_j$.
Then, for any $j \ge d$,
\begin{equation}
\hat{\beta}_{j} \mid \mathcal{F}_j \sim \mathcal{N}(\beta, \sigma^2 \ (X_{1:j}^T X_{1:j})^{-1}).
\end{equation}
\end{lemma}

\begin{proof}
We prove the lemma by induction. The statement is true for $t = d$. We want to prove the induction, that is if $\hat{\beta}_{t} \mid \mathcal{F}_t \sim \mathcal{N}(\beta, \sigma^2 \ (\X_t^\transp \X_t)^{-1})$, then
\begin{equation}
\hat{\beta}_{t+1} \mid \mathcal{F}_{t+1} \sim \mathcal{N}(\beta, \sigma^2 (\X_{t+1}^\transp \X_{t+1})^{-1}).
\end{equation}
Let us first derive a recursive expression for $\hat{\beta}_{t+1}$. Let $\St_t = \X_t^\transp \X_t$, then
\begin{align*}
\wh\beta_{t+1} &= \beta + \St_{t+1}^{-1} \X_{t+1}^\transp \bepsilon_{t+1} = \big(\St_{t} + X_{t+1}X_{t+1}^\transp\big)^{-1} \big(\X_{t}^\transp \bepsilon_{t} + \epsilon_{t+1}X_{t+1}\big) \\
&= \bigg( \St_t^{-1} - \frac{\St_t^{-1} X_{t+1} X_{t+1}^\transp \St_t^{-1}}{1 + X_{t+1}^\transp \St_t^{-1} X_{t+1}}\bigg) \big(\X_{t}^\transp \bepsilon_{t} + \epsilon_{t+1}X_{t+1}\big),
\end{align*}
where we used Sherman-Morrison formula. By developing the previous expression we obtain
\begin{align*}
\wh\beta_{t+1} &=\big( \beta + \St_t^{-1} \X_{t}^\transp \bepsilon_{t}\big) + \epsilon_{t+1}\St_t^{-1} X_{t+1} \bigg( 1 - \frac{X_{t+1}^\transp \St_t^{-1} X_{t+1}}{1 + X_{t+1}^\transp \St_t^{-1} X_{t+1}}\bigg) - \frac{\St_t^{-1}X_{t+1}X_{t+1}^\transp \St_t^{-1} \X_{t}^\transp \bepsilon_{t}}{1 + X_{t+1}^\transp \St_t^{-1} X_{t+1}} \\
&=\wh\beta_{t} + \frac{\epsilon_{t+1}\St_t^{-1} X_{t+1}}{1 + X_{t+1}^\transp \St_t^{-1} X_{t+1}} - \frac{\St_t^{-1}X_{t+1}X_{t+1}^\transp (\wh\beta_t - \beta)}{1 + X_{t+1}^\transp \St_t^{-1} X_{t+1}}.
\end{align*}
We can conveniently rewrite the previous expression as
\begin{align}\label{eq:recursive.beta}
\wh\beta_{t+1} -\beta &=\bigg(I - \frac{\St_t^{-1}X_{t+1}X_{t+1}^\transp}{1 + X_{t+1}^\transp \St_t^{-1} X_{t+1}}\bigg) (\wh\beta_t - \beta) + \frac{\epsilon_{t+1}\St_t^{-1} X_{t+1}}{1 + X_{t+1}^\transp \St_t^{-1} X_{t+1}}\nonumber\\
&= (I - \alpha_t)(\wh\beta_t-\beta) + \gamma_t \epsilon_{t+1},
\end{align}
where $\alpha_t\in\R^{d\times d}$ and $\gamma_t\in\R^{d}$ are defined implicitly. By Lemma \ref{lm:recurrent_form_hatsigma}, we notice that the sequence of empirical variances in $\mathcal{F}_t$ is equivalent to the sequence of squared deviations up to $t$. In order to make this equivalence more apparent we define the filtration
\begin{equation*}
\mathcal{G}_t = \left\{ \{ X_s \}_{s=1}^n \cup \hat{\sigma}_2^2 \cup \{ (X_{s+1}^\transp \hat{\beta}_s - \epsilon_{s+1})^2 \}_{s=2}^{t-1} \right\},
\end{equation*}
so that $\hat{\beta}_{t+1} \mid \mathcal{F}_{t+1} \sim \hat{\beta}_{t+1} \mid \mathcal{G}_{t+1}$. We introduce two auxiliary random vectors conditioned on $\mathcal{G}$
\begin{equation*}
U = \epsilon_{t+1} - X_{t+1}^\transp (\hat{\beta}_t - \beta) \mid \mathcal{G}_{t}, \qquad V = \hat{\beta}_{t+1} - \beta \mid \mathcal{G}_{t}.
\end{equation*}
We want to show that the random vectors $U\in\R$ and $V\in\R^d$ are independent. We first recall that the noise $\epsilon_{t+1} \mid \mathcal{G}_{t} \sim \mathcal{N}(0, \sigma^2)$, and it is independent of $\epsilon_1, \dots, \epsilon_t$, and $\hat{\beta}_t$ under $\mathcal{G}_{t}$. Furthermore, by the induction assumption $\hat{\beta}_t \mid \mathcal{G}_{t}$ is also Gaussian, so we have that $(\hat{\beta}_t, \epsilon_{t+1})$ are jointly Gaussian given $\mathcal{G}_{t}$. Then we can conveniently rewrite $U$ as
\begin{equation*}
U = (-X_{t+1}, 1)^\transp(\hat\beta_t, \epsilon_{t+1}) + X_{t+1}^\transp \beta,
\end{equation*}
which shows that it is a Gaussian vector. Using the recursive formulation in Eq.~\ref{eq:recursive.beta} we can also rewrite $V$ as
\begin{equation*}
V = (\mathrm{Id} - \alpha_t)( \hat{\beta}_{t} - \beta) + \gamma_t \  \epsilon_{t+1}
=
\begin{bmatrix}
    \mathrm{I} - \alpha_t & \gamma_t
\end{bmatrix}
\begin{bmatrix}
    \hat\beta_t - \beta \\
    \epsilon_{t+1}
\end{bmatrix},
\end{equation*}
which is also Gaussian.
Furthermore, we notice that under the induction assumption, $\E_{\mathcal{G}_{t}}[U] = 0$ and $\E_{\mathcal{G}_{t}}[V] = 0$ and thus we need to show that $\E[UV \mid \mathcal{G}_{t}] = 0$ to prove that $U$ and $V$ are uncorrelated
\begin{align*}
\E[&UV \mid \mathcal{G}_{t}] 
= \E_{\mathcal{G}_{t}} \left[\left( \epsilon_{t+1} - X_{t+1}^\transp (\hat{\beta}_t - \beta) \right) \left( (\mathrm{Id} - \alpha_t)( \hat{\beta}_{t} - \beta) + \gamma_t \  \epsilon_{t+1} \right)\right] \\
&= \gamma_t \ \E_{\mathcal{G}_{t}} \left[ \epsilon_{t+1}^2 \right]
- \E_{\mathcal{G}_{t}} \left[ X_{t+1}^\transp (\hat{\beta}_t - \beta) (\mathrm{I} - \alpha_t) ( \hat{\beta}_{t} - \beta )\right] \\
&= \sigma^2 \gamma_t 
- \E_{\mathcal{G}_{t}} \left[ (\mathrm{I} - \alpha_t) ( \hat{\beta}_{t} - \beta ) (\hat{\beta}_t - \beta)^\transp X_{t+1} \right] \\
&= \sigma^2 \gamma_t 
- (\mathrm{I} - \alpha_t) \ \E_{\mathcal{G}_{t}} \left[ ( \hat{\beta}_{t} - \beta ) (\hat{\beta}_t - \beta)^\transp \right] X_{t+1} \\
&= \sigma^2 \gamma_t 
- \sigma^2 (\mathrm{I} - \alpha_t) (\X_t^\transp \X_t)^{-1} X_{t+1} \\
&= \sigma^2 \frac{\St_t^{-1} X_{t+1}}{1 + X_{t+1}^\transp \St_t^{-1} X_{t+1}}
- \sigma^2 \left(\mathrm{I} - \frac{\St_t^{-1} X_{t+1} X_{t+1}^\transp}{1 + X_{t+1}^\transp \St_t^{-1} X_{t+1}} \right) \St_t^{-1} X_{t+1} \\
&=  \sigma^2 \ \frac{\St_t^{-1} X_{t+1} - (1 + X_{t+1}^\transp \St_t^{-1} X_{t+1}) \St_t^{-1} X_{t+1} + \St_t^{-1} X_{t+1} X_{t+1}^\transp \St_t^{-1} X_{t+1} }{1 + X_{t+1}^\transp \St_t^{-1} X_{t+1}} \nonumber \\
&= 0.
\end{align*}
It thus follows that, as $U$ and $V$ are uncorrelated, they are also independent. Combining the definition of $\mathcal{G}_t$, $U$ and its independence w.r.t\ $V$, we have
\begin{align*}
V \mid \mathcal{G}_{j+1} &= V \mid U, \mathcal{G}_{j} \\
&= V \mid \{ X_1, \dots, X_T, \hat\sigma_2^2, \{ (X_{s+1}^\transp \hat{\beta}_s - \epsilon_{s+1})^2 \}_{s=2}^{t-1} \} \\
&=
\begin{bmatrix}
    \mathrm{I} - \alpha_t & \gamma_t
\end{bmatrix}
\begin{bmatrix}
    \hat\beta_t - \beta \\
    \epsilon_{t+1}
\end{bmatrix}
\mid \mathcal{G}_{t}.
\end{align*}
By the induction hypothesis the vector in the previous expression is distributed as
\begin{equation*}
\begin{bmatrix}
    \hat\beta_t - \beta \\
    \epsilon_{t+1}
\end{bmatrix}
\sim \mathcal{N}\left(
\begin{bmatrix}
    0 \\
    0
\end{bmatrix},
\sigma^2
\begin{bmatrix}
    S_t^{-1} & 0 \\
    0 & 1
\end{bmatrix}
\right).
\end{equation*}
Therefore, we conclude that
\begin{equation*}
V \mid \mathcal{G}_{t+1}
\sim \mathcal{N}\left(0, \sigma^2
\begin{bmatrix}
  \mathrm{I} - \alpha_t & \gamma_t
\end{bmatrix}
\begin{bmatrix}
    S_t^{-1} & 0 \\
    0 & 1
\end{bmatrix}
\begin{bmatrix}
  \mathrm{I} - \alpha_t & \gamma_t
\end{bmatrix}^\transp
\right) =
\mathcal{N}(0, \sigma^2 \ \Sigma^\prime),
\end{equation*}
where the covariance matrix $\Sigma'$ can be written as
\begin{align*}
\Sigma^\prime &=
\begin{bmatrix}
 \mathrm{I} - \alpha_t & \gamma_t
\end{bmatrix}
\begin{bmatrix}
    S_t^{-1} & 0 \\
    0 & 1
\end{bmatrix}
\begin{bmatrix}
   \mathrm{I} - \alpha_t & \gamma_t
\end{bmatrix}^\transp
\\
&=
\begin{bmatrix}
   \mathrm{I} - \alpha_t & \gamma_t
\end{bmatrix}
\begin{bmatrix}
    S_t^{-1} (\mathrm{I} - \alpha_t)^\transp \\
    \gamma_t^\transp
\end{bmatrix} \\
&=
(\mathrm{I} - \alpha_t) S_t^{-1} (\mathrm{I} - \alpha_t)^\transp +  \gamma_t \gamma_t^\transp.
\end{align*}
Recalling the definitions of $\alpha_t$ and $\gamma_t$, and defining $\psi_{t+1} = X_{t+1}^\transp \St_t^{-1} X_{t+1}$
\begin{align*}
\Sigma^\prime =& \left(\mathrm{I} - \frac{\St_t^{-1} X_{t+1} X_{t+1}^\transp}{1 + X_{t+1}^\transp \St_t^{-1} X_{t+1}} \right) \St_t^{-1} \left(\mathrm{I} - \frac{\St_t^{-1} X_{t+1} X_{t+1}^\transp}{1 + X_{t+1}^\transp \St_t^{-1} X_{t+1}} \right)^\transp \\
&+  \left( \frac{\St_t^{-1} X_{t+1}}{1 + X_{t+1}^\transp \St_t^{-1} X_{t+1}} \right) \left( \frac{\St_t^{-1} X_{t+1}}{1 + X_{t+1}^\transp \St_t^{-1} X_{t+1}} \right)^\transp \\
&= \St_t^{-1} - 2 \frac{\St_t^{-1} X_{t+1} X_{t+1}^\transp \St_t^{-1}}{1 + r} + \psi_{t+1} \frac{\St_t^{-1} X_{t+1} X_{t+1}^\transp \St_t^{-1}}{(1 + \psi_{t+1})^2} + \frac{\St_t^{-1} X_{t+1} X_{t+1}^\transp \St_t^{-1}}{(1 + \psi_{t+1})^2} \nonumber \\
&= \St_t^{-1} - \frac{\St_t^{-1} X_{t+1} X_{t+1}^\transp \St_t^{-1}}{1 + \psi_{t+1}} = S_{t+1}^{-1} = (\X_{t+1}^\transp \X_{t+1})^{-1},
\end{align*}
where we applied the Woodbury matrix identity in the last step. Finally, it follows that
\begin{equation*}
\hat{\beta}_{t+1} \mid \mathcal{F}_{t+1} \sim \mathcal{N}(\beta, \sigma^2 \ (\X_{t+1}^\transp \X_{t+1})^{-1}),
\end{equation*}
and the induction is complete. 
\end{proof}


Now we can prove Lemma~\ref{lem:loss.learning}:
\begin{lemma*}
Let $\mathcal{A}$ be a learning algorithm that selects instances $I_t$ as a function of the previous history, that is, $\mathcal{D}_{t-1} = \{X_1, I_1, Y_{I_1, 1}, \ldots, X_{t-1}, I_{t-1}, Y_{I_{t-1},t-1}\}$ and computes estimates $\wh\beta_{i,n}$ using OLS. Then, its loss after $n$ steps can be expressed as
\begin{align}
L_n(\mathcal{A}) = \max_{i\in[m]}\ \E_{\mathcal{D}_t}\bigg[ \frac{\sigma_i^2}{k_{i,n}} \Tr\Big(\Sigma \wh\Sigma_{i,n}^{-1}\Big)\bigg],
\end{align}
where $k_{i,n} = \sum_{t=1}^n \mathbb{I}\{I_t=i\}$ and $\wh\Sigma_{i,n} = \X_{i,n}^\transp \X_{i,n} / k_{i,n}$.
\end{lemma*}

\begin{proof}
For any instance $i$, we can assume that the following random variables are sampled before \traceucb starts collecting observations (we omit the $i$ index in the table):
\begin{center}
\begin{tabular}{ c c c c }
  \hline			
  $ t = 1$ & $ t = 2$ & $\dots$ & $t = n$ \\
  \hline
  $X_1$ & $X_2$ & $\dots$ & $X_n$ \\
  $\epsilon_1$ & $\epsilon_2$ & $\dots$ & $\epsilon_n$ \\
  $\hat{\beta}_1$ & $\hat{\beta}_2$ & $\dots$ & $\hat{\beta}_n$ \\
  $\hat{\sigma}^2_1$ & $\hat{\sigma}^2_2$ & $\dots$ & $\hat{\sigma}^2_n$ \\
  \hline  
\end{tabular}
\end{center}
As a result, we can interpret \traceucb as controlling the \emph{stopping} time $t_i = k_{i,n}$ for each problem $i$, that is, the total number of samples $k_{i,n}$, leading to the final estimates $\hat{\beta}_{t_i}$ and $\hat{\sigma}_{t_i}^2$. In the following we introduce the notation $\X_{1:j}$ as the sample matrix constructed from exactly $j$ samples, unlike $\X_{i,n}$ which is the sample matrix obtained with $k_{i,n}$. So we have $\X_{1:k_{i,n}} = \X_{i,n}$. Crucially, when the errors $\epsilon_j$ are Gaussian, then $\hat{\beta}_j \mid \X_{1:j}$ and $\hat{\sigma}^2_j \mid \X_{1:j}$ are independent for any fixed $j$ (note these random variables have nothing to do with the algorithm's decisions).

Let $\mathcal{F}_j$ be the $\sigma$-algebra generated by $X_1, \dots, X_n$ and $\hat{\sigma}^2_1, \dots, \hat{\sigma}^2_j$. We recall that from Lemma~\ref{lm:indep_hatbeta_hatsigma} 
\begin{equation}
\wh\beta_j | \X_{1:j} = \hat{\beta}_{j} \mid \mathcal{F}_j \sim \mathcal{N}(\beta_j, \sigma^2 \ (X_{1:j}^T X_{1:j})^{-1}).
\end{equation}
Intuitively, this results says that, given the data $\X_{1:n}$, if we are additionally given all the estimates for the variance $\{\hat{\sigma}^2_s\}_{s=2}^{j}$ ---which obviously depend on $\epsilon_1,\ldots,\epsilon_j$---, then the \textit{updated} distribution for $\hat{\beta}_{j}$ does not change at all. This is a crucial property since \traceucb ignores the current context $X_t$ and it makes decisions only based on previous contexts and the variance estimates $\{\hat{\sigma}^2_s\}_{s=2}^{j}$, thus allowing us to proceed and do inference on $\hat{\beta}_{j}$ as in the fixed allocation case.

We now need to take into consideration the filtration $\F_{i,j}$ for a specific instance $i$ and the \textit{environment} filtration $\mathcal{E}_{-i}$ containing all the contexts $X$ and noise $\epsilon$ from all other instances (different from $i$). Since the environment filtration $\mathcal{E}_{-i}$ is independent from the samples from instance $i$, then we can still apply Lemma~\ref{lm:indep_hatbeta_hatsigma} and obtain
\begin{equation}
\wh\beta_{i,j} \mid \mathcal{F}_{i,j}, \mathcal{E}_{-i} \sim \wh\beta_{i,j} \mid \mathcal{F}_{i,j}.
\end{equation}
Now we can finally study the expected prediction error
\begin{align}
L_{i,n}(\wh\beta_{i,n}) = &\E[(\hat\beta_i - \beta_i) (\hat\beta_i - \beta_i)^\transp] \notag\\
&= \E_{\X_{1:n}, \varepsilon_{-i}} \left[ \E[(\hat\beta_i - \beta_i) (\hat\beta_i - \beta_i)^\transp \mid \X_{1:n}, \varepsilon_{-i}] \right] \notag\\
&= \E_{\X_{1:n}, \varepsilon_{-i}} \left[ \sum_{j = 1}^n \E[(\hat\beta_{k_i} - \beta_i) (\hat\beta_{k_i} - \beta_i)^\transp \mid \X_{1:n}, \varepsilon_{-i}, k_i = j] \ \p(k_i = j) \right] \notag\\
&= \E_{\X_{1:n}, \varepsilon_{-i}} \left[ \sum_{j = 1}^n \E\left[ \E_{\mathcal{F}_j}[ (\hat\beta_j - \beta_i) (\hat\beta_j - \beta_i)^\transp \mid \mathcal{F}_j, \X_{1:n}, \varepsilon_{-i}, k_i = j ] \mid \X_{1:n}, \varepsilon_{-i}, k_i = j \right] \ \p(k_i = j) \right] \nonumber \notag\\
&= \E_{\X_{1:n}, \varepsilon_{-i}} \left[ \sum_{j = 1}^n \E\left[ \E_{\mathcal{F}_j}[ (\hat\beta_j - \beta_i) (\hat\beta_j - \beta_i)^\transp \mid \mathcal{F}_j, \X_{1:n} ] \mid \X_{1:n}, \varepsilon_{-i}, k_i = j \right] \ \p(k_i = j) \right] \label{eq:lemma_indep_applied} \\
&= \E_{\X_{1:n}, \varepsilon_{-i}} \left[ \sum_{j = 1}^n \E\left[ \sigma_i^2 (\X_{1:j}^\transp \X_{1:j})^{-1} \mid \X_{1:n}, k_i = j \right] \ \p(k_i = j) \right] \notag\\
&= \E_{\X_{1:n}, \varepsilon_{-i}} \left[ \sum_{j = 1}^n \sigma_i^2 (\X_{1:j}^\transp \X_{1:j})^{-1}  \ \p(k_i = j) \right]  \notag\\
&= \sigma_i^2 \ \E_{\X_{1:n}, \varepsilon_{-i}} \left[ \E_{k_i} [ (\X_{1:k_i}^\transp \X_{1:k_i})^{-1} ]  \right] \notag\\
&= \sigma_i^2 \ \E \left[ (\X_{1:k_i}^\transp \X_{1:k_i})^{-1} \right],\notag
\end{align}
where in Eq.~\ref{eq:lemma_indep_applied} we applied Lemma~\ref{lm:indep_hatbeta_hatsigma}. Hence, going back to the definition of loss (see e.g., Eq.~\ref{eq:loss.stat.allocation.step}), we obtain an expression for the loss which applies under \traceucb (while not in general for other algorithms)
\begin{align*}
L_n(\mathcal{A}) &= \max_i \ \E \left[ \sigma_i^2 \ \mathrm{Tr}(\Sigma (\X_{i,n}^\transp \X_{i,n})^{-1} ) \right] \\
&= \max_i \ \E \left[ \frac{\sigma_i^2}{k_{i,n}}  \mathrm{Tr}\left( \Sigma \hat\Sigma^{-1}_{i,n} \right) \right]. 
\end{align*}
%
\end{proof}

\newpage

\section{Concentration Inequalities (Proofs of Propositions~\ref{prop:sigma.concentration} and~\ref{prop:trace.concentration})}
\label{app:concentration}

In the next two subsections, we prove Propositions~\ref{prop:sigma.concentration} and~\ref{prop:trace.concentration}, respectively.
In addition, we also show a confidence ellipsoid result for the $\wh \beta$ estimates, and a concentration inequality for the norm of the observations $X_t$.

\subsection{Concentration Inequality for the Variance (Proof of Proposition~\ref{prop:sigma.concentration})}
\label{app:concentration1}

We use the following concentration inequality for sub-exponential random variables. 

\begin{proposition}\label{pro:bernstein}
Let $X$ be a mean-zero $(\tau^2, b)$-subexponential random variable. Then, for all $\eta > 0$,
\begin{equation}\label{pro:bernstein_lemma}
\mathbb{P}(|X| \ge \eta) \le \exp \left( - \frac{1}{2} \min\left\{ \frac{\eta^2}{\tau^2}, \frac{\eta}{b} \right\} \right).
\end{equation}
\end{proposition}
\begin{proof}
See Proposition 2.2 in \cite{wwright2015highdim}.
\end{proof}

We first prove the concentration inequality for one single instance. \\

\begin{proposition}\label{lm:conf_int_sigma}
Let $t > d$, $\X_t \in \R^{t \times d}$ be a random matrix whose entries are independent standard normal random variables, $\Y_t = \X_t^\transp \beta + \bepsilon_t$, where the noise $\bepsilon_t \sim \mathcal{N}(0, \sigma_t^2 \ I_d)$ is independent from $\X_t$, and $\delta \in (0, 3/4]$. Then, with probability at least $1 - \delta$, we have
\begin{equation}\label{eq:sigma.concentration}
| \hat{\sigma}_t^2 - \sigma^2 | \le \sigma^2 \sqrt{\frac{64}{t - d} \left( \log \frac{1}{\delta} \right)^2},
\end{equation}
where $\hat{\sigma}_t^2$ is the unbiased estimate $\wh\sigma_{t}^2 = \frac{1}{t-d} \| \Y_t - \X_t \wh\beta_t\|^2$ and $\wh\beta_t$ is the OLS estimator of $\beta$, given $\X_t$ and $\Y_t$.
\end{proposition}

\begin{proof}
First note that the distribution of $\hat \sigma_t^2$ conditioned on $\X_t$ follows the scaled chi-squared distribution, i.e.,
\begin{equation*}
    \hat \sigma^2_t \mid \X \sim \frac{\sigma^2}{t - d} \ \chi^2_{t - d}.
\end{equation*}
Also note that the distribution of the estimate does \emph{not} depend on $\X_t$ and we can integrate out the randomness in $\X_t$. In order to show concentration around the mean, we directly use the sub-exponential properties of $\hat \sigma^2_t$. The $\chi^2_k$ distribution is sub-exponential with parameters $(4k, 4)$.\footnote{See Example 2.5 in \cite{wwright2015highdim}.} Furthermore, we know that for any constant $C > 0$, $C\chi^2_k$ is $(4C^2k, 4C)$-sub-exponential. As a result, we have that $\hat \sigma^2_t$ is subexponential with parameters 
\begin{equation*}
(\tau^2, b) = \left( \frac{4\sigma^4}{t-d}, \frac{4\sigma^2}{t-d} \right).
\end{equation*}
Now we use Proposition~\ref{pro:bernstein} as our concentration bound. In our case, $\eta^2/\tau^2 < \eta / b$, when $\eta < \sigma^2$. In such a case, if we denote the RHS of \eqref{pro:bernstein_lemma} by $\delta$, we conclude that
\begin{equation*}
\eta = \sigma^2 \sqrt{\frac{8}{t - d} \log \frac{1}{\delta}}.
\end{equation*}
Then, $\eta < \sigma^2$ holds when $t \ge d + 8 \log(1/\delta)$. Otherwise, if $\eta^2 / \tau^2 > \eta / b$, by Eq.~\ref{pro:bernstein_lemma}, we have
\begin{equation*}
\eta = {\frac{8 \sigma^2}{t - d} \log \frac{1}{\delta}}.
\end{equation*}
In this case, when $t < d + 8 \log(1/\delta)$, we have that
\begin{equation*}
| \hat{\sigma}^2_t - \sigma^2 | \le \sigma^2 {\frac{8}{t - d} \log \frac{1}{\delta}}.
\end{equation*}
We would like to derive a bound that is valid in both cases. Let $x = 8 \log(1/\delta) / (t -d)$, then we have
\begin{equation}\label{eq_bernstein_max_version}
\mathbb{P} \left( | \hat{\sigma}^2_t - \sigma^2 | \geq \sigma^2 \max(x, \sqrt{x}) \right) \le \delta.
\end{equation}
Suppose $x \ge \sqrt{x}$, so $t < d + \log(1/\delta)$. Then, we would like to find $C$, such that $x \le C \sqrt{x}$. As $t \ge d+1$, we see that
\begin{equation*}
\sqrt{x} = \sqrt{\frac{8 \log(1/\delta)}{t - d}} \le \sqrt{8 \log(1/\delta)} \stackrel{\Delta}{=} C.
\end{equation*}
if $C>1$, it does follow that $\max(x, \sqrt{x}) < \max(C \sqrt{x}, \sqrt{x}) < \sqrt{8\log(1/\delta)x}$, which corresponds to $\delta < 0.88$. By~\eqref{eq_bernstein_max_version}, we now conclude that
\begin{equation*}
\mathbb{P} \left( | \hat{\sigma}^2_t - \sigma^2 | \geq \sigma^2 \sqrt{\frac{64}{t - d} \left( \log \frac{1}{\delta} \right)^2} \right) \le \delta,
\end{equation*}
and the proof is complete.
\end{proof}

In order to prove Proposition~\ref{prop:sigma.concentration}, we are just left to apply a union bound over steps $t\in\{1,\ldots,n\}$ and instances $i\in\{1,\ldots,m\}$. In order to avoid confusion, let $\hat\sigma_{i,t}$ be the estimate obtained by the algorithm after $t$ steps and $\hat\sigma_{i}(j)$ the estimate obtained using $j$ samples. Let $j > d$, then
\begin{equation*}
\mathcal{E}_{i}(j) = \bigg\{| \hat{\sigma}_{i}^2(j) - \sigma_i^2 | \ge \sigma^2_i \sqrt{\frac{64}{j - d} \left( \log \frac{1}{\delta} \right)^2} \bigg\},
\end{equation*}
is the high-probability event introduced in Proposition~\ref{lm:conf_int_sigma}, which holds with probability $1-\delta$. Then we have that the event
\begin{equation*}
\mathcal{E} = \bigcap_{i=1}^m \bigcap_{j=1}^n \mathcal{E}_{i}(j),
\end{equation*}
holds with probability $1-\delta'$, with $\delta' = mn\delta$. We complete the proof of Proposition~\ref{prop:sigma.concentration} by properly tuning $\delta$ and taking $R \geq \max_i \sigma_i^2$.
Recall that Proposition~\ref{prop:sigma.concentration} is as follows.
\begin{proposition*}
Let the number of pulls $k_{i,t} \geq d+1$ and $R \ge \max_i \sigma_i^2$. If $\delta\in(0,3/4)$, then for any instance $i$ and step $t > m(d+1)$, with probability at least $1 - \frac{\delta}{2}$, we have 
\begin{equation}
|\hat{\sigma}_{i,t}^2 - \sigma_i^2 | \le \Delta_{i,t} \stackrel{\Delta}{=} R \sqrt{\frac{64}{k_{i,t} - d} \left( \log \frac{2mn}{\delta} \right)^2 }.
\end{equation}
\end{proposition*}


\subsection{Concentration Inequality for the Trace (Proof of Proposition~\ref{prop:trace.concentration})}
\label{app:concentration2}

We first recall some basic definitions. For any matrix $\A \in \R^{n \times d}$, the $i$-th singular value $s_i(\A)$ is equivalent to $s_i(\A)^2 = \lambda_i(\A^\transp \A)$, where $\lambda_i$ is the $i$-th eigenvalue. The smallest and largest singular values $s_{\min}$ and $s_{\max}$ satisfy
\begin{equation*}
s_{\min} \ \| x \|_2 \le \| \A x \|_2 \le s_{\max} \ \| x \|_2 \qquad \text{ for all } x \in \R^d.
\end{equation*}
The extreme singular values measure the maximum and minimum distortion of points and their distance when going from $\R^d$ to $\R^n$ via the linear operator $\A$. We also recall that the spectral norm of $\A$ is given by
\begin{equation*}
\| \A \| = \sup_{x \in \R^d \backslash 0} \frac{\| \A x \|_2}{\| x \|_2} = \sup_{x \in S^{n-1}} \| \A x \|_2,
\end{equation*}
and thus, $s_{\max}(\A) = \| \A \|$ and $s_{\min}(\A) =1 / \| \A^{-1} \|$, if $\A$ is invertible. 

We report the following concentration inequality for the eigenvalues of random Gaussian matrices.

\begin{proposition}\label{prop:trace_inv_covariance}
Let $n \ge d$, $\wb\X \in \R^{n \times d}$ be a random matrix whose entries are independent standard normal random variables, and $\wb\Sigma = \wb\X^\transp \wb\X / n$ be the corresponding empirical covariance matrix. Let $\alpha > 0$, then with probability at least $1 - 2\exp(-\alpha^2 d/2)$, we have
\begin{equation*}
\Tr \left( \wb{\Sigma}^{-1} \right) \ge d \left( 1 - \frac{2(1+\alpha)\sqrt{{d}} + (1+\alpha)^2 {d}/\sqrt{n}}{\sqrt{n} + 2(1+\alpha)\sqrt{{d}} + (1+\alpha)^2 {d}/\sqrt{n}} \right),
\end{equation*}
and
\begin{equation*}
\Tr \left( \wb{\Sigma}^{-1} \right) \le d \left( 1 + \frac{2(1+\alpha)\sqrt{{d}} - (1+\alpha)^2 {d}/\sqrt{n}}{\sqrt{n} - 2(1+\alpha)\sqrt{{d}} + (1+\alpha)^2 {d}/\sqrt{n}} \right).
\end{equation*}
In particular, we have 
\begin{equation}
d \Big( 1-(1+\alpha)\sqrt{\frac{d}{n}} \Big)^2 \le \Tr \left( \wb{\Sigma}^{-1} \right) \le d \Big( 1+2(1+\alpha)\sqrt{\frac{d}{n}} \Big)^2. \nonumber
\end{equation}
\end{proposition}

\begin{proof}
We first derive the concentration inequality for the eigenvalues of the empirical covariance matrix and then we invert it to obtain the guarantee for the inverse matrix. From Corollary 5.35 in~\cite{vershynin2010introduction}, we have that for any $t > 0$
\begin{equation}\label{eq:concentration.covariance}
\left( \sqrt{n} - \sqrt{d} - t \right)^2 \le \lambda_{\min}(\wb\X^\transp\wb\X) = s_{\min}(\wb\X)^2 \le s_{\max}(\wb\X)^2 = \lambda_{\max}(\wb\X^\transp\wb\X) \le \left( \sqrt{n} + \sqrt{d} + t \right)^2,
\end{equation}
with probability at least $1 - 2\exp(-t^2/2)$.
Let $\alpha > 0$ and take $t = \alpha \sqrt{d}$, then with probability at least $1 - 2\exp(-\alpha^2 d/2)$, we obtain the desired statement
\begin{equation*}
\left( 1 - (1+\alpha) \sqrt{\frac{d}{n}} \right)^2 \le \lambda_{\min} \left( \wb\Sigma \right) \le \lambda_{\max} \left( \wb\Sigma \right) \le \left( 1 + (1 + \alpha) \sqrt{\frac{d}{n}} \right)^2.
\end{equation*}
We now proceed by studying the eigenvalues of the inverse of the empirical covariance matrix $\lambda_{\min} ( \wb{\Sigma}^{-1} ) = 1/\lambda_{\max}( \wb{\Sigma} )$ and $\lambda_{\max} ( \wb{\Sigma}^{-1} ) = 1/\lambda_{\min} ( \wb{\Sigma} )$.
Combined with Eq.~\ref{eq:concentration.covariance} we have
\begin{align*}
\lambda_{\min} \left( \wb{\Sigma}^{-1} \right) &\ge \frac{1}{\left( 1 + (1 + \alpha) \sqrt{\frac{d}{n}} \right)^2} \\
&= \frac{1}{1 + 2(1+\alpha)\sqrt{\frac{d}{n}} + (1+\alpha)^2 \frac{d}{n}} \\
&= 1 - \frac{2(1+\alpha)\sqrt{\frac{d}{n}} + (1+\alpha)^2 \frac{d}{n}}{1 + 2(1+\alpha)\sqrt{\frac{d}{n}} + (1+\alpha)^2 \frac{d}{n}}.
\end{align*}
Similarly, we have that
\begin{align*}
\lambda_{\max} \left( \wb{\Sigma}^{-1} \right) &\le \frac{1}{\left( 1 - (1+\alpha) \sqrt{\frac{d}{n}} \right)^2} \\
&= \frac{1}{1 - 2(1+\alpha)\sqrt{\frac{d}{n}} + (1+\alpha)^2 \frac{d}{n}} \\
&= 1 + \frac{2(1+\alpha)\sqrt{\frac{d}{n}} - (1+\alpha)^2 \frac{d}{n}}{1 - 2(1+\alpha)\sqrt{\frac{d}{n}} + (1+\alpha)^2 \frac{d}{n}}.
\end{align*}
Using the fact that for any matrix $\A \in \R^{d \times d}$, we may write $d \ \lambda_{\min}(\A) \le \Tr (\A) \le d \ \lambda_{\max}(\A)$, we obtain the final statement on the trace of $\wb\Sigma^{-1}$. The first of the two bounds can be further simplified by using $1/(1+x) \geq 1-x$ for any $x\geq 0$, thus obtaining
\begin{align*}
\lambda_{\min} \big( \wb{\Sigma}^{-1} \big) &\ge \Big( 1-(1+\alpha)\sqrt{\frac{d}{n}} \Big)^2.
\end{align*}
While under the assumption that $n \geq 4(1+\alpha)^2d$ we can use $1/(1-x) \leq 1+2x$ (for any $x\geq 1/2$) and obtain
\begin{align*}
\lambda_{\max} \big( \wb{\Sigma}^{-1} \big) &\le \Big( 1+2(1+\alpha)\sqrt{\frac{d}{n}} \Big)^2.
\end{align*}
\end{proof}

The statement of Proposition~\ref{prop:trace.concentration} (below) is obtained by recalling that $\Sigma\wh\Sigma_{i,n}^{-1}$ is the empirical covariance matrix of the whitened sample matrix $\wb\X_{i,n}$ and by a union bound over the number of samples $k_{i,n}$ and the number of instances $i$.

\begin{proposition*}
Force the number of samples $k_{i,t} \geq d+1$. If $\delta\in(0,1)$, for any $i \in [m]$ and step $t > m(d+1)$ with probability at least $1 - \delta/2$, we have
\begin{align*}
\bigg(1-C_{\Tr}\sqrt{\frac{d}{n}}\bigg)^2 \leq \frac{\Tr\Big( \Sigma \hat\Sigma^{-1}_{i,t} \Big)}{d} \le \bigg(1+2C_{\Tr}\sqrt{\frac{d}{n}}\bigg)^2,
\end{align*}
with $C_{\Tr} = 1+\sqrt{2\log(4nm/\delta)/d}$.
\end{proposition*}


\subsection{Concentration Inequality for $\wh \beta$ Estimates}
We slightly modify Theorem 2 from \cite{abbasi2011improved} to obtain a confidence ellipsoid over the $\wh \beta_i$'s.

\begin{theorem}\label{th:martingale_conf_ellip}
Let $\{ F_t \}_{t=0}^\infty$ be a filtration.
Let $\{ \eta_t \}_{t=1}^\infty$ be a real-valued stochastic process such that $\eta_t$ is $F_t$ measurable and $\eta_t$ is conditionally $R$-subgaussian for some $R \ge 0$, i.e.
\begin{equation}
\forall \lambda \in \R \qquad \E[e^{\lambda \eta_t} \mid F_{t-1}] \le \exp \left( \frac{\lambda^2 R^2}{2} \right).
\end{equation}
Let $\{ X_t \}_{t=1}^\infty$ be an $\R^d$-valued stochastic process such that $X_t$ is $F_{t-1}$ measurable.
Assume that $V$ is a $d \times d$ positive definite matrix. For any $t \ge 0$, define
\begin{equation}
\bar{V}_t = V + \sum_{s=1}^t X_s X_s^T, \qquad \qquad S_t = \sum_{s=1}^t \eta_s X_s.
\end{equation}
Let $V = \lambda \mathrm{Id}$, $\lambda > 0$, and define $Y_t = X_t^T \beta^* + \eta_t$.
Assume that $\| \beta^* \|_2 \le S$.
Also, let $\hat\beta_t = \bar{V}_t^{-1} \X_t^T \Y_t$ be the ridge estimate for $\beta$ after $t$ observations $\X_t, \Y_t$.
Then, for any $\delta > 0$, with probability at least $1 - \delta$, for all $t \ge 0$, $\beta^*$ lies in 
\begin{equation}\label{eq:martingale_th}
C_t = \left\{ \beta \in \R^d : \| \hat\beta_t - \beta \|_{\bar V_{t} / t} \le \frac{R}{\sqrt{t}} \sqrt{2 \log \left( \frac{\det\left( \bar V_{t} \right)^{1/2} \det \left( \lambda I \right)^{-1/2}}{\delta} \right)} 
+ \sqrt{\frac{\lambda}{t}} \ S \right\}.
\end{equation}
\end{theorem}
\begin{proof}
Take $x = \frac{\bar{V}_t}{t} (\hat \beta_t - \beta^*)$ in equation 5 in the proof of Theorem 2 in \cite{abbasi2011improved}.
\end{proof}

We use the previous theorem by lower bounding the $\bar{V}_t/t$ norm in $\Sigma$ norm.


\subsection{Bounded Norm Lemma}
\begin{lemma}\label{lm:bounded_norm_obs}
Let $X_1, \dots, X_t \in \R^d$ be iid subgaussian random variables.

If $\| X_1 \|^2$ is subexponential with parameters $(a^2, b)$, then, for $\alpha > 0$
\begin{equation}
\p \left( \frac{1}{t} \sum_{j=1}^t \| {X}_j \|^2 \le \E[ \| X_1 \|^2] + \frac{\alpha}{t} \right) \ge
\begin{cases}
1 - \exp\left( - \frac{\alpha^2}{2ta^2} \right) \qquad& \text{ if } 0 \le \alpha \le ta^2 / b, \\
1 - \exp\left( - \frac{\alpha}{2b} \right) \qquad& \text{ if } \alpha > ta^2 / b.
\end{cases}
\end{equation}
\end{lemma}
\begin{proof}
The proof directly follows by Proposition \ref{pro:bernstein}, by defining zero-mean subexponential random variable $Z$ with parameters $(a^2/t, b/t)$
\begin{equation}
Z = \frac{1}{t} \sum_{j=1}^t \| {X}_j \|^2 - \E \left[\frac{1}{t} \sum_{j=1}^t \| {X}_j \|^2\right].
\end{equation}
\end{proof}

\begin{corollary}\label{cor:bounded_norm_obs_gauss}
Let $X_1, \dots, X_t \in \R^d$ be iid gaussian variables, $X \sim \mathcal{N}(0, \mathrm{Id})$.
Assume $t \ge d + 1$.
Let $\delta > 0$.
Then, with probability at least $1 - \delta$,
\begin{equation}\label{eq_def_meanzero_subexp_bernstein}
\frac{1}{t} \sum_{j=1}^t \| {X}_j \|^2 \le d + 8 \log \left( \frac{1}{\delta} \right) \sqrt{\frac{d}{t}},
\end{equation}
\end{corollary}

\begin{proof}
For standard Gaussian $X \sim \mathcal{N}(0, \mathrm{Id})$, $\| X \|^2 \sim \chi_d^2$, and $a^2 = 4d$ and $b = 4$.
Note that $\E [\| X_j \|^2] = d$.
By the proof of Lemma \ref{lm:bounded_norm_obs} and \eqref{eq_def_meanzero_subexp_bernstein}
\begin{align}
\p \left( |Z| \ge a \sqrt{ \frac{2}{t} \log \left( \frac{1}{\delta} \right) } \right) \le \delta, \qquad \text{ when } t \ge 2 \left( \frac{b}{a} \right)^2 \ \log \left( \frac{1}{\delta} \right). \\
\p \left( |Z| \ge \frac{2b}{t} \log \left( \frac{1}{\delta} \right) \right) \le \delta, \qquad \text{ when } t < 2 \left( \frac{b}{a} \right)^2 \ \log \left( \frac{1}{\delta} \right).
\end{align}
Substituting $a = 2 \sqrt{d}$ and $b = 4$ leads to
\begin{align}
\p \left( |Z| \ge \sqrt{ \frac{8d}{t} \log \left( \frac{1}{\delta} \right) } \right) \le \delta, \qquad \text{ when } t \ge \frac{8}{d} \ \log \left( \frac{1}{\delta} \right). \\
\p \left( |Z| \ge \frac{8}{t} \log \left( \frac{1}{\delta} \right) \right) \le \delta, \qquad \text{ when } t < \frac{8}{d} \ \log \left( \frac{1}{\delta} \right). \label{eq_berns_fast}
\end{align}

We would like to upper bound $8 \log \left( {1}/{\delta} \right) / t$ in \eqref{eq_berns_fast}.
As $t > d$, we see
\begin{align}
\frac{8}{t} \log \left( \frac{1}{\delta} \right) \le \frac{8}{\sqrt{dt}} \log \left( \frac{1}{\delta} \right).
\end{align}
As a consequence,
\begin{align}
\p \left( |Z| \ge \frac{8}{\sqrt{dt}} \log \left( \frac{1}{\delta} \right) \right) \le \delta, \qquad \text{ when } t < \frac{8}{d} \ \log \left( \frac{1}{\delta} \right).
\end{align}
It follows that for all $t > d$
\begin{align}
\p \left( |Z| \ge \max \left( \frac{8}{\sqrt{dt}} \log \left( \frac{1}{\delta} \right), \sqrt{ \frac{8d}{t} \log \left( \frac{1}{\delta} \right) } \right) \right) \le \delta.
\end{align}
As $\delta < 1$, we finally conclude that
\begin{align}
\p \left( |Z| \ge 8 \sqrt{\frac{d}{t}} \log \left( \frac{1}{\delta} \right) \right) \le \delta.
\end{align}
Therefore, with probability at least $1 - \delta$,
\begin{equation}
\frac{1}{t} \sum_{j=1}^t \| {X}_j \|^2 \le d + 8 \log \left( \frac{1}{\delta} \right) \sqrt{\frac{d}{t}},
\end{equation}
as stated in the corollary.
\end{proof}

\newpage

\section{Performance Guarantees for \traceucb}
\label{app:traceucb}

\subsection{Lower Bound on Number of Samples (Proof of Theorem~\ref{th:lower_bound_numpulls})}
\label{app:traceucb1}

We derive the high-probability guarantee on the number of times each instance is selected.

\

\begin{theorem*}
Let $\delta > 0$.
With probability at least $1 - \delta$, the total number of contexts that \traceucb allocates to each problem instance $i$ after $n$ rounds satisfies
\begin{equation}
\label{eq_lower_bound_numpulls_app}
k_{i,n} \geq k_{i,n}^* - \frac{C_\Delta + 8C_{\Tr}}{\sigma_{\min}^2} \sqrt{\frac{nd}{\lambda_{\min}}} -\Omega(n^{1/4})
\end{equation}
where $R \ge \sigma_{\max}^2$ is known by the algorithm, and we defined $C_\Delta = 16 R \log(2mn/\delta)$, $C_{\Tr} = 1+\sqrt{2\log(4nm/\delta)/d}$, and $\lambda_{\min} = \sigma_{\min}^2 / \sum_j \sigma_j^2$.
\end{theorem*}

\begin{proof}
We denote by $\mathcal{E}_\delta$ the joint event on which Proposition~\ref{prop:sigma.concentration} and Proposition~\ref{prop:trace.concentration} hold at the same time with an overall probability $1-\delta$. This immediately gives upper and lower confidence bounds on the score $s_{i,t}$ used in \traceucb as
\begin{equation*}
\left( 1 - C_\Tr\sqrt{\frac{d}{k_{i,t}}} \right)^2 \frac{\sigma_i^2}{k_{i,t}} \le \frac{s_{i,t}}{d} \le  \left( 1 + 2C_\Tr\sqrt{\frac{d}{k_{i,t}}} \right)^2\frac{\sigma_i^2 + 2\Delta_{i,t}}{k_{i,t}}.
\end{equation*}
Recalling the definition of $\Delta_{i,t}$ we can rewrite the last term as
\begin{align*}
\frac{\sigma_i^2 + 2\Delta_{i,t}}{k_{i,t}} &= \left( 1 + \frac{16 R \log(2mn/\delta)}{\sigma_i^2 \sqrt{k_{i,t} - d}} \right) \frac{\sigma_i^2}{k_{i,t}} = \left( 1 + \frac{C_\Delta}{\sigma_i^2 \sqrt{k_{i,t} - d}} \right) \frac{\sigma_i^2}{k_{i,t}},
\end{align*}
where $C_\Delta = 16 R \log(2mn/\delta)$.
We consider a step $t+1 \leq n$ at which $I_{t+1} = q$. By algorithmic construction we have that $s_{p,t} \leq s_{q,t}$ for every arm $p \in [m]$. Using the inequalities above we obtain
\begin{align*}
\left( 1 - C_\Tr\sqrt{\frac{d}{k_{p,t}}} \right)^2 \frac{\sigma_p^2}{k_{p,t}} \leq \frac{s_{p,t}}{d} \leq \frac{s_{q,t}}{d} \leq \left( 1 + 2C_\Tr\sqrt{\frac{d}{k_{q,t}}} \right)^2\frac{\sigma_q^2 + 2\Delta_{q,t}}{k_{q,t}}
\end{align*}
If $t+1$ is the last time step at which arm $q$ is pulled, then $k_{q,t} = k_{q,t+1} -1 = k_{q,n}-1$ and $k_{p,n} \geq k_{p,t}$. Then we can rewrite the previous inequality as
\begin{align}\label{eq:comparison_armp_armq}
\left( 1 - C_\Tr\sqrt{\frac{d}{k_{p,n}}} \right)^2 \frac{\sigma_p^2}{k_{p,n}} =: A_{p,n} \leq B_{q,n} := \left( 1 + 2C_\Tr\sqrt{\frac{d}{k_{q,n}-1}} \right)^2\left( 1 + \frac{C_\Delta}{\sigma_q^2 \sqrt{k_{q,n} - d-1}} \right) \frac{\sigma_q^2}{k_{q,n}-1}.
\end{align}
If every arm is pulled exactly the optimal number of times, then for any $i\in[m]$, $k_{i,n} = k_{i,n}^*$ and the statement of the theorem trivially holds. Otherwise, there exists at least one arm that is pulled more than $k_{i,n}^*$. Let $q$ be this arm, then $k_{q,n} > k_{q,n}^*$. We recall that $L^*_n = d\sigma_q^2 / (k_{q,n}^*-d-1)$ and we rewrite the RHS of Eq.~\ref{eq:comparison_armp_armq} as
\begin{align*}
B_{q,n} &\leq \left( 1 + 2C_\Tr\sqrt{\frac{d}{k^*_{q,n}-d-1}} \right)^2\left( 1 + \frac{C_\Delta}{\sigma_q^2 \sqrt{k^*_{q,n} - d-1}} \right) \frac{\sigma_q^2}{k^*_{q,n}-d-1} \\
&\leq \left( 1 + 2C_\Tr\sqrt{\frac{L^*_n}{\sigma_q^2}} \right)^2\left( 1 + C_\Delta\sqrt{\frac{L^*_n}{d\sigma_q^6}} \right) \frac{L_n^*}{d}.
\end{align*}
We also simplify the LHS of Eq.~\ref{eq:comparison_armp_armq} as
\begin{align*}
A_{p,n} = \left( 1 - 2C_\Tr\sqrt{\frac{d}{k_{p,n}}} + C_\Tr^2\frac{d}{k_{p,n}} \right)  \frac{\sigma_p^2}{k_{p,n}} \geq \left( 1 - 2C_\Tr\sqrt{\frac{d}{k_{p,n}}}\right)  \frac{\sigma_p^2}{k_{p,n}}.
\end{align*}
At this point we can solve Eq.~\ref{eq:comparison_armp_armq} for $k_{p,n}$ and obtain a lower bound on it. We study the inequality $1/A_{p,n} \geq 1/B_{p,n}$.

We first notice that
\begin{align*}
\frac{1}{A_{p,n}} \leq \frac{k_{p,n}}{\sigma_p^2} \left( 1 + 4C_\Tr\sqrt{\frac{d}{k_{p,n}}}\right) \leq \frac{1}{\sigma_p^2} \left( \sqrt{k_{p,n}} + 2C_\Tr\sqrt{d}\right)^2,
\end{align*}
where we used $1/(1-x) \leq 1+2x$ for $x\leq 1/2$ and we added a suitable positive term to obtain the final quadratic form.
Similarly we have
\begin{align*}
\frac{1}{B_{q,n}} \geq \left( 1 - 2C_\Tr\sqrt{\frac{L^*_n}{\sigma_q^2}} \right)^2\left( 1 - C_\Delta\sqrt{\frac{L^*_n}{d\sigma_q^6}} \right) \frac{d}{L_n^*} = \left( 1 - 2C_\Tr\sqrt{\frac{L^*_n}{\sigma_q^2}} \right)^2\left( \frac{d}{L_n^*} - C_\Delta\sqrt{\frac{d}{L_n^*\sigma_q^6}} \right) ,
\end{align*}
where we used $1/(1+x) \geq 1-x$ for any $x\geq 0$. In order to ease the derivation of an explicit lower-bound on $k_{p,n}$, we further simplify the previous expression by replacing higher order terms with a big-$\Omega$ notation. We first recall that $L_n^* = \wt\Theta(md \wb\sigma^2 / n)$, then the terms of order $(1/L_n^*)$ and $(1/\sqrt{L_n^*})$ clearly dominate the expression, while all other terms are asymptotically constant or decreasing in $n$ and thus we can rewrite the previous bound as
\begin{align*}
\frac{1}{B_{q,n}} \geq \frac{d}{L_n^*} - (C_\Delta+4C_{\Tr}\sqrt{d})\sqrt{\frac{d}{L_n^*\sigma_q^6}} - \Omega(1).
\end{align*}
By setting $C = C_\Delta+4C_{\Tr}\sqrt{d}$ we can finally use the upper bound on $1/A_{p,n}$ and the lower bound on $1/B_{q,n}$ to obtain
\begin{align*}
\frac{1}{\sigma_p^2} \left( \sqrt{k_{p,n}} + 2C_\Tr\sqrt{d}\right)^2 \geq \frac{d}{L_n^*} - C\sqrt{\frac{d}{L_n^*\sigma_q^6}} - \Omega(1).
\end{align*}
We proceed with solving the previous inequality for $k_{p,n}$ and obtain
\begin{align*}
k_{p,n} \geq \sigma_p^2 \left(\left( \frac{d}{L_n^*} - C\sqrt{\frac{d}{L_n^*\sigma_q^6}} - \Omega(1) \right)^{1/2} - 2C_\Tr\sqrt{d}\right)^{2}.
\end{align*}
Taking the square on RHS and adding and subtracting $d+1$ we have
\begin{align*}
k_{p,n} \geq d+1+\sigma_p^2 \left( \frac{d}{L_n^*} - C\sqrt{\frac{d}{L_n^*\sigma_q^6}} - 4C_\Tr\sqrt{d}\left( \frac{d}{L_n^*} - C\sqrt{\frac{d}{L_n^*\sigma_q^6}} - \Omega(1)\right)^{1/2} + 4C_\Tr^2 d\right) -d-1 - \Omega(1).
\end{align*}
We clearly notice that the first three terms in the RHS are dominant (they are higher order function of $n$ through $L_n^*$) and thus we can isolate them and replace all other terms by their asymptotic lower bound as
\begin{align*}
k_{p,n} \geq d+1+ \frac{d\sigma_p^2 }{L_n^*} - \sqrt{\frac{1}{L_n^*}} \Big(C\sqrt{\frac{d\sigma_p^4}{\sigma_q^6}} + 4C_{\Tr}d\Big) -\Omega(n^{1/4}),
\end{align*}
where we used the fact that $L_n^* = \wt\Theta(md \wb\sigma^2 / n)$ to bound the higher order terms. Furthermore, we recall that $k_{p,n}^* = d\sigma_p^2/L_n^* +d +1$ and thus we can finally write the previous bound as
\begin{align*}
k_{p,n} \geq k_{p,n}^* - \sqrt{\frac{1}{L_n^*}} \Big(C\sqrt{\frac{d\sigma_p^4}{\sigma_q^6}} + 4C_{\Tr}d\Big) -\Omega(n^{1/4}).
\end{align*}
The final bound is obtained by using $\sigma_p^2 / \sum_{j} \sigma_j^2 = \lambda_p \geq \lambda_{\min}$ and $\sigma_q^2 \geq \sigma_{\min}^2$ with the final expression
\begin{align*}
k_{p,n} \geq k_{p,n}^* - \sqrt{n} \Big(\frac{C}{\sigma_{\min}^2} \sqrt{\frac{1}{\lambda_{\min}}} + 4C_{\Tr}\sqrt{d}\Big) -\Omega(n^{1/4}).
\end{align*}
A quite loose bound based on the definition of $C$ for the previous expression gives the final more readable result
\begin{align*}
k_{p,n} \geq k_{p,n}^* - \frac{C_\Delta + 8C_{\Tr}}{\sigma_{\min}^2} \sqrt{\frac{nd}{\lambda_{\min}}} -\Omega(n^{1/4}).
\end{align*}
\end{proof}


\subsection{Regret Bound (Proof of Theorem~\ref{thm_exp_regret})}
\label{app:traceucb2}

\begin{theorem*}
The regret of the Trace-UCB algorithm, i.e.,~the difference between its loss and the loss of optimal static allocation (see Eq.~\ref{eq_opt_stat_loss}), is upper-bounded by
\begin{equation}
\label{eq:regret-bound1_app}
L_n(\mathcal{A}) - L^*_n \leq O\bigg(\frac{1}{\sigma_{\min}^2}\Big(\frac{d}{\lambda_{\min}n}\Big)^{3/2}\bigg),
\end{equation}
where $\lambda_{\min} = {\sigma_{\min}^2}/{\sum_{j} \sigma_j^2}$.
\end{theorem*}

\begin{proof}
We first simplify the expression of the loss for \traceucb in Lemma~\ref{lm:algo_ridge_loss}. We invert trace operator and expectation and have
\begin{align*}
L_{i,n}(\wh\beta_i^\lambda) = \E \left( \Tr\left[ \Sigma\W_{i,n} \left( \sigma_i^2 \X_{i,n}^\transp \X_{i,n} + \lambda^2 \beta_i\beta_i^\transp \right) \W_{i,n}^\transp \right] \right).
\end{align*}
We notice that $\W_{i,n} = (\X_{i,n}^\transp\X_{i,n} + \lambda I)^{-1} \preceq (\X_{i,n}^\transp \X_{i,n})^{-1}$, where $\preceq$ is the Lower ordering between positive-definite matrices. We focus on the two additive terms in the trace separately. We have
\begin{align}\label{eq:loss.first.term}
\Tr \big( \Sigma\W_{i,n} \X_{i,n}^\transp \X_{i,n} \W_{i,n}^\transp\big) &= \Tr \big( \W_{i,n} \X_{i,n}^\transp \X_{i,n} \W_{i,n}^\transp \Sigma\big) \nonumber\\
&\leq \Tr \big( (\X_{i,n}^\transp \X_{i,n})^{-1}\X_{i,n}^\transp \X_{i,n} \W_{i,n}^\transp \Sigma\big) = \Tr \big( \Sigma\W_{i,n}^\transp\big)\\
&\leq \Tr \big( \Sigma (\X_{i,n}^\transp \X_{i,n})^{-1}\big) = \frac{1}{k_{i,n}} \Tr \big( \Sigma \wh\Sigma_{i,n}^{-1}\big),\nonumber
\end{align}
where we used the fact that $\Tr(\mathbf{A}\mathbf{B}) = \Tr(\mathbf{B}\mathbf{A})$, $\Tr(\A\mathbf{B}) \leq \Tr(\mathbf{C}\mathbf{B})$ if $\A \preceq \mathbf{C}$ and the definition of $\wh\Sigma_{i,n}$.

Similarly, we have
\begin{align*}
\Tr \big( \Sigma\W_{i,n}\beta_i\beta_i^\transp \W_{i,n}^\transp\big) &= \|\beta_i\|^2 \Tr \big( \Sigma\W_{i,n}\W_{i,n}^\transp\big) \\
&\leq \|\beta_i\|^2\Tr \big( (\X_{i,n}^\transp \X_{i,n})^{-1} \Sigma\W_{i,n}\big) \leq \|\beta_i\|^2 \frac{\lambda_{\max}(\wh\Sigma^{-1}_{i,n})}{k_{i,n}}\Tr \big( \Sigma\W_{i,n}\big)\\
&\leq \|\beta_i\|^2 \frac{\lambda_{\max}(\wh\Sigma^{-1}_{i,n})}{k_{i,n}}\Tr \big( \Sigma (\X_{i,n}^\transp \X_{i,n})^{-1}\big) = \|\beta_i\|^2 \frac{\lambda_{\max}(\wh\Sigma^{-1}_{i,n})}{k_{i,n}^2}\Tr \big( \Sigma \wh\Sigma_{i,n}^{-1}\big).
\end{align*}
Going back to the loss expression we have
\begin{align*}
L_{i,n}(\wh\beta_i^\lambda) \leq \E \left[ \frac{\Tr \big( \Sigma \wh\Sigma_{i,n}^{-1}\big)}{k_{i,n}} \bigg(\sigma_i^2 + \|\beta_i\|^2 \frac{\lambda_{\max}(\wh\Sigma^{-1}_{i,n})}{k_{i,n}}\bigg) \right].
\end{align*}
We decompose the loss in two terms depending on the high-probability event $\mathcal{E}_\delta$ under which the concentration inequalities Proposition~\ref{prop:sigma.concentration} and Proposition~\ref{prop:trace.concentration} hold at the same time
\begin{align*}
L_{i,n}(\wh\beta_i^\lambda) \leq \E \left[ \frac{\Tr \big( \Sigma \wh\Sigma_{i,n}^{-1}\big)}{k_{i,n}} \bigg(\sigma_i^2 + \|\beta_i\|^2 \frac{\lambda_{\max}(\wh\Sigma^{-1}_{i,n})}{k_{i,n}}\bigg) \Big| \mathcal{E}_\delta\right] + \delta\E \left( \Tr\left[ \Sigma\W_{i,n} \left( \sigma_i^2 \X_{i,n}^\transp \X_{i,n} + \lambda^2 \beta_i\beta_i^\transp \right) \W_{i,n}^\transp \right] \big| \mathcal{E}_\delta^{\mathsf{c}}\right),
\end{align*}
where we used $\mathbb{P}(\mathcal{E}_\delta^{\mathsf{c}} \leq \delta)$. If we denote the second expectation in the previous expression by $L_{i,n}^{\mathsf{c}}(\wh\beta_i^\lambda)$, then we can use Eq.~\ref{eq:loss.first.term} and obtain
\begin{align*}
L_{i,n}^{\mathsf{c}}(\wh\beta_i^\lambda) \leq \sigma_i^2\E \left( \Tr \big( \Sigma\W_{i,n}^\transp\big) \big| \mathcal{E}_\delta^{\mathsf{c}}\right) + \|\beta_i\| \lambda^2\E \left( \Tr \big( \Sigma\W_{i,n}\W_{i,n}^\transp\big) \big| \mathcal{E}_\delta^{\mathsf{c}}\right)
\end{align*}
Using the fact that $\Tr(AB) \leq \lambda_{\max}(A) \Tr(B)$, we can upper bound the previous equation as
\begin{align*}
L_{i,n}^{\mathsf{c}}(\wh\beta_i^\lambda) \leq \sigma_i^2\Tr(\Sigma)\E \left( \lambda_{\max}(\W_{i,n}) \big| \mathcal{E}_\delta^{\mathsf{c}}\right) + \|\beta_i\| \Tr(\Sigma)\lambda^2\E \left( \lambda_{\max}(\W_{i,n})^2 \big| \mathcal{E}_\delta^{\mathsf{c}}\right)
\end{align*}
Recalling that thanks to the regularization $\lambda_{\max}(\W_{i,n}) \leq 1/\lambda$, we finally obtain 
\begin{align}\label{eq:loss.bad.event}
L_{i,n}^{\mathsf{c}}(\wh\beta_i^\lambda) \leq \Tr(\Sigma) \Big(\frac{\sigma_i^2}{\lambda} + \|\beta_i\| \Big).
\end{align}
The analysis of the high-probability part of the bound relies on the concentration inequalities for the trace and $\lambda_{\max}$ and the lower bound on the number of samples $k_{i,n}$ from Thm.~\ref{th:lower_bound_numpulls}. We recall the three main inequalities we are going to use to bound the loss
\begin{align*}
k_{i,n} &\geq k_{i,n}^* - C\sqrt{nd} - \Omega(n^{1/4}),\\
\Tr(\Sigma \wh\Sigma_{i,n}^{-1}) &\leq d\bigg(1+2(1+\alpha)\sqrt{\frac{d}{n}}\bigg)^2, \\
\lambda_{\max}(\wh\Sigma_{i,n}^{-1}) &\leq \frac{1}{\lambda_{\min}(\Sigma)}\bigg(1+2(1+\alpha)\sqrt{\frac{d}{n}}\bigg)^2,
\end{align*}
where $C = \frac{C_\Delta + 8C_{\Tr}}{\sigma_{\min}^2\sqrt{\lambda_{\min}}}$ and the last inequality is obtained by multiplying by $\Sigma^{-1}\Sigma$ to whiten $\wh\Sigma_{i,n}$ and using Proposition~\ref{prop:trace_inv_covariance}, and $\lambda_{\max}(AB) \leq \lambda_{\max}(A)\lambda_{\max}(B)$ and finally $\lambda_{\max}(\Sigma^{-1}) = 1/\lambda_{\min}(\Sigma)$.
We can invert the first inequality as
\begin{align}\label{eq:inverse.pulls.bound}
\frac{1}{k_{i,n}} \leq \frac{1}{k_{i,n}^*- C\sqrt{nd}-\Omega(n^{1/4})} \leq \frac{1}{k_{i,n}^*} + O\bigg(\frac{2C}{k_{i,n}^*}\sqrt{ \frac{d}{n}}\bigg) \leq \frac{1}{k_{i,n}^*} + O\bigg(\frac{\sqrt{d}}{\sigma_{\min}^2(\lambda_{\min}n)^{3/2}}\bigg),
\end{align}
where the last inequality is obtained by recalling that $k_{i,n}^* = \Theta(\lambda_{i}n)$ and using the definition of $C$ (where we ignore $C_\Delta$ and $C_{\Tr}$).
We can then rewrite the high-probability loss as
\begin{align*}
\E \left[ \frac{\Tr \big( \Sigma \wh\Sigma_{i,n}^{-1}\big)}{k_{i,n}} \bigg(\sigma_i^2 + \|\beta_i\|^2 \frac{\lambda_{\max}(\wh\Sigma^{-1}_{i,n})}{k_{i,n}}\bigg) \Big| \mathcal{E}_\delta\right] \leq \frac{d\sigma_i^2}{k_{i,n}^*} + O\bigg(\frac{1}{\sigma_{\min}^2}\Big(\frac{d}{\lambda_{\min}n}\Big)^{3/2}\bigg) \leq L_n^* + O\bigg(\frac{1}{\sigma_{\min}^2}\Big(\frac{d}{\lambda_{\min}n}\Big)^{3/2}\bigg).
\end{align*}
By recalling the regret $R_n = \max_i L_{i,n}(\beta_{i,n}^\lambda) - L_n^*$, bringing the bounds above together and setting $\delta = O(n^{-3/2-\epsilon})$ for any $\epsilon>0$ and a suitable multiplicative constant, we obtain the final regret bound
\begin{align*}
R_n \leq O\bigg(\frac{1}{\sigma_{\min}^2}\Big(\frac{d}{\lambda_{\min}n}\Big)^{3/2}\bigg).
\end{align*}
\end{proof}

\subsection{High Probability Bound for Trace-UCB Loss (Proof of Theorem \ref{thm_high_prob})}

In this section, we start by defining a new loss function for algorithm $\mathcal{A}$:
\begin{equation}
\wt L_n(\mathcal{A}) = \max_{i \in [m]}  \| \beta_i - \hat\beta_{i,n} \|^2_{\Sigma}.
\end{equation}
Note that $\wt L_n(\mathcal{A})$ is a random variable as $\hat\beta_{i,n}$ is random, and the expectation is only taken with respect to the test point $X \sim \F$ (leading to the $\Sigma$-norm).
We expect results of the following flavor: let $\delta > 0$, then with probability at least $1 - \delta$,
\begin{equation}
\wt L_n(\mathcal{A}) - \wt L_n^* \le \tilde{O} \left( \left( \sum_j \sigma_j^2 \frac{d}{n} \right)^{3/2} \right),
\end{equation}
when $\mathcal{A}$ corresponds to \traceucb, and $\wt L_n^*$ to the optimal static allocation under ordinary least squares.

We start by focusing on $\wt L_n(\mathcal{A})$, and proving Theorem \ref{thm_high_prob}:

\

\begin{theorem*}
Let $\delta > 0$, and assume $\| \beta_i \|_2 \le Z$ for all $i$, for some $Z > 0$.
With probability at least $1 - \delta$,
\begin{equation}\label{eq_thm_highprob_app}
\wt L_n(\mathcal{A}) \le \frac{\sum_{j=1}^m \sigma_j^2}{n}\left( d+2\log\frac{3m}{\delta} \right) + O\left( \frac{1}{\sigma_{\min}^2}\left( \frac{d}{n\lambda_{\min}}  \right)^{3/2} \right),
\end{equation}
where $\lambda_{\min} = {\sigma_{\min}^2}/{\sum_{j} \sigma_j^2}$.
\end{theorem*}

\begin{proof}
We define a set of events that help us control the loss, and then we show that these events simultaneously hold with high probability.
In particular, we need the following events:

\begin{enumerate}
\item $\mathcal{E}_G \equiv$ the good event holds (for all arms $i$, and all times $t$), which includes a confidence interval for $\hat\sigma_{i,t}^2$ and the trace of the empirical covariance matrix.

Holds with probability $1 - \delta_G$.
This event is described and controlled in Proposition \ref{prop:sigma.concentration} and Proposition \ref{prop:trace.concentration}. 

\item $\mathcal{E}_{M,i} \equiv$ the confidence intervals $C_{i, t}$ created for arm $i$ at time $t$ contain the true $\beta_i$ at all times $t$ ---based on the vector-valued martingale in \cite{abbasi2011improved}.

Holds with probability $1 - \delta_{M, i}$.
This event is described and controlled in Theorem \ref{th:martingale_conf_ellip}. 

\item $\mathcal{E}_{C,i,t} \equiv$ the empirical covariance $\hat \Sigma_{i,t}$ for arm $i$ at time $t$ is close to $\Sigma$.
This event is a direct consequence of event $\mathcal{E}_G$.

\item $\mathcal{E}_{B,i,t} \equiv$ the first $t$ observations pulled at arm $i$ have norm reasonably bounded.
The empirical average norm is not too far from its mean.
Holds with probability $1 - \delta_{B, i, t}$.
This event is described and controlled in Corollary \ref{cor:bounded_norm_obs_gauss}.
\end{enumerate}
Let $H$ be the set of all the previous events.
Then, by the union bound
\begin{equation}
\p \left( \cap_{\epsilon \in H} \ \epsilon \right) \ge 1 - \sum_{\epsilon \in H} \delta_\epsilon.
\end{equation}
Our goal is to show that if $\cap_{\epsilon \in H} \ \epsilon$ holds, then the loss $\wt L_n(\mathcal{A}) = \max_{i \in [m]}  \| \beta_i - \hat\beta_{i,n} \|^2_{\Sigma}$ is upper bounded by a quantity that resembles the expected loss of the algorithm that knows the $\sigma_i^2$'s in advance.

Fix $\delta > 0$.
We want $\delta = \sum_{\epsilon \in H} \delta_\epsilon$, and we would like to assign equal weight to all the sets of events.
First, $\delta_G = \delta / 3$.
Also, $\sum_i \delta_{M, i} = \delta / 3$, implying $\delta_{M, i} = \delta / 3m$ for every arm $i \in [m]$.
Finally, to bound observation norms, we set $\sum_i \sum_t \delta_{B, i, t} = \delta / 3$.
It follows that we can take $\delta_{B, i, t} = \delta / 3mT$, even though $t$ really ranges from $d$ to $n$.

Assume that $\mathcal{E}_G, \mathcal{E}_{M, i}$ and $\mathcal{E}_{B, i, t}$ hold for all arms $i$ and times $t$.
Then, by Theorem \ref{th:lower_bound_numpulls}, the final number of pulls for arm $i$ can be lower bounded by
\begin{equation}
k_i \ge \frac{\sigma_i^2}{\sum_j \sigma_j^2} \ n -  c \left( \sqrt{\frac{\sigma_i^2}{\sigma_{\min}^2}} + 1 \right) \sqrt{\frac{\sigma_i^2}{\sum_j \sigma_j^2} \ dn} + o\left(\sqrt{dn}\right),
\end{equation}
where $c = 2\left(1 + \sqrt{2 \log ({12mn}/{\delta}) / d} \right)$.

For notational simplicity, we denote by $\hat \beta_{i,t}$ the estimate after $t$ pulls.
Thus, with respect to our previous notation where $\hat \beta_{i, n}$ referred to our final estimate, we have that $\hat \beta_{i, k_{i,n}} = \hat \beta_{i, n}$ as $k_{i,n}$ is the total number of pulls for arm $i$.

If the $\mathcal{E}_{M, i}$ events hold, then we know that our $\hat\beta_{i,t}$ estimates are not very far from the true values $\beta_i$ when $t$ is large.
In particular, we know that the error is controlled by the radius $R_{i, t}$ of the confidence ellipsoids.
We expect these radiuses to decrease with the number of observations per arm, $t$.
As we have a lower bound on the total number of pulls for arm $i$, $k_{i,n}$, if the confidence ellipsoids apply, then we can directly obtain an upper bound on the radius $R_{i, t}$ at the end of the process.

We need to do a bit of work to properly bound $\| \hat \beta_{i, k_{i,n}} - \beta_i \|^2_{\Sigma}$.

Fix arm $i$, and assume $\mathcal{E}_{M, i}$ holds.
In addition, assume $\| \beta_i \|_2 \le Z$ for all $i$.
Let $\bar V_{i, t} = \lambda I + \X_{i, t}^T \X_{i, t}$, where $\X_{i, t}$ contains the first $t$ observations pulled by arm $i$.
We modify the proof of Theorem 2 in \cite{abbasi2011improved} by taking $x = (\hat V_t / t) (\hat \beta_t - \beta_*)$ in their equation 5 (we are using their notation in the latter expression).
Assume the algorithm pulls arm $i$ a total of $t$ times ---$k_{i,n}$ is a stopping time with respect to the $\sigma$-algebra that includes the environment (other arms)--- then, by Theorem \ref{th:martingale_conf_ellip}
\begin{align}\label{eq:martingale_th}
\| \hat \beta_{i, t} - \beta_i \|_{\bar V_{i, t} / t} \le \frac{\sigma_i}{\sqrt{t}} \sqrt{2 \log \left( \frac{\det\left( \bar V_{i, t} \right)^{1/2} \det \left( \lambda I \right)^{-1/2}}{\delta_{M, i}} \right)} + \sqrt{\frac{{\lambda}}{{t}}} \ Z.
\end{align}
We would like to upper bound $\| \hat \beta_{i, k_{i,n}} - \beta_i \|_{\Sigma}$ by means of $\| \hat \beta_{i, k_{i,n}} - \beta_i \|_{\bar V_{i, k_{i,n}} / k_{i,n}}$.
Note that when $t$ grows, $\bar V_{i, t} / t \to \Sigma$ as the regularization is washed out.
The distance between $\hat \Sigma_{i, t} = \bar V_{i, t} / t - (\lambda / t) I$ and $\Sigma$ is captured by event $\epsilon_{C,i,t}$.

Formally, as $\mathcal{E}_G$ holds, we know that the difference between $\Sigma$ and $\hat \Sigma_{i, t}$ is bounded in operator norm for any $i$ and $t$ by
\begin{equation}
\| \Sigma - \hat \Sigma_{i, t} \| \le 2 \left( 1 + \sqrt{\frac{2}{d} \ \log \frac{2}{\delta_G}} \right) \sqrt{\frac{d}{t}} \ \| \Sigma \| = c \ \sqrt{\frac{d}{t}} \ \lambda_{\max}(\Sigma).
\end{equation}
Then, as a consequence, for all $x \in \R^s$
\begin{equation}\label{eq_operator_norm_bound}
x^T (\Sigma - \hat \Sigma_{i, t}) x \le c \ \lambda_{\max}(\Sigma) \sqrt{\frac{d}{t}} \ \| x \|_2^2.
\end{equation}
In particular, by taking $x = \hat \beta_{i, t} - \beta_i$,
\begin{align}
c \ \lambda_{\max}(\Sigma) \ \sqrt{\frac{d}{t}} \ \| \hat \beta_{i, t} - \beta_i \|_2^2 &\ge (\hat \beta_{i, t} - \beta_i)^T (\Sigma - \hat \Sigma_{i, t}) (\hat \beta_{i, t} - \beta_i) \\
&= \| \hat \beta_{i, t} - \beta_i \|^2_{\Sigma} - \| \hat \beta_{i, t} - \beta_i \|^2_{\hat \Sigma_{i, t}}.
\end{align}
In addition, note that $\| x \|_{\hat \Sigma_{i, t}}^2 = \| x \|^2_{\bar V_{i, t} / t} - (\lambda / t) \| x \|_2^2$.
We conclude that
\begin{align}
\| \hat \beta_{i, t} - \beta_i \|^2_{\Sigma} &\le \| \hat \beta_{i, t} - \beta_i \|^2_{\hat \Sigma_{i, t}} + c \ \lambda_{\max}(\Sigma) \ \sqrt{\frac{d}{t}} \ \| \hat \beta_{i, t} - \beta_i \|_2^2 \\
&= \| \hat \beta_{i, t} - \beta_i \|^2_{\bar V_{i, t} / t} + \left( c \ \lambda_{\max}(\Sigma) \ \sqrt{\frac{d}{t}} - \frac{\lambda}{t} \right) \| \hat \beta_{i, t} - \beta_i \|_2^2.
\end{align}
On the other hand, we know that $\| \hat \beta_{i, t} - \beta_i \|^2_{\Sigma} \ge \lambda_{\min}(\Sigma) \| \hat \beta_{i, t} - \beta_i \|_2^2$.

Therefore, by \eqref{eq:martingale_th}
\begin{align}
\| \hat \beta_{i, t} - \beta_i \|^2_{\Sigma} &\le \frac{1}{1 - \frac{1}{\lambda_{\min}(\Sigma)} \left( c \ \lambda_{\max}(\Sigma) \sqrt{\frac{d}{t}} - \frac{\lambda}{t} \right) } \| \hat \beta_{i, t} - \beta_i \|^2_{\bar V_{i, t} / t} \\
&\le \frac{1}{1 - \gamma_t} \ \left[
\frac{\sigma_i}{\sqrt{t}} \sqrt{2 \log \left( \frac{\det\left( \bar V_{i, t} \right)^{1/2} \det \left( \lambda I \right)^{-1/2}}{\delta_{M,i}} \right)} + \frac{\sqrt{\lambda} Z}{\sqrt{t}} \right]^2 \\
&\le \frac{1}{1 - \gamma_t} \ \frac{1}{t} \left[
\sigma_i \sqrt{2 \left( \frac{1}{2} \log \left( \frac{\det\left( \bar V_{i, t} \right)}{\det \left( \lambda I \right)}\right) + \log \left( \frac{1}{\delta_{M,i}} \right) \right)} + \sqrt{\lambda} Z \right]^2  \\
&\le \frac{1}{1 - \gamma_t} \ \frac{1}{t} \ \left[
\sigma_i \sqrt{2 \left( \frac{1}{2} \sum_{j=1}^t \| X_{j} \|^2_{\bar V_{i, t}^{-1}} + \log \left( \frac{1}{\delta_{M,i}} \right) \right)} + \sqrt{\lambda} Z \right]^2, \label{eq_martingale_control}
\end{align}
where we defined $\gamma_t = \left( c \ \lambda_{\max}(\Sigma) \sqrt{\frac{d}{t}} - \frac{\lambda}{t} \right) / \lambda_{\min}(\Sigma)$, and we used Lemma 11 in \cite{abbasi2011improved} which shows that
\begin{equation}\label{eq_lemma11_csaba}
\log \left( \frac{\det\left( \bar V_{i, t} \right)}{\det \left( \lambda I \right)}\right) \le \sum_{j=1}^t \| X_{j} \|^2_{\bar V_{i, t}^{-1}}.
\end{equation}
We would like to approximate the $\bar V_{i, t}^{-1}$ norm, by means of the inverse covariance norm, $\Sigma^{-1}$.
The whitened equation that is equivalent to \eqref{eq_operator_norm_bound} --- see Lemma \ref{prop:trace_inv_covariance} --- is given by $\| I - \hat{\bar{\Sigma}}_{i, t} \| \le \epsilon$, with $\epsilon = c \sqrt{{d}/{t}}$.

It implies that for any $j = 1, \dots, d$,
\begin{equation}
1 - c \sqrt{\frac{d}{t}} - O\left( \frac{d}{t} \right) \le \lambda_j (\hat{\bar{\Sigma}}_{i, t}) \le 1 + c \sqrt{\frac{d}{t}} + O\left( \frac{d}{t} \right).
\end{equation}
The $\bar V_{i, t}^{-1}$ norm can be bounded as follows
\begin{align}
\| x \|_{\bar V_{i, t}^{-1}}^2 &= x^T \bar V_{i, t}^{-1} x = x^T \left( \lambda I + \X_{i,t}^T\X_{i,t} \right)^{-1} x \\
&= x^T \Sigma^{-1/2} \Sigma^{1/2}  \left( \lambda I + \X_{i,t}^T\X_{i,t} \right)^{-1} \Sigma^{1/2} \Sigma^{-1/2} x \\
&= \bar{x}^T \left( \lambda \Sigma^{-1} + \bar{\X}_{i,t}^T\bar{\X}_{i,t} \right)^{-1} \bar{x} \\
&= \frac{1}{t} \bar{x}^T \left( \frac{\lambda}{t} \Sigma^{-1} + \hat{\bar{\Sigma}}_{i,t}^{-1} \right)^{-1} \bar{x},
\end{align}
where $\bar{x}$ denotes the whitened version of $x$.
We can now apply the matrix inversion lemma to see that
\begin{align}
\| x \|_{\bar V_{i, t}^{-1}}^2 &= \frac{1}{t} \bar{x}^T \left( \frac{\lambda}{t} \Sigma^{-1} + \hat{\bar{\Sigma}}_{i,t}^{-1} \right)^{-1} \bar{x} \\
&= \frac{1}{t} \bar{x}^T \left( \hat{\bar{\Sigma}}_{i,t} - \hat{\bar{\Sigma}}_{i,t} \Sigma^{-1/2} \left( \frac{t}{\lambda} I + \Sigma^{-1/2} \hat{\bar{\Sigma}}_{i,t} \Sigma^{-1/2} \right)^{-1} \Sigma^{-1/2} \hat{\bar{\Sigma}}_{i,t} \right) \bar{x} \\
&= \frac{1}{t} \bar{x}^T \left( \hat{\bar{\Sigma}}_{i,t} - \hat{\bar{\Sigma}}_{i,t} \Sigma^{-1/2} R^{-1} \Sigma^{-1/2} \hat{\bar{\Sigma}}_{i,t} \right) \bar{x},
\end{align}
where we implicitly defined $R = ({t}/{\lambda}) I + \Sigma^{-1/2} \hat{\bar{\Sigma}}_{i,t} \Sigma^{-1/2}$, a positive definite matrix.
We upper bound the previous expression to conclude that
\begin{align}
\| x \|_{\bar V_{i, t}^{-1}}^2 &= \frac{1}{t} \bar{x}^T \left( \hat{\bar{\Sigma}}_{i,t} - \hat{\bar{\Sigma}}_{i,t} \Sigma^{-1/2} R^{-1} \Sigma^{-1/2} \hat{\bar{\Sigma}}_{i,t} \right) \bar{x} \\
&\le \frac{1}{t} \bar{x}^T \hat{\bar{\Sigma}}_{i,t} \bar{x} \le \frac{\lambda_{\max}(\hat{\bar{\Sigma}}_{i,t})}{t} \| \bar{x} \|_2^2 \le \frac{1 + c \sqrt{{d}/{t}} + O\left( {d}/{t} \right)}{t} \| \bar{x} \|_2^2.
\end{align}

\

If we now go back to \eqref{eq_lemma11_csaba}, using the previous results, we see that
\begin{align}
\sum_{j=1}^t \| X_{j} \|^2_{\bar V_{i, t}^{-1}} &\le \left( 1 + c \sqrt{\frac{d}{t}} + O\left( \frac{d}{t} \right) \right) \left( \frac{1}{t} \sum_{j=1}^t {\| \bar{X}_j \|^2_2} \right).
\end{align}
Substituting the upper bound in \eqref{eq_martingale_control}:
\begin{align}
&\| \hat \beta_{i, t} - \beta_i \|^2_{\Sigma} \le \frac{1}{1 - \gamma_t} \ \frac{1}{t} \ \left[
\sigma_i \sqrt{2 \left( \frac{1}{2} \sum_{j=1}^t \| X_{j} \|^2_{\bar V_{i, t}^{-1}} + \log \left( \frac{1}{\delta_{M,i}} \right) \right)} + \sqrt{\lambda} Z \right]^2 \\
&\le \frac{1}{1 - \gamma_t} \ \frac{1}{t} \ \left[
\sigma_i \sqrt{ \left( 1 + c \sqrt{\frac{d}{t}} + O\left( \frac{d}{t} \right) \right)\left( \frac{1}{t} \sum_{j=1}^t {\| \bar{X}_j \|^2_2} \right) + 2 \log \frac{1}{\delta_{M,i}} } + \sqrt{\lambda} Z \right]^2. \nonumber
\end{align}
By Corollary \ref{cor:bounded_norm_obs_gauss}, with probability $1 - \delta_{B, i, t}$, the empirical average norm of the white gaussian observations is controlled by
\begin{equation}
\frac{1}{t} \sum_{j=1}^t \| \bar{X}_j \|^2 \le d + 8 \log \left( \frac{1}{\delta_{B, i, t}} \right) \sqrt{\frac{d}{t}}.
\end{equation}
As $\delta_{B, i, t} = \delta / 3mn$ and $\delta_{M, i} = \delta / 3m$, we conclude that
\begin{align}
&\| \hat \beta_{i, t} - \beta_i \|^2_{\Sigma} \le \frac{1}{1 - \gamma_t} \ \frac{1}{t} \ \left[
\sigma_i \sqrt{ \left( 1 + c \sqrt{\frac{d}{t}} + O\left( \frac{d}{t} \right) \right)\left( d + 8 \log \left( \frac{3mn}{\delta} \right) \sqrt{\frac{d}{t}} \right) + 2 \log \left( \frac{3m}{\delta} \right)} + \sqrt{\lambda} Z \right]^2 \nonumber \\
&\le \frac{1}{1 - \left( c \lambda_{\max}(\Sigma) \sqrt{\frac{d}{t}} - \frac{\lambda}{t} \right) / \lambda_{\min}(\Sigma)} \ \frac{1}{t} \ \left[
\sigma_i \sqrt{ \left( d + \left( c + 8 \log \frac{3mn}{\delta} \right) \sqrt{\frac{d}{t}} + O\left( \frac{d}{t} \right) \right)  + 2 \log \frac{3m}{\delta}} + \sqrt{\lambda} Z \right]^2. \label{eq_main_bound_expression}
\end{align}
At this point, recall that under our events
\begin{equation}\label{eq_lower_bound_num_pulls_bis}
k_{i,n} \geq k_{i,n}^* - C\sqrt{nd} - \Omega(n^{1/4}),
\end{equation}
where $C = \frac{C_\Delta + 8C_{\Tr}}{\sigma_{\min}^2\sqrt{\lambda_{\min}}}$.
As \eqref{eq_main_bound_expression} decreases in $t$, we will bound the error $\| \hat \beta_{i, t} - \beta_i \|^2_{\Sigma}$ by taking the number of pulls $t = ({\sigma_i^2}/{\sum_j \sigma_j^2}) n + O(\sqrt{dn})$ (in particular, the RHS of \eqref{eq_lower_bound_num_pulls_bis}).

If we take $\lambda = 1/n$, we have that
\begin{align}
&\| \hat \beta_{i, t} - \beta_i \|^2_{\Sigma} \\
&\le \frac{1}{1 - \left( c \lambda_{\max}(\Sigma) \sqrt{\frac{d}{t}} - \frac{\lambda}{t} \right) / \lambda_{\min}(\Sigma)} \ \frac{1}{t} \ \left[
\sigma_i \sqrt{ \left( d + \left( c + 8 \log \frac{3mn}{\delta} \right) \sqrt{\frac{d}{t}} + O\left( \frac{d}{t} \right) \right)  + 2 \log \frac{3m}{\delta}} + \sqrt{\lambda} Z \right]^2 \nonumber \\
&\le \left( 1 + c \frac{\lambda_{\max}(\Sigma)}{\lambda_{\min}(\Sigma)} \sqrt{\frac{d}{t}} + O\left(\frac{d}{t}\right)  \right) \ \frac{1}{t} \ \left[
\sigma_i \sqrt{ \left( d + \left( c + 8 \log \frac{3mn}{\delta} \right) \sqrt{\frac{d}{t}} + O\left( \frac{d}{t} \right) \right)  + 2 \log \frac{3m}{\delta}} + \sqrt{\lambda} Z \right]^2 \nonumber \\
&\le \left( 1 + O \left( \sqrt{\frac{d}{t}} \right) \right) \ \frac{1}{t} \ \left[
\sigma_i^2 \left( d + 2 \log \frac{3m}{\delta} + \left( c + 8 \log \frac{3mn}{\delta} \right) \sqrt{\frac{d}{t}}  \right)  + \frac{Z^2}{n} + 2Z\sigma_i \sqrt{\frac{d + 2 \log \frac{3m}{\delta}}{n}} + o \left( \sqrt{\frac{d}{n}} \right) \right]. \nonumber
\end{align}
Now, by \eqref{eq_lower_bound_num_pulls_bis} and \eqref{eq:inverse.pulls.bound}, and using the $\lambda_i = \sigma_i^2 / \sum_j \sigma_j^2$ notation
 \begin{align}
&\| \hat \beta_{i, t} - \beta_i \|^2_{\Sigma} \\
&\le \left( 1 + O \left( \sqrt{\frac{d}{n}} \right) \right) \frac{\left[
\sigma_i^2 \left( d + 2 \log \frac{3m}{\delta} \right) + \sigma_i^2 \left( c + 8 \log \frac{3mn}{\delta} \right) \sqrt{\frac{d}{t}} + 2Z\sigma_i \sqrt{\frac{d}{n}} + o \left( \sqrt{\frac{d}{n}} \right) \right]}{k_{i,n}^* - C\sqrt{nd} - \Omega(n^{1/4})} \nonumber \\
&= \left( 1 + O \left( \sqrt{\frac{d}{n}} \right) \right) \frac{\left[
\sigma_i^2 \left( d + 2 \log \frac{3m}{\delta} \right) + \left( \sigma_i^2 \left( c + 8 \log \frac{3mn}{\delta} \right) + 2Z\sigma_i \right) \sqrt{\frac{d}{t}}  + o \left( \sqrt{\frac{d}{n}} \right) \right]}{k_{i,n}^* - C\sqrt{nd} - \Omega(n^{1/4})} \nonumber \\
&= \left( 1 + O \left( \sqrt{\frac{d}{n}} \right) \right) \left( \frac{1}{k_{i,n}^*} + O\bigg(\frac{\sqrt{d}}{\sigma_{\min}^2(\lambda_{\min}n)^{3/2}}\bigg) \right) \left[ \sigma_i^2 \left( d + 2 \log \frac{3m}{\delta} \right) + \tilde{O} \left(\sqrt{\frac{d}{n}} \right) \right] \nonumber \\
&= \frac{\sigma_i^2}{k_{i,n}^*} \left( d + 2 \log \frac{3m}{\delta} \right) + O\bigg(\frac{1}{\sigma_{\min}^2}\Big(\frac{d}{\lambda_{\min}n}\Big)^{3/2}\bigg).
\end{align} 
\end{proof}

\newpage

\section{Loss of a RLS-based Learning Algorithm}
\label{app:rls.loss.learning}

\subsection{Distribution of RLS estimates}

\begin{proposition}\label{prop:distribution.rls}
Given a linear regression problem with observations $Y = X^\transp \beta + \epsilon$ with Gaussian noise with variance $\sigma^2$, after $n$ contexts $\X$ and the corresponding observations $\Y$, the ridge estimate of parameter $\lambda$ is obtained as 
\begin{equation*}
\hat\beta^\lambda = (\X^\transp\X + \lambda I)^{-1} \X^\transp\Y = \W \X^\transp\Y,
\end{equation*}
with $\W = (\X^\transp\X + \lambda I)^{-1}$, and its distribution conditioned on $\X$ is
\begin{equation}
\hat\beta^\lambda \mid \X \sim \mathcal{N} \left( \beta - \lambda \W \beta, \sigma^2 \ \W (\X^\transp \X) \W^\transp \right).
\end{equation}
\end{proposition}

\begin{proof}
Recalling the definition of the OLS estimator $\wh\beta$ (assuming it exists), we can easily rewrite the RLS estimator as
\begin{equation*}
\hat\beta^\lambda = (\X^\transp\X + \lambda I)^{-1} (\X^\transp\X) (\X^\transp\X)^{-1} \X^\transp\Y = (\X^\transp\X + \lambda I)^{-1} (\X^\transp\X) \wh\beta.
\end{equation*}
This immediately gives that the conditional distribution of $\hat\beta^\lambda$ is Gaussian as for $\wh\beta$. We just need to compute the corresponding mean vector and the covariance matrix. We first notice that the RLS estimator is biased as
\begin{equation*}
\E[\hat\beta^\lambda \big| \X] = (\X^\transp\X + \lambda I)^{-1} (\X^\transp\X) \beta.
\end{equation*}
Let $\St = \X^\transp \X$, then we can further rewrite the bias as
\begin{align*}
\E[\hat\beta^\lambda \big| \X] &= (\St + \lambda \St \St^{-1})^{-1} \St \beta = \Big( \St \big(I + \lambda \St^{-1}\big) \Big)^{-1} \St \beta \\
&= (I + \lambda \St^{-1})^{-1} \beta = \big(I - \lambda (\St+\lambda I)^{-1} \big) \beta \\
&= \beta - \lambda (\St+\lambda I)^{-1} \beta = \beta - \lambda \W \beta,
\end{align*}
where we used the matrix inversion lemma. Recalling that the covariance of $\wh\beta$ is $\sigma^2(\X^\transp \X)^{-1}$, the covariance of $\wh\beta^\lambda$ is then
\begin{equation*}
\text{Cov}\left[ \hat\beta^\lambda | \X \right] = \W(\X^\transp\X)\text{Cov}\left[ \hat\beta | \X \right] (\X^\transp\X) \W^\transp = \sigma^2 \W(\X^\transp\X) \W^\transp.
\end{equation*}
\end{proof}


\subsection{Loss Function of a RLS-based Algorithm}

We start by proving the loss function in the case of a static algorithm.

\begin{lemma}\label{lm:loss_ridge_static}
Let $\mathcal{A}$ be a learning algorithm that selects instance $i$ for $k_{i,n}$ times, where $k_{i,n}$ is a fixed quantity chosen in advance, and that returns estimates $\wh\beta_i^\lambda$ obtained by RLS with regularization $\lambda$. Then its loss after $n$ steps can be expressed as
\begin{equation}\label{eq:ridge_static_loss}
L_n(\mathcal{A}_{\text{stat}}) = \max_{i \in [m]} \ \Tr \left( \Sigma \E\left[ \W_{i,n} \left( \sigma_i^2 \X_{i,n}^\transp \X_{i,n} + \lambda^2 \beta_i\beta_i^\transp \right) \W_{i,n}^\transp \right] \right),
\end{equation}
where $\W_{i,n} = (\X_{i,n}^\transp\X_{i,n} + \lambda I)^{-1}$, and $\X_{n_i}$ is the matrix with the $k_{i,n}$ contexts from instance $i$.
\end{lemma}

\begin{proof}
The proof follows the same steps as in App.~\ref{app:optimal.static} up to Eq.~\ref{eq:loss.stat.allocation.step}, where we have
\begin{align*}
L_n(\mathcal{A}_{\text{stat}}) &= \max_{i \in [m]} \ \Tr \bigg( \E_{\X_i} \bigg[ \E_{\boldsymbol{\epsilon}_{i}} \Big[ \Sigma(\beta_i -\hat\beta_i)(\beta_i -\hat\beta_i)^\transp \big| \X_i\Big] \bigg]\bigg).
\end{align*}
Following Proposition~\ref{prop:distribution.rls}, we can refine the inner expectation as
\begin{align*}
&\E\left[(\hat\beta - \beta) (\hat\beta - \beta)^\transp \mid \X \right] \\
&= \E\left[(\hat\beta - \beta + \lambda \W\beta - \lambda \W\beta) (\hat\beta - \beta + \lambda \W\beta - \lambda \W\beta)^\transp \mid \X \right] \\
&= \E\left[(\hat\beta - \E[\hat\beta \mid \X] - \lambda \W\beta) (\hat\beta - \E[\hat\beta \mid \X] - \lambda \W\beta)^\transp \mid \X \right] \\
&= \E\left[(\hat\beta - \E[\hat\beta \mid \X]) (\hat\beta - \E[\hat\beta \mid \X])^\transp \mid \X \right] + \lambda^2 \W \beta \beta^\transp \W^\transp \\
&= \sigma^2 \ \W (\X^\transp \X) \W^\transp + \lambda^2 \W \beta \beta^\transp \W^\transp \\
&= \W \left[ \sigma^2 \X^\transp \X + \lambda^2 \beta\beta^\transp \right] \W^\transp.
\end{align*}
Plugging the final expression back into $L_n(\mathcal{A}_{\text{static}})$ we obtain the desired expression.
\end{proof}

We notice that a result similar to Lemma~\ref{lm:indep_hatbeta_hatsigma} holds for RLS estimates as well.

\begin{proposition}\label{prop:indep_betaridge_sigmaols}
Assume the noise $\epsilon$ is Gaussian.
Let $\hat\sigma^2$ be the estimate for $\sigma^2$ computed by using the residuals of the OLS solution $\hat\beta$.
Then, $\hat\beta^\lambda$ and $\hat\sigma^2$ are independent random variable conditionally to $\X$.
\end{proposition}

\begin{proof}
As shown in the proof of Proposition~\ref{prop:distribution.rls}, we have $\hat\beta^\lambda = (\X^\transp\X + \lambda I)^{-1} (\X^\transp\X) \wh\beta$ and we know that functions of independent random variables are themselves independent. Since the matrix mapping $\wh\beta$ to $\wh\beta^\lambda$ is fixed given $\X$, and $\hat\beta$ and $\hat\sigma^2$ are conditionally independent from Lemma~\ref{lm:indep_hatbeta_hatsigma}, then the statement follows.
\end{proof}

We can now combine Proposition~\ref{prop:indep_betaridge_sigmaols} and Lemma~\ref{lm:loss_ridge_static} to conclude that a similar expression to Eq.~\ref{eq:ridge_static_loss} holds for the ridge estimators also when a non-static algorithm such as \traceucb is run.

\begin{lemma}\label{lm:algo_ridge_loss}
Let $\mathcal{A}$ be a learning algorithm such that $I_t$ is chosen as a function of $\mathcal{D}_{t-1} = \{X_1, I_1, Y_{I_1, 1}, \ldots, X_{t-1}, I_{t-1}, Y_{I_{t-1},t-1}\}$, and that it returns estimates $\wh\beta_i^\lambda$ obtained by RLS with regularization $\lambda$. Then its loss after $n$ steps can be expressed as
\begin{equation}\label{eq:ridge_static_loss}
L_n(\mathcal{A}) = \max_{i \in [m]} \ \Tr \left( \Sigma \E\left[ \W_{i,n} \left( \sigma_i^2 \X_{i,n}^\transp \X_{i,n} + \lambda^2 \beta_i\beta_i^\transp \right) \W_{i,n}^\transp \right] \right),
\end{equation}
where $\W_{i,n} = (\X_{i,n}^\transp\X_{i,n} + \lambda I)^{-1}$, and $\X_{i, n}$ is the matrix with the $k_{i,n}$ contexts from instance $i$.
\end{lemma}
\begin{proof}
The proof follows immediately by extending Lemma \ref{lm:indep_hatbeta_hatsigma} to $\hat\beta_{\lambda}$ as, by Proposition~\ref{prop:indep_betaridge_sigmaols}, $\hat{\beta}_{\lambda}$ and $\hat\sigma^2_{\text{OLS}}$ are independent.
Then, we proceed in a way similar to that in the proof of Lemma \ref{lem:loss.learning} to perform the required conditioning.
\end{proof}

\newpage

\section{Sparse Trace-UCB Algorithm}
\label{app:high.dimensions}

\subsection{Summary}
High-dimensional linear regression models are remarkably common in practice.
Companies tend to record a large number of features of their customers, and feed them to their prediction models. There are also cases in which the number of problem instances under consideration $m$ is large, e.g.,~too many courses in the MOOC example described in the introduction. Unless the horizon $n$ is still proportionally large w.r.t.\ $md$, these scenarios require special attention. In particular, algorithms like \traceucb that adaptively use contexts in their allocation strategy become more robust than their context-free counterparts. 

A natural assumption in such scenarios is {\em sparsity}, i.e.,~only a small subset of features are relevant to the prediction problem at hand (have non-zero coefficient). In our setting of $m$ problem instances, it is often reasonable to assume that these instances are related to each other, and thus, it makes sense to extend the concept of sparsity to {\em joint sparsity}, i.e.,~a sparsity pattern across the instances. Formally, we assume that there exists a $s \ll d$ such that
\begin{equation}
|S| \stackrel{\Delta}{=} | \cup_{i \in [m]} \text{supp}(\beta_i) | = s,
\end{equation}
where $\text{supp}(\beta_i) = \{ j \in [d] : \beta_i^{(j)} \neq 0 \}$ denotes the support of the $i$'th problem instance. A special case of joint sparsity is when $| \text{supp}(\beta_i) | \approx s$, for all $i$, i.e.,~most of the relevant features are shared across the instances.

In this section, we focus on the scenario where $dm > n$. When we can only allocate a small (relative to $d$) number of contexts to each problem instance, proper balancing of contexts becomes extremely important, and thus, the algorithms that do not take into account context in their allocation are destined to fail. Although \traceucb has the advantage of using context in its allocation strategy, it still needs to quickly discover the relevant features (those in the support) and only use those in its allocation strategy.

This motivates a two-stage algorithm, we call it \sparsetraceucb, whose pseudocode is in Algorithm~\ref{alg:sparse_trace_ucb}. In the first stage, the algorithm allocates contexts uniformly to all the instances, $L$ contexts per instance, and then recovers the support. In the second stage, it relies on the discovered support $\wh S$, and applies the standard \traceucb to all the instances, but only takes into account the features in $\wh S$. Note that $L$ should be large enough that with high probability, support is exactly discovered, i.e., $\wh S = S$.

There exists a large literature on how to perform \emph{simultaneous} support discovery in jointly sparse linear regression problems~\cite{negahban2011simultaneous, obozinski2011support, wang2013block}, which we discuss in detail below.

Most of these algorithms minimize the regularized empirical loss
\begin{equation*}\label{opt_problem}
\min_{\mathrm{M} \in \R^{d \times m}} \frac{1}{k} \sum_{i=1}^m \| \Y_i - \X_i \ \mathrm{M}[,i] \|^2 + \lambda \ \| \mathrm{M} \|,
\end{equation*}
where $k$ is the number of samples per problem, $\mathrm{M}$ be the matrix whose $i$'th column is $\mathrm{M}[,i] = \hat{\beta}_i$, $\X_i \in \R^{k \times d}$, and $\Y_i = \X_i \beta_i + \epsilon_i$. In particular, they use $l_a/l_b$ block regularization norm, i.e.,~$\| \mathrm{M} \|_{l_a/l_b} = \| v \|_{l_a}$, where $v_i = \| \mathrm{M}[i,] \|_{l_b}$ and $\mathrm{M}[i,]$ is the $i$'th row of $\mathrm{M}$. 
In short, the \sparsetraceucb algorithm uses the $l_1/l_2$ block regularization Lasso algorithm~\cite{wang2013block}, an extension of the algorithm in~\cite{obozinski2011support}, for its support discovery stage.

We extend the guarantees of Theorem \ref{thm_high_prob} to the high dimensional case with joint sparsity, assuming $s$ is known.

The following is the main result of this section:
\begin{theorem}\label{thm_high_prob_sparse}
Let $\delta_1 > 0$.
Assume $\| \beta_i \|_2 \le Z$ for all $i$, for some $Z > 0$, and assume the parameters $(n, d, s, \beta_i, \Sigma)$ satisfy conditions \textbf{C1} to \textbf{C5} in \cite{wang2013block}.
Let $\psi$ be the sparsity overlap function defined in \cite{obozinski2011support}.
If $L > 2(1+v) \ \psi \ \log(d-s) {\rho_u(\Sigma^{(1:m)}_{S_C S_C | S})}/{\gamma^2}$ for some constant $v > 0$, and $n-Lm \ge (s+1)m$, then, with probability at least $1 - \delta_1 - \delta_2$,
\begin{equation}
\wt L_n(\mathcal{A}) \le \frac{\sum_j \sigma_j^2}{n-Lm} \left( s + 2 \log \frac{3m}{\delta_1} \right) +  \frac{2c}{\sqrt{\sigma_{\min}^2}}   \left( \frac{s \sum_j \sigma_j^2}{n - Lm} \right)^{3/2} + o \left( z \right),
\end{equation}
where $c \le 2\left(1 + \sqrt{2 \log ({12mn}/{\delta_1}) / s} \right)$
and we defined $\delta_2 = m \exp(-c_0 \log s) + \exp(-c_1 \log(d-s))$ for positive constants $c_0, c_1 > 0$, and $z = \left( {s}/{(n-Lm)} \right)^{3/2}$.
\end{theorem}
The exact technical assumptions and the proof are given and discussed in below.
We simply combine the high-probability results of Theorem \ref{thm_high_prob}, and the high-probability support recovery of Theorem 2 in \cite{wang2013block}.

In addition, we provide Corollary~\ref{cor:high_dim_traceucb}, where we study the regime of interest where the support overlap is complete (for simplicity), $n = C_1 \ m s \log d \ll md$ for $C_1 > 0$, and $L = C_2 \ s \log d$, for $C_1 - C_2 > 0$.
\begin{corollary}\label{cor:high_dim_traceucb}
Under the assumptions of Theorem \ref{thm_high_prob_sparse}, let $\delta_1 > 0$, assume $n = C_1 \ m s \log d$, the support of all arms are equal, and set $L = C_2 \ s \log d$, for $\bar{C} := C_1 - C_2 > 0$.
Then, with probability at least $1 - \delta_1 - \delta_2$,
\begin{equation}
\wt L_n(\mathcal{A}) \le \frac{\sum_j \sigma_j^2}{\bar{C} ms \log d} \left( s + 2 \log \frac{3m}{\delta_1} \right) +  \frac{2c}{\sqrt{\sigma_{\min}^2}}   \left( \frac{\sum_j \sigma_j^2}{\bar{C} m \log d} \right)^{3/2} + o \left( z \right)
\end{equation}
where $c \le 2\left(1 + \sqrt{2 \log ({12mn}/{\delta_1}) / s} \right)$ and we defined $\delta_2 = m \exp(-c_0 \log s) + \exp(-c_1 \log(d-s))$ for constants $c_0, c_1 > 0$, and $z = \left( \bar{C} m \log d \right)^{-3/2}$.
\end{corollary}

Algorithm~\ref{alg:sparse_trace_ucb} contains the pseudocode of our Sparse-\traceucb algorithm.
 
\begin{algorithm}[ht]
\begin{algorithmic}[1]
\FOR{$\;i=1,\ldots,m\;$}
\STATE Select problem instance $i$ exactly $L$ times
\ENDFOR
\STATE Run $l_1/l_2$ Lasso to recover support $\bar{S} = \cup_i \ \text{supp}(\hat\beta_{i,L})$
\FOR{$\;i=1,\ldots,m\;$}
\STATE Select problem instance $i$ exactly $s+1$ times
\STATE Compute its OLS estimates $\hat\beta_{i,m(L+s+1)}$ and $\hat\sigma^2_{i,m(L+s+1)}$ with respect to dimensions in $\bar{S}$.
\ENDFOR
\FOR{steps $\;t=m(L+s+1)+1,\ldots,n\;$}
\FOR{problem instance $\;1 \le i \le m\;$}
\STATE Compute score based on $\bar{S}$ dimensions only:
\begin{equation*}
s_{i,t-1} = \frac{\wh\sigma_{i,t-1}^2 + \Delta_{i,t-1}}{k_{i,t-1}}\mathrm{Tr}\big( \Sigma \hat\Sigma^{-1}_{i,t-1} \big) 
\end{equation*}
\ENDFOR
\STATE Select problem instance $I_t = \arg\max_{i\in[m]} s_{i,t}$
\STATE Observe $X_t$ and $Y_{I_t, t}$
\STATE Update OLS estimators $\hat\beta_{I_t,t}$ and $\hat\sigma^2_{I_t,t}$ based on $\bar{S}$
\ENDFOR
\STATE Return RLS estimates $\{ \hat\beta_{i}^{\lambda} \}_{i=1}^m$, with $\hat\beta_{ij}^{\lambda} = 0$ if $j \notin \bar{S}$
\end{algorithmic}
\caption{Sparse-\traceucb Algorithm.}
\label{alg:sparse_trace_ucb}
\end{algorithm}

Given our pure exploration perspective, it is obviously more efficient to learn the true supports as soon as possible.
That way we can adjust our behavior by collecting the right data based on our initial findings.
Note that this is not always the case; for example, if the total number of pulls is unknown.
Then it is not clear what is the right amount of budget to invest upfront to recover the supports (see tracking algorithms and doubling trick).

We briefly describe Algorithm~\ref{alg:sparse_trace_ucb} in words.
First, in the \emph{recovery} stage we sequentially pull all arms a number of times, say $L$ times.
We do not take into account the context, and just apply a round robin technique to pull each arm exactly $L$ times.
In total, there are exactly $s$ components that are non-zero for at least one arm (out of $d$).
After the $Lm$ pulls, we use a block-regularized Lasso algorithm to recover the joint sparsity pattern.
We discuss some of the alternatives later.
The outcome of this stage is a common support $\wh{S} := \cup_i \ \textbf{supp}(\hat\beta_i)$.
With high probability we recover the true support $\wh{S} = S$.
In the second stage, or \emph{pure exploration} stage, the original \traceucb algorithm is applied.
The \traceucb algorithm works by computing an estimate $\hat\sigma_i^2$ at each step $t$ for each arm $i$.
Then, it pulls the arm maximizing the score
\begin{equation*}
s_{i,t-1} = \frac{\wh\sigma_{i,t-1}^2 + \Delta_{i,t-1}}{k_{i,t-1}}\mathrm{Tr}\big( \Sigma \hat\Sigma^{-1}_{i,t-1} \big).
\end{equation*}
The key observation is that in the second stage we only consider the components of each context that are in $\wh S$.
In particular, we start by pulling $s+1$ times each arm so that we can compute the initial OLS estimates $\hat\beta_i^{\text{OLS}}$ and $\hat\sigma_i^{2}$.
We keep updating those estimates when an arm is pulled, and the trace is computed with respect to the components in $\wh S$ only.

Finally, we return the Ridge estimates based only on the data collected in the second stage.

\subsection{A note on the Static Allocation}
What is the optimal static performance in this setting if the $\sigma^2$'s are known?
For simplicity, suppose we pull arm $i$ exactly $(\sigma_i^2 / \sum_j \sigma_j^2) \ n$ times.
We are interested in Lasso guarantees for $\| X^T(\hat\beta_i - \beta_i) \|^2_2$.
Note in this case we can actually set $\lambda_i$ as a function of $\sigma_i^2$ as required in most Lasso analyses, because $\sigma_i^2$ is known.

A common guarantee is as follows (see \cite{hastie2015statistical, raskutti2010restricted}).
With high probability
\begin{equation*}
\| \hat\beta_i - \beta_i \|_2^2 \le \frac{c^2 \sigma_i^2}{\gamma^2} \ \frac{\tau s \log d}{k},
\end{equation*}
where $k$ is the number of observations, $d$ the ambient dimension, $s$ the efficient dimension, $\gamma$ is the restricted eigenvalues constant for $\Sigma$, $\tau > 2$ is the parameter that tunes the probability bound, and $c$ is a universal constant.

Thus, if we set $k = (\sigma_i^2 / \sum_j \sigma_j^2) \ n$, then we obtain that whp
\begin{equation}
\| \hat\beta_i - \beta_i \|_2^2 \le \frac{c^2 \tau}{\gamma^2} \left( \sum_{j=1}^m \sigma_j^2 \right) \frac{s \log d}{n}.
\end{equation}
Note that the latter event is independent across different $i \in [m]$, so all of them simultaneously hold with high probability.
The term $\gamma^{-2}$ was expected as depending on the correlation levels in $\Sigma$ the problem can be easier or harder.
In addition, note that as $\| \hat\beta_i - \beta_i \|_\Sigma^2 = \Tr(\Sigma (\hat\beta_i - \beta_i) (\hat\beta_i - \beta_i)^T )$, we have that
\begin{equation}
\lambda_{\min}(\Sigma) \ \| \hat\beta_i - \beta_i \|_2^2 \le \| \hat\beta_i - \beta_i \|_\Sigma^2 \le \lambda_{\max}(\Sigma) \| \hat\beta_i - \beta_i \|_2^2.
\end{equation}

\subsection{Simultaneous Support Recovery}

There has been a large amount of research on how to perform \emph{simultaneous} support recovery in sparse settings for multiple regressions.
Let $\mathrm{M}$ be the matrix whose $i$-th column is $\mathrm{M}^{(i)} = \beta_i$.

A common objective function after $k$ observations per problem is
\begin{equation}\label{opt_problem}
\min_{\bar{\mathrm{M}} \in \R^{d \times m}} \frac{1}{k} \sum_{j=1}^m \| \Y_j - \X_j \bar{\mathrm{M}}^{(j)} \|^2 + \lambda \ \| \bar{\mathrm{M}} \|,
\end{equation}
where we assumed $\Y_j = \X_j \beta_j + \epsilon_j$, and $\X_j \in \R^{k \times d}, \Y_j, \epsilon_j \in \R^k$ and $\beta_j \in \R^d$.

The $l_a/l_b$ block regularization norm is
\begin{equation}
\| \bar{\mathrm{M}} \|_{l_a/l_b} = \| v \|_a, \qquad \text{ where } v_j = \| \bar{\mathrm{M}}_j \|_b \quad \bar{\mathrm{M}}_j \text{ is the j-th row of } \bar{\mathrm{M}}.
\end{equation}
There are a few differences among the most popular pieces of work.

Negahban and Wainwright \cite{negahban2011simultaneous} consider random Gaussian designs $\X_j \sim \mathcal{N}(0, \Sigma_j)$ with random Gaussian noise (and common variance).
The regularization norm is $l_1 / l_{\infty}$.
In words, they take the sum of the absolute values of the maximum element per row in $\bar{\mathrm{M}}$.
This forces sparsity (via the $l_1$ norm), but once a row is selected there is no penalty in increasing the $\bar{\beta}$ components up to the current maximum of the row.
They tune $\lambda$ as in the standard analysis of Lasso, that is, proportionally to $\sigma^2$, which is unknown in our case.
Results are non-asymptotic, and recovery happens with high probability when the number of observations is $k > C s (m + \log d)$.
They show that if the overlap is not large enough (2/3 of the support, for $m=2$ regression problems), then running independent Lasso estimates has higher statistical efficiency.
We can actually directly use the results in \cite{negahban2011simultaneous} if we assume an upper bound $\sigma_{\max}^2 \le R$ is known.

Obozinski, Wainwright and Jordan \cite{obozinski2011support} use $l_1 / l_2$ block regularization (aka Multivariate Group Lasso).
Their design is random Gaussian, but it is fixed across regressions: $\X_j = \X$.
They provide asymptotic guarantees under the scaling $k, d, s \to \infty$, $d - s \to \infty$, and standard assumptions like bounded $\Sigma$-eigenspectrum, the irrepresentable condition, and self-incoherence.
The first condition is not only required for support recovery, but also for $l_2$ consistency.
The last two conditions are not required for risk consistency, while essential for support recovery.
To capture the amount of non-zero pattern overlap among regressions, they define the sparsity overlap function $\psi$, and their sample requirements are a function of $\psi$.
In particular, one needs $k > C \ \psi \ log (d-s)$, where the constant $C$ depends on quantities related to the covariance matrix of the design matrices, and $\psi$ can be equal to $s / m$, if all the patterns overlap, and at most $s$ if they are disjoint.

Their theorems use a sequence of regularization parameters
\begin{equation*}
\lambda_k = \sqrt{\frac{f(d) \log d}{k}}, \qquad \text{ where } f(d) \to \infty \text{ as } d \to \infty,
\end{equation*}
in such a way that $\lambda_k \to 0$ as $k, d \to \infty$.
Finally, $k > 2s$ is also required.
They also provide a two-stage algorithm for efficient estimation of \emph{individual} supports for each regression problem.
All these optimization problems are convex, and can be efficiently solved in general.

To overcome the issue of common designs (we do not pull each context several times), we use the results by Wang, Liang, and Xing in \cite{wang2013block}.
They extend the guarantees in \cite{obozinski2011support} to the case where the design matrices are independently sampled for each regression problem.
In order to formally present their result, we describe some assumptions.
Let $\Sigma^{(i)}$ be the covariance matrix for the design observations of the $i$-th regression (in our case, they are all equal to $\Sigma$), and $S$ the union of the sparse supports across regressions.

\begin{itemize}
\item \textbf{C1} There exists $\gamma \in (0, 1]$ such that $\||A\||_{\infty} \le 1 - \gamma$, where\begin{equation}
A_{js} = \max_{1 \le i \le m} \bigg| \left( \Sigma_{S^CS}^{(i)} \left( \Sigma_{SS}^{(i)} \right)^{-1} \right)_{js} \bigg|,
\end{equation}
for $j \in S^C$ and $s \in S$.

\item \textbf{C2} There are constants $0< C_{\min} \le C_{\max} < \infty$, such that the eigenvalues of all matrices $\Sigma^{(i)}$ are in $[C_{\min}, C_{\max}]$.

\item \textbf{C3} There exists a constant $D_{\max} < \infty$ such that
\begin{equation}
\max_{1 \le i \le m} | \|\ \left( \Sigma_{SS}^{(i)} \right)^{-1} \||_{\infty} \le D_{\max}.
\end{equation}

\item \textbf{C4} Define the regularization parameter
\begin{equation}
\lambda_k = \sqrt{\frac{f(d) \log d}{k}}, \qquad \text{ where } f(d) \to \infty \text{ as } d \to \infty,
\end{equation}
such that $\lambda_k \to 0$ as $k \to \infty$.

\item \textbf{C5} Define $\rho(k, s, \lambda_k)$ as
\begin{equation}
\rho(k, s, \lambda_k) := \sqrt{\frac{8 \sigma^2_{\max} s \log s}{k \ C_{\min}}} + \lambda_k \left( D_{\max} + \frac{12s}{C_{\min} \sqrt{k}} \right),
\end{equation}
and assume $\rho(k, s, \lambda_k) / b^*_{\min} = o(1)$, where $b^*_{\min} = \min_{j \in S} \| \mathrm{M}_j \|_2 $.
\end{itemize}
We state the main theorem in \cite{wang2013block}; $k$ is the number of observations \emph{per} regression.

\begin{theorem}\label{th:main_support_recovery}
Assume the parameters $(k, d, s, \mathrm{M}, \Sigma^{(1:m)})$ satisfy conditions \textbf{C1} to \textbf{C5}.
If for some small constant $v > 0$,
\begin{equation}
k > 2(1+v) \ \psi \ \log(d-s) \frac{\rho_u(\Sigma^{(1:m)}_{S_C S_C | S})}{\gamma^2},
\end{equation}
then the $l_1 / l_2$ regularized problem given in \eqref{opt_problem} has a unique solution $\hat{\mathrm{M}}$, the support union $\mathbf{supp}(\hat{\mathrm{M}})$ equals the true support $S$, and $\| \hat{\mathrm{M}} - \mathrm{M} \|_{l_\infty / l_2} = o(b_{\min}^*)$,
with probability greater than
\begin{equation}
1 - m \exp(-c_0 \log s) - \exp(-c_1 \log(d-s)),
\end{equation}
where $c_0$ and $c_1$ are constants.
\end{theorem}
The following proposition is also derived in \cite{wang2013block} (Proposition 1):
\begin{proposition}\label{prop:main_support_recovery}
Assume $\Sigma^{(1:m)}$ satisfy \textbf{C2}, then $\psi$ is bounded by
\begin{equation}
\frac{s}{m \ C_{\min}} \le \psi = \psi \left( \mathrm{M}, \Sigma^{(1:m)} \right) \le \frac{s}{C_{\min}}.
\end{equation}
\end{proposition}
For our purposes, there is a single $\Sigma$, which implies that we can remove the $\max$ expressions in \textbf{C1} and \textbf{C3}.
Corollary 2 in \cite{wang2013block} establishes that when supports are equal for all arms, the number of samples required per arm is reduced by a factor of $m$.

\subsection{High-Dimensional Trace-UCB Guarantees}
If the support overlap is complete we can reduce the sampling complexity of the first stage by a factor of $m$; we only need
\begin{equation}\label{eq_initial_pulls_condition}
Lm > 2(1+v) \ s \ \log(d-s) \frac{\rho_u(\Sigma^{(1:m)}_{S_C S_C | S})}{C_{\min} \ \gamma^2}
\end{equation}
observations in total, for some small constant $v > 0$.

Now we show our main result for high-dimensional Trace-UCB, Theorem~\ref{thm_high_prob_sparse}.
\begin{theorem*}
Let $\delta_1 > 0$.
Assume $\| \beta_i \|_2 \le Z$ for all $i$, for some $Z > 0$, and assume the parameters $(n, d, s, \beta_i, \Sigma)$ satisfy conditions \textbf{C1} to \textbf{C5} in \cite{wang2013block}.
Let $\psi$ be the sparsity overlap function defined in \cite{obozinski2011support}.
If $L > 2(1+v) \ \psi \ \log(d-s) {\rho_u(\Sigma^{(1:m)}_{S_C S_C | S})}/{\gamma^2}$ for some constant $v > 0$, and $n-Lm \ge (s+1)m$, then, with probability at least $1 - \delta_1 - \delta_2$,
\begin{equation}
\wt L_n(\mathcal{A}) \le \frac{\sum_j \sigma_j^2}{n-Lm} \left( s + 2 \log \frac{3m}{\delta_1} \right) +  \frac{2c}{\sqrt{\sigma_{\min}^2}}   \left( \frac{s \sum_j \sigma_j^2}{n - Lm} \right)^{3/2} + o \left( z \right),
\end{equation}
where $c \le 2\left(1 + \sqrt{2 \log ({12mn}/{\delta_1}) / s} \right)$
and we defined $\delta_2 = m \exp(-c_0 \log s) + \exp(-c_1 \log(d-s))$ for positive constants $c_0, c_1 > 0$, and $z = \left( {s}/{(n-Lm)} \right)^{3/2}$.
\end{theorem*}
\begin{proof}
We start by assuming the recovered support $\hat{S}$ is equal to the true support $S$.
This event, say $\mathcal{E}_S$, holds with probability at least $1 - \delta_2$ by Theorem \ref{th:main_support_recovery} when $L$ satisfies \eqref{eq_initial_pulls_condition}.

Then, we fix $\delta_1 > 0$, and run the second stage applying the Trace-UCB algorithm in the $s$-dimensional space given by the components in $\hat{S}$.

By Theorem \ref{thm_high_prob}, if $n-Lm \ge (s+1)m$, then, with probability at least $1-\delta_1$, the following holds:
\begin{equation}\label{eq:low_dimensional_loss}
\wt L_n(\mathcal{A})_S \le \frac{\sum_j \sigma_j^2}{n-Lm} \left( s + 2 \log \frac{3m}{\delta_1} \right) +  \frac{2c}{\sqrt{\sigma_{\min}^2}}   \left( \frac{s \sum_j \sigma_j^2}{n - Lm} \right)^{3/2} + o \left( \left( \frac{s}{n-Lm} \right)^{3/2} \right),
\end{equation}
where $\wt L_n(\mathcal{A})_S$ denotes the loss restricted to the components in $\beta$ that are in $\hat{S}$ (and $\Sigma_S$).
However, under event $\mathcal{E}_S$, we recovered the true support, and our final estimates for $\beta_{ij}$ for each $j \not \in S$ and arm $i$ will be equal to zero, which corresponds to their true value.
Hence $\wt L_n(\mathcal{A}) = \wt L_n(\mathcal{A})_S$.

We conclude that \eqref{eq:low_dimensional_loss} holds with probability at least $1 - \delta_1 - \delta_2$.
\end{proof}

One regime of interest is when $n = C_1 \ m s \log d \ll md$.
In addition, let us assume complete support overlap across arms, so $\psi = s / Cm$.
Then, we set the number of initial pulls per arm to be $L = C_2 \ s \log d,$ with $C_1 > C_2$.

In this case, we have that Corollary \ref{cor:high_dim_traceucb} holds.

\end{document}